%% file: main.tex
\documentclass[11pt,oneside,a4paper]{ociamthesis}  

\usepackage{amsmath,amsthm,amssymb}
\usepackage[utf8]{inputenc}
\usepackage[english]{babel}
\usepackage{thmtools, thm-restate}
\usepackage{physics}
\usepackage{hyperref} 
\usepackage{mathtools}
\usepackage{subfigure}
\usepackage{pgfplots}
\usepackage[top=3cm,left=3cm,right=3cm,bottom=3cm]{geometry}
\usepackage{datetime}
\usepackage{nomencl}
\makenomenclature
\usepackage{etoolbox}
\usepackage{bm}
\usepackage[utf8]{inputenc}
\usepackage{booktabs}
\usepackage[font=small]{caption}

\pgfplotsset{every axis/.append style={tick label style={/pgf/number format/fixed},font=\scriptsize,ylabel near ticks,xlabel near ticks,grid=major}}

\newcommand{\matr}[1]{\mathbf{#1}} 
\renewcommand{\vec}[1]{\bm{#1}}
\newcommand{\rvec}[1]{\bm{#1}}

\newdateformat{monthyeardate}{%
  \monthname[\THEMONTH], \THEYEAR}

\DeclarePairedDelimiterX{\infdivx}[2]{(}{)}{%
  #1\;\delimsize\|\;#2%
}
\newcommand{\infdiv}{\infdivx}
\newcommand{\Cov}{\mathrm{Cov}}

\title{Text to Image Synthesis\\[1ex]     
       Using Generative Adversarial Networks }   

\author{Cristian Bodnar\\[1cm]{\small Supervisor: Dr Jon Shapiro}} 
\college{}

\degree{BSc Computer Science} 
\degreedate{\monthyeardate\today} 

\include{notation}

\renewcommand\nomgroup[1]{%
  \item[\bfseries
  \ifstrequal{#1}{A}{Numbers and Arrays}{%
  \ifstrequal{#1}{B}{Sets}{%
  \ifstrequal{#1}{C}{Functions}{%
  \ifstrequal{#1}{D}{Calculus}{%
  \ifstrequal{#1}{E}{Distributions and Information Theory}{%
  }}}}}%
]}

\begin{document}

\baselineskip=16pt plus1pt
\setlength{\parskip}{0.1 \baselineskip}

\setcounter{secnumdepth}{3}
\setcounter{tocdepth}{3}

\maketitle                  
\include{dedication}        
\include{acknowlegements}   
\include{abstract}          

\begin{romanpages}          
\tableofcontents            
\listoffigures              
\listoftables               
\printnomenclature[2.5cm]
\end{romanpages}            

\include{chapter1}

\include{chapter2}
\include{chapter3}

\include{chapter4}
\include{chapter5}

\appendix
\include{appendix1}
\include{appendix2}
\include{appendix3}
\include{appendix4}

\include{appendix5}

\addcontentsline{toc}{chapter}{Bibliography}
\bibliography{refs}        
\bibliographystyle{plain}  

\end{document}

%% file: notation.tex
\nomenclature[A, 01]{$x$}{A scalar or element of a set $\mathcal{X}$ depending on the context}
\nomenclature[A, 02]{$\matr{X}$}{A matrix}
\nomenclature[A, 03]{$\vec{x}$}{A vector}
\nomenclature[A, 04]{$\matr{I}$}{Identity matrix with dimensions inferred from context}
\nomenclature[A, 05]{$X$}{A random variable}
\nomenclature[A, 06]{$\bm{X}$}{A random vector}
\nomenclature[A, 07]{$\vec{x}^{(i)}$}{The $i^{th}$ example from a dataset}
\nomenclature[A, 08]{$\vec{x}_i$}{The $i^{th}$ element of the vector $\vec{x}$}

\nomenclature[B, 01]{$\mathcal{A}$}{An arbitrary set}
\nomenclature[B, 02]{$\mathbb{R}$}{The set of real numbers}
\nomenclature[B, 03]{$[a, b]$}{The real interval which includes $a$ and $b$}
\nomenclature[B, 04]{$(\mathcal{A}, d_{\mathcal{A}})$}{A metric space on the set $\mathcal{A}$ with associated distance $d_{\mathcal{A}}$}

\nomenclature[C, 01]{$f:\mathcal{A} \to \mathcal{B}$}{A function $f$ with domain $\mathcal{A}$ and codomain $\mathcal{B}$}
\nomenclature[C, 02]{$f_{\vec{\theta}}(x)$}{A function of $x$ parametrised by $\vec{\theta}$}
\nomenclature[C, 03]{$G_{\vec{\theta}}$}{The generator parametric function. Sometimes $G$ is used instead.}
\nomenclature[C, 04]{$D_{\vec{\omega}}$}{The discriminator parametric function. Sometimes $D$ is used instead.}
\nomenclature[C, 05]{$\log(x)$}{Natural logarithm of $x$}
\nomenclature[C, 05]{$\exp(x)$}{Natural exponential function}
\nomenclature[C, 06]{$\norm{\vec{x}}$}{The $L^2$ norm of a vector $\vec{x}$}
\nomenclature[C, 07]{$\phi(t)$}{The parametric function which takes a text as input and outputs a text embedding. The parameters are not explicitly specified to save space.}
\nomenclature[C, 08]{$L_G$}{The loss function of the generator}
\nomenclature[C, 09]{$L_D$}{The loss function of the discriminator}
\nomenclature[C, 10]{$L_{GP}$}{Gradient penalty term}
\nomenclature[C, 10]{$L_{LP}$}{Lipschitz penalty term}
\nomenclature[C, 10]{$IS(G)$}{The Inception Score for a generator G}
\nomenclature[C, 11]{$F[x, f, f^{\prime}]$}{A functional $F$ which depends on $f(x)$ and $f^{\prime}(x)$. $F[x, f]$ is used if $F$ does not depend on $f^{\prime}$}
\nomenclature[C, 12]{$\norm{f}_L \leq K$}{A K-Lipschitz function $f$ for some constant $K$}

\nomenclature[D, 01]{$\dv{f}{x}$}{The derivative of the function $f$ with respect to $x$}
\nomenclature[D, 02]{$\pdv{f}{x}$}{The partial derivative of the multivariate function $f$ with respect to $x$}
\nomenclature[D, 03]{$f^{\prime}$}{The derivative of the function $f$ with respect to its argument}
\nomenclature[D, 04]{$\nabla_{\vec{x}}f(\vec{x})$}{The gradient of the function $f$ with respect to $\vec{x}$}
\nomenclature[D, 05]{$\int_{\mathcal{X}}f(x)dx$}{The definite integral with respect to $x$ over the set $\mathcal{X}$}

\nomenclature[E, 01]{$p_{X}(x)$}{A probability density of a continuous random variable X}
\nomenclature[E, 01]{$P_{X}(x)$}{A probability mass function of a discrete random variable X}
\nomenclature[E, 02]{$\mathbb{P}_r$}{The unknown data generating distribution of a dataset}
\nomenclature[E, 04]{$\mathbb{P}_g$}{The distribution of a generative model}
\nomenclature[E, 04]{$\mathbb{P}_{ge}$}{The joint distribution of the generated images and the text embeddings they were generated from}
\nomenclature[E, 05]{$\mathbb{P}_{r-mat}$}{The distribution of real image and text pairs which match}
\nomenclature[E, 05]{$\mathbb{P}_{r-mis}$}{The distribution of real image and text pairs which mismatch}
\nomenclature[E, 05]{$\mathbb{P}_{\eta}$}{The joint distributions of interpolations between real and generated images and text descriptions which match the real images used in the interpolation.}
\nomenclature[E, 06]{$X \sim P$}{A random variable $X$ with distribution $P$}
\nomenclature[E, 07]{$\mathbb{E}_{X \sim P}[f(x)]$}{The expected value of function $f$ with respect to $P$}
\nomenclature[E, 08]{$\mathcal{N}(\vec{\mu}, \matr{\Sigma})$}{Multivariate normal distribution with mean $\vec{\mu}$ and covariance matrix $\matr{\Sigma}$}
\nomenclature[E, 09]{$KL\infdiv{P}{Q}$}{The Kullback-Leibler divergence between $P$ and $Q$}
\nomenclature[E, 10]{$JS\infdiv{P}{Q}$}{The Jensen-Shannon divergence between $P$ and $Q$}
\nomenclature[E, 11]{$W(P, Q)$}{The Wasserstein distance between $P$ and $Q$}
\nomenclature[E, 12]{$U(a, b)$}{Uniform distribution on the interval $[a, b]$}
\nomenclature[E, 13]{$U(a, b)^n$}{Uniform distribution on the set $[a, b]^n$}

%% file: dedication.tex
\begin{dedication}
This report is dedicated to my brother Andrei
\end{dedication}

%% file: acknowlegements.tex
\begin{acknowledgements}
I would like to thank my supervisor Dr Jon Shapiro for his guidance and constructive suggestions during the planning and development of this research work. I also want to express my gratitude to Teodora Stoleru for her assistance, patience, and useful critique of this report. Finally, I wish to thank my family for their support and encouragement throughout my studies.
\end{acknowledgements}

%% file: abstract.tex
\begin{abstract}
Generating images from natural language is one of the primary applications of recent conditional generative models. Besides testing our ability to model conditional, highly dimensional distributions, text to image synthesis has many exciting and practical applications such as photo editing or computer-aided content creation. Recent progress has been made using Generative Adversarial Networks (GANs). This material starts with a gentle introduction to these topics and discusses the existent state of the art models. Moreover, I propose Wasserstein GAN-CLS, a new model for conditional image generation based on the Wasserstein distance which offers guarantees of stability. Then, I show how the novel loss function of Wasserstein GAN-CLS can be used in a Conditional Progressive Growing GAN. In combination with the proposed loss, the model boosts by $7.07\%$ the best Inception Score (on the Caltech birds dataset) of the models which use only the sentence-level visual semantics. The only model which performs better than the Conditional Wasserstein Progressive growing GAN is the recently proposed AttnGAN which uses word-level visual semantics as well.
\end{abstract}

%% file: chapter1.tex
\chapter{Background}

This chapter contains a short introduction to the problem of text to image synthesis, generative models and Generative Adversarial Networks. An overview of these subjects is needed for the reader to understand the more complex ideas introduced in the later chapters.

\section{Text to Image Synthesis}\label{sec:text-to-image}

One of the most common and challenging problems in Natural Language Processing and Computer Vision is that of image captioning: given an image, a text description of the image must be produced. Text to image synthesis is the reverse problem: given a text description, an image which matches that description must be generated. 

From a high-level perspective, these problems are not different from language translation problems. In the same way similar semantics can be encoded in two different languages, images and text are two different ``languages'' to encode related information. 

Nevertheless, these problems are entirely different because text-image or image-text conversions are highly multimodal problems. If one tries to translate a simple sentence such as ``This is a beautiful red flower'' to French, then there are not many sentences which could be valid translations. If one tries to produce a mental image of this description, there is a large number of possible images which would match this description. Though this multimodal behaviour is also present in image captioning problems, there the problem is made easier by the fact that language is mostly sequential. This structure is exploited by conditioning the generation of new words on the previous (already generated) words. Because of this, text to image synthesis is a harder problem than image captioning.

The generation of images from natural language has many possible applications in the future once the technology is ready for commercial applications. People could create customised furniture for their home by merely describing it to a computer instead of spending many hours searching for the desired design. Content creators could produce content in tighter collaboration with a machine using natural language.

\subsection{Datasets}

\begin{table}[hb]
    \centering
    \begin{tabular}{| l | l | l | l |}
    \hline
    Dataset/Number of images & Train & Test & Total \\ \hline
    Flowers & 7,034 & 1,155 & 8,192 \\ \hline
    Augmented flowers (256x256) & 675,264 & 110,880 & 786,432 \\ \hline
    Birds & 8,855 & 2,933 & 11,788 \\ \hline
    Augmented birds (256x256) & 850,080 & 281,568 & 1,131,648 \\ \hline

    \end{tabular}
    \caption{Summary statistics of the datasets.}
    \label{tab:datasets}
\end{table}

The publicly available datasets used in this report are the Oxford-102 flowers dataset \cite{Nilsback08} and the Caltech CUB-200 birds dataset \cite{WahCUB_200_2011}. These two datasets are the ones which are usually used for research on text to image synthesis. Oxford-102 contains 8,192 images from 102 categories of flowers. The CUB-200 dataset includes 11,788 pictures of 200 types of birds. These datasets include only photos, but no descriptions. Nevertheless, I used the publicly available captions collected by Reed et al.\cite{reed2016learning} for these datasets using Amazon Mechanical Turk. Each of the images has five descriptions. They are at least ten words in length, they do not describe the background, and they do not mention the species of the flower or bird (Figure \ref{fig:dataset-sample}). 

\begin{figure}[h]
\centering
\includegraphics[width=0.9\textwidth]{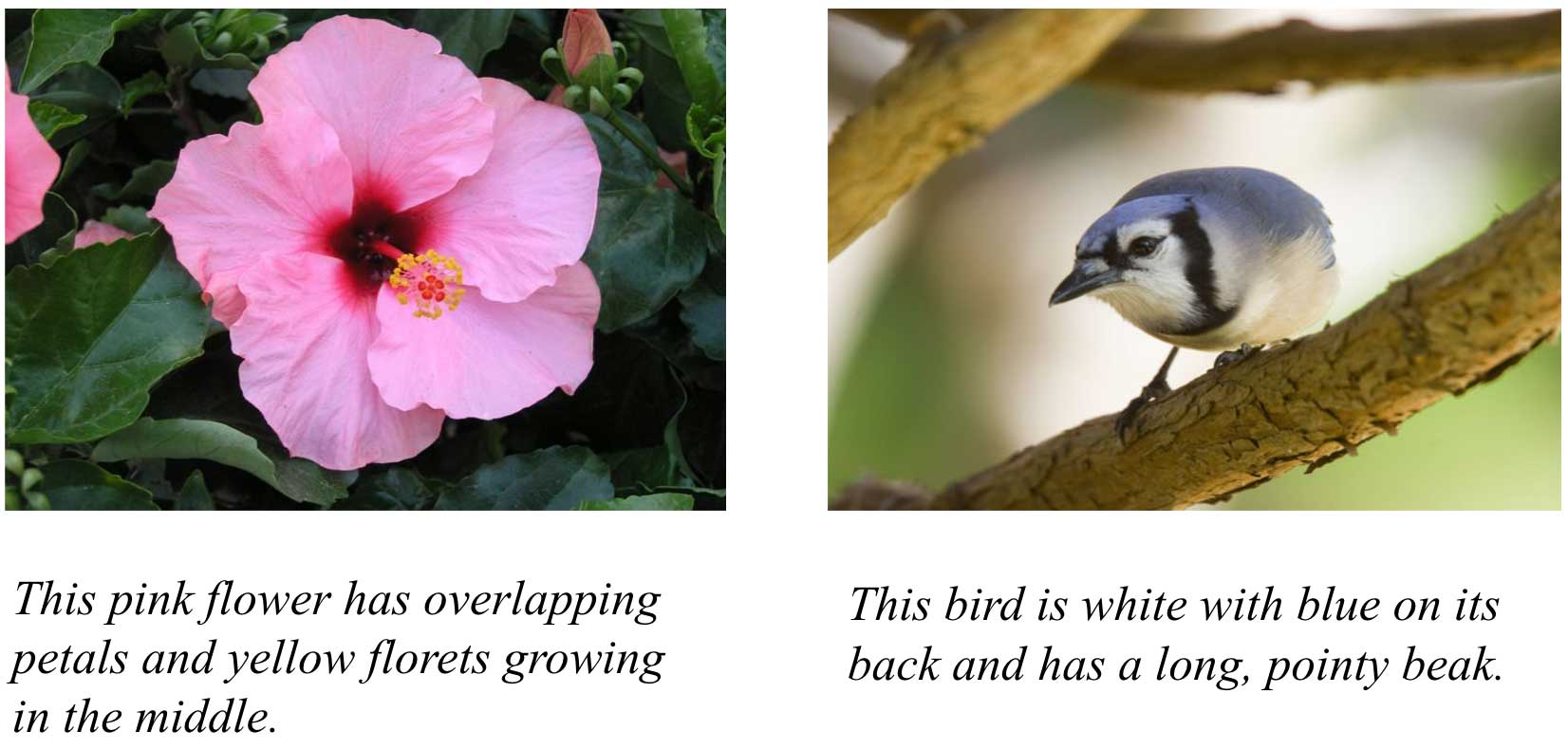}
\caption[Dataset samples]{A sample from the Oxford-102 dataset (left) and CUB-200 dataset (right), together with one of their associated descriptions collected by Reed et al. \cite{reed2016learning}}
\label{fig:dataset-sample}
\end{figure}

Because the number of images is small, I perform data augmentation. I apply random cropping and random left-right flipping of the images. I split the datasets into train and test datasets such that they contain disjoint classes of images. The datasets are summarised in Table \ref{tab:datasets}.

The report is focused on the flowers dataset for practical reasons detailed in Chapter \ref{sec:conclusions} where I discuss this decision. Nevertheless, a small number of experiments were run on the birds dataset as well.

\section{Generative Models}

The task of text to image synthesis perfectly fits the description of the problem generative models attempt to solve. The current best text to image results are obtained by Generative Adversarial Networks (GANs), a particular type of generative model. Before introducing GANs, generative models are briefly explained in the next few paragraphs.

Before defining them, I will introduce the necessary notation. Consider a dataset $\mathbb{X} = \{\vec{x}^{(1)}, \dots, \vec{x}^{(m)}\}$ composed of $m$ samples where $\vec{x}^{(i)}$ is a vector. In the particular case of this report, $\vec{x}^{(i)}$ is an image encoded as a vector of pixel values. The dataset is produced by sampling the images from an unknown data generating distribution $\mathbb{P}_{r}$, where $r$ stands for real. One could think of the data generating distribution as the hidden distribution of the Universe which describes a particular phenomenon. A generative model is a model which learns to generate samples from a distribution $\mathbb{P}_{g}$ which estimates $\mathbb{P}_{r}$. The model distribution, $\mathbb{P}_{g}$, is a hypothesis about the true data distribution $\mathbb{P}_r$.

Most generative models explicitly learn a distribution $\mathbb{P}_{g}$ by maximising the expected log-likelihood $\mathbb{E}_{\rvec{X} \sim \mathbb{P}_r}\log(\mathbb{P}_g(\vec{x} \vert \vec{\theta}))$ with respect to $\vec{\theta}$, the parameters of the model. Intuitively, maximum likelihood learning is equivalent to putting more probability mass around the regions of $\mathcal{X}$ with more examples from $\mathbb{X}$ and less around the regions with fewer examples. It can be shown that the log-likelihood maximisation is equivalent to minimising the Kullback-Leibler divergence $KL\infdiv{\mathbb{P}_{r}}{\mathbb{P}_{g}} = \int_{\mathcal{X}} \mathbb{P}_r \log \frac{\mathbb{P}_r}{\mathbb{P}_g}dx$ assuming $\mathbb{P}_{r}$ and $\mathbb{P}_{g}$ are densities. One of the valuable properties of this approach is that no knowledge of the unknown $\mathbb{P}_r$ is needed because the expectation can be approximated with enough samples according to the weak law of large numbers. 

Generative Adversarial Networks (GANs) \cite{NIPS2014_5423} are another type of generative model which takes a different approach based on game theory. The way they work and how they compare to other models is explained in the next section.

\section{Generative Adversarial Networks}

Generative Adversarial Networks (GANs) solve most of the shortcomings of the already existent generative models:

\begin{itemize}
    \item The quality of the images generated by GANs is better than the one of the other models.
    \item GANs do not need to learn an explicit density $\mathbb{P}_g$ which for complex distributions might not even exist as it will later be seen.
    \item GANs can generate samples efficiently and in parallel. For example, they can generate an image in parallel rather than pixel by pixel. 
    \item GANs are flexible, both regarding the loss functions and the topology of the network which generates samples.
    \item When GANs converge, $\mathbb{P}_g = \mathbb{P}_r$. This equality does not hold for other types of models which contain a biased estimator in the loss they optimise.
\end{itemize}

Nevertheless, these improvements come at the expense of two new significant problems: the instability during training and the lack of any indication of convergence. GAN training is relatively stable on specific architectures and for carefully chosen hyper-parameters, but this is far from ideal. Progress has been made to address these critical issues which I will discuss in Chapter \ref{sec:research}.

\begin{figure}[h]
    \centering
    \includegraphics[width=0.9\textwidth]{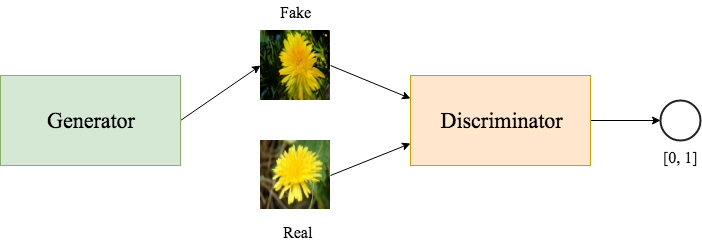}
    \caption[The GAN framework]{High-level view of the GAN framework. The generator produces synthetic images. The discriminator takes images as input and outputs the probability it assigns to the image of being real. A common analogy is that of an art forger (the generator) which tries to forge paintings and an art investigator (the discriminator) which tries to detect imitations.}
    \label{fig:gan_framework}
\end{figure}

The GAN framework is based on a game played by two entities: the discriminator (also called the critic) and the generator. Informally, the game can be described as follows. The generator produces images and tries to trick the discriminator that the generated images are real. The discriminator, given an image, seeks to determine if the image is real or synthetic. The intuition is that by continuously playing this game, both players will get better which means that the generator will learn to generate realistic images (Figure \ref{fig:gan_framework}). 

I will now show how this intuition can be modelled mathematically. Let $\mathbb{X}$ be a dataset of samples $\vec{x}^{(i)}$ belonging to a compact metric set $\mathcal{X}$ such as the space of images $[-1,1]^n$. The discriminator learns a parametric function $D_{\vec{\omega}}: \mathcal{X} \rightarrow [0, 1]$ which takes as input an image $\vec{x}$ and outputs the probability it assigns to the image of being real. Let $\mathcal{Z}$ be the range of a random vector $\bm{Z}$ with a simple and fixed distribution such as $p_{\rvec{Z}}=\mathcal{N}(\vec{0}, \vec{I})$. The generator learns a parametric function $G_{\vec{\theta}}: \mathcal{Z} \rightarrow \mathcal{X}$ which maps the states of the random vector $\rvec{Z}$ to the states of a random vector $\rvec{X}$. The states of $\rvec{X} \sim \mathbb{P}_g$ correspond to the images the generator creates. Thus, the generator learns to map a vector of noise to images.

The easiest way to define and analyse the game is as a zero-sum game where $D_{\vec{\omega}}$ and $G_{\vec{\theta}}$ are the strategies of the two players. Such a game can be described by a value function $V(D, G)$ which in this case represents the payoff of the discriminator. The discriminator wants to maximise $V$ while the generator wishes to minimise it. The payoff described by $V$ must be proportional to the ability of $D$ to distinguish between real and fake samples. The value function from equation \ref{eq:1} which was originally proposed in \cite{NIPS2014_5423} comes as a natural choice.
\begin{equation} \label{eq:1}
    V(D, G) = \mathbb{E}_{\bm{X} \sim \mathbb{P}_{r}}[\log(D(\vec{x}))] + \mathbb{E}_{\bm{Z} \sim p_{\rvec{Z}}}[\log(1 - D(G(\vec{z})))]
\end{equation}
On the one hand, note that $V(D, G)$ becomes higher when the discriminator can distinguish between real and fake samples. On the other hand, $V$ becomes lower when the generator is performing well, and the critic cannot distinguish well between real and fake samples.  

\begin{figure}[h]
    \centering
    \includegraphics[width=0.5\textwidth]{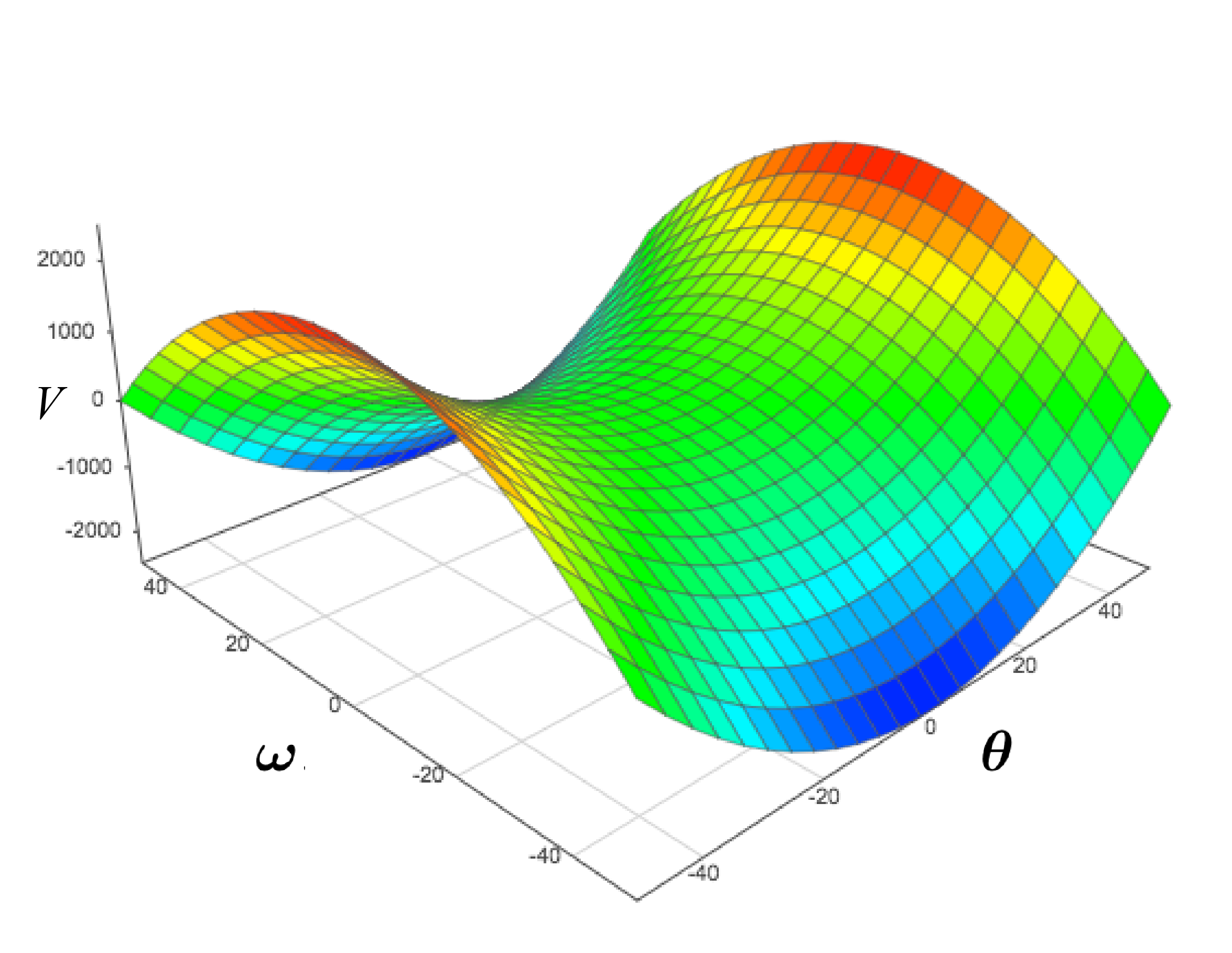}
    \caption[Saddle point example]{The Nash Equilibrium corresponds to a saddle point in space where $V$ is at a minimum with respect to $\vec{\theta}$ and at a maximum with respect to $\vec{\omega}$. Neither of the networks has any interest to change their parameters.}
    \label{fig:saddle_point}
\end{figure}

The difference from maximum likelihood models is that samples generated by the generator have an implicit distribution $\mathbb{P}_{g}$ determined by the particular value of $\vec{\theta}$. $\mathbb{P}_{g}$ cannot be explicitly evaluated. The discriminator forces the generator to bring $\mathbb{P}_g$ close to $\mathbb{P}_r$.

Zero-sum games are minimax games so the optimal parameters of the generator can be described as in Equation \ref{eq:2}.
\begin{equation} \label{eq:2}
    \vec{\theta}^* = \arg\!\min_{\vec{\theta}} \max_{\vec{\omega}} V(D, G)
\end{equation}
The solution of the minimax optimisation is a Nash Equilibrium (Figure \ref{fig:saddle_point}). Theorem \ref{theo:1} describes the exciting properties of the equilibrium for non-parametric functions. The theorem confirms the intuition that as the two players play this game, the generator will improve at producing realistic samples. 

\begin{restatable}{theorem}{fta}
\label{theo:1}
The Nash equilibrium of the (non-parametric) GAN game occurs when: 
\begin{enumerate}
    \item The strategy of the discriminator is $D = \frac{\mathbb{P}_{r}}{\mathbb{P}_{r} + \mathbb{P}_{g}}$
    \item The strategy $G$ of the generator makes $\mathbb{P}_{g}=\mathbb{P}_{r}$.
\end{enumerate}
\end{restatable}

For completeness, I offer a slightly different and more detailed proof than the one from \cite{NIPS2014_5423} in appendix \ref{appendix:A}. I encourage the reader to go through it. Note that when the equilibrium is reached, $D = \frac{1}{2}$. The interpretation of Theorem \ref{theo:1} is that, at the end of the learning process, the generator learned $\mathbb{P}_{r}$ and the discriminator is not able to distinguish between real and synthetic data, so its best strategy is to output $\frac{1}{2}$ for any input it receives. At equilibrium, the discriminator has a payoff of $-\log(4)$ and the generator $\log(4)$, respectively.

In practice, the generator does not minimise $V(D, G)$. This would be equivalent to minimising $\mathbb{E}_{\rvec{Z} \sim p_{\rvec{Z}}}[\log(1 - D_{\vec{\omega}}(G_{\vec{\theta}}(\vec{z})))]$, but the function $\log(1 - D_{\vec{\omega}}(G_{\vec{\theta}}(\vec{z})))$ has saturating gradient when the discriminator is performing well and the function approaches $\log(1)$. This makes it difficult for the generator to improve when it is not performing well. Instead, the generator is minimising $L_G$ from \ref{eq:cg} whose gradient saturates only when the generator is already performing well. The loss $L_{D}$ of the critic is also included in \ref{eq:cg}.
\begin{equation} \label{eq:cg}
\begin{split}
        L_G &= -\mathbb{E}_{\rvec{Z} \sim \mathbb{P}_{\vec{z}}}[\log(D_{\vec{\omega}}(G_{\vec{\theta}}(\vec{z}))) \\
        L_D &= -\mathbb{E}_{\rvec{X} \sim \mathbb{P}_{r}}[\log(D_{\vec{\omega}}(x))] - \mathbb{E}_{\rvec{Z} \sim p_{\rvec{Z}}}[\log(1 - D_{\vec{\omega}}(G_{\vec{\theta}}(\vec{z})))]
\end{split}
\end{equation}
With this generator loss function, the game is no longer a zero-sum game. Nevertheless, this loss works better in practice.

Based on the min-max expression \ref{eq:2}, one might expect that the discriminator is trained until optimality for a fixed generator, the generator is trained afterwards, and the process repeats. This approach does not work for reasons which will become clear in Section \ref{sec:wgan}. In practice, the networks are trained alternatively, for only one step each.

\subsection{Conditional Generative Adversarial Networks} \label{sec:cond-gans}

The original GAN paper \cite{NIPS2014_5423} describes how one can trivially turn GANs into a conditional generative model. To generate data conditioned on some condition vector $\vec{c}$, $\vec{c}$ is appended to both the generator and the discriminator in any of their layers. The networks will learn to adapt and adjust their parameters to these additional inputs. 

\begin{figure}[h]
\centering
\includegraphics[width=0.4\textwidth]{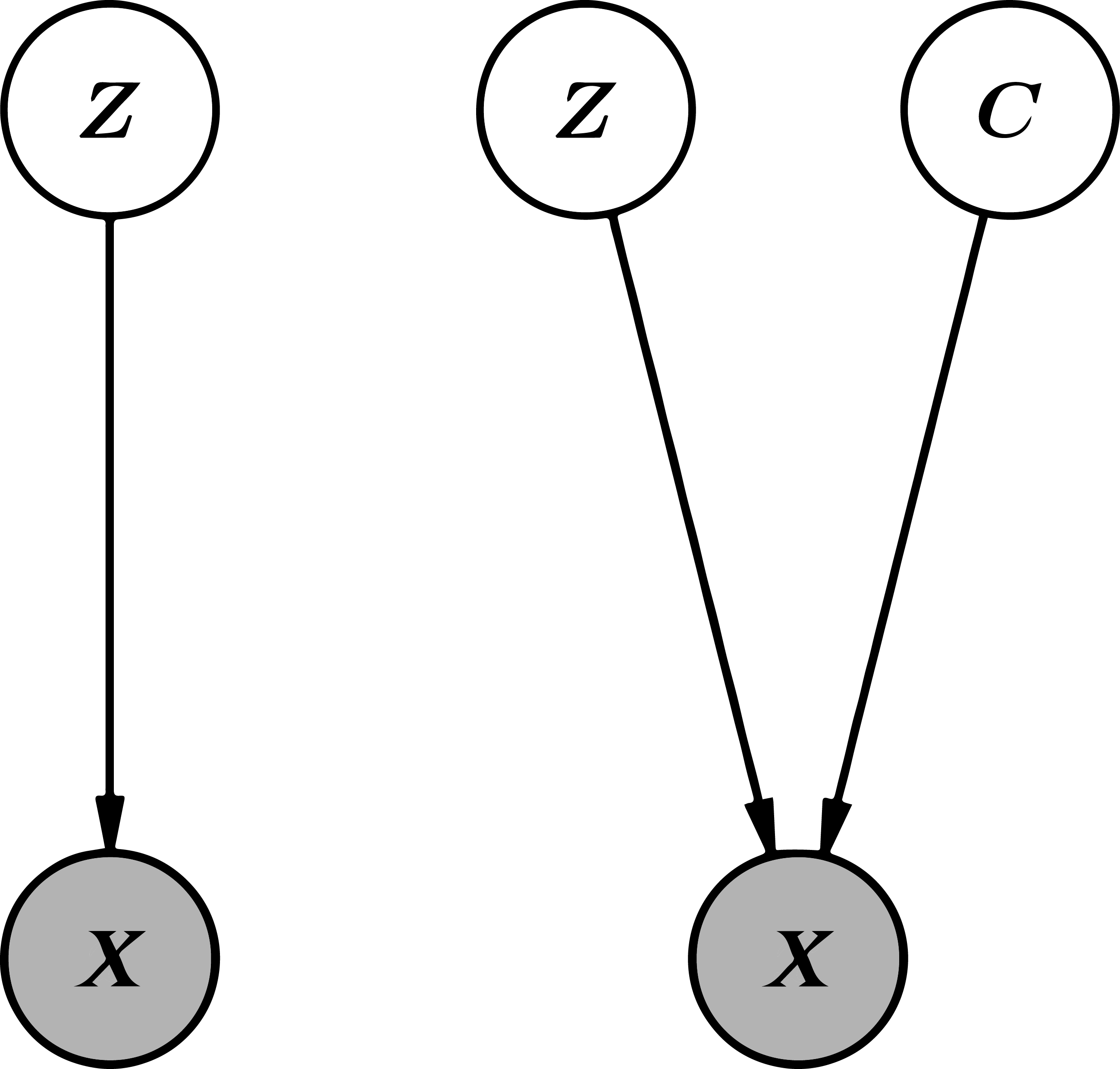}
\caption[GANs as probabilistic graphical models]{Probabilistic graphical model view of regular GANs (left) and conditional GANs (right).}
\label{fig:gan-graphical-model}
\end{figure}

Conditional GANs can also be seen from the perspective of a probabilistic graphical model (Figure \ref{fig:gan-graphical-model}). In the case of regular GANs, the noise $\rvec{Z}$ influences the observable $\rvec{X}$. For conditional $GANs$, both $\rvec{Z}$ and $\rvec{C}$ influence $\rvec{X}$. 
In the particular case of text to image synthesis, the states $\vec{c}$ of $\rvec{C}$ are vectors encoding a text description. How this encoding of a sentence into a vector is computed can vary, and it is discussed in Section \ref{sec:text-to-image}. 

\subsection{Text Embeddings} \label{sec:text-processing}

The text descriptions must be vectorised before they can be used in any model. These vectorisations are commonly referred to as text embeddings. Text embedding models were not the focus of this work, and that is why the already computed vectorisations by Reed et al. \cite{reed2016learning} are used. Other state of the art models \cite{reed2016generative, han2017stackgan} use the same embeddings and their usage makes comparisons between models easier.

\begin{figure}[h]
\centering
\includegraphics[width=0.9\textwidth]{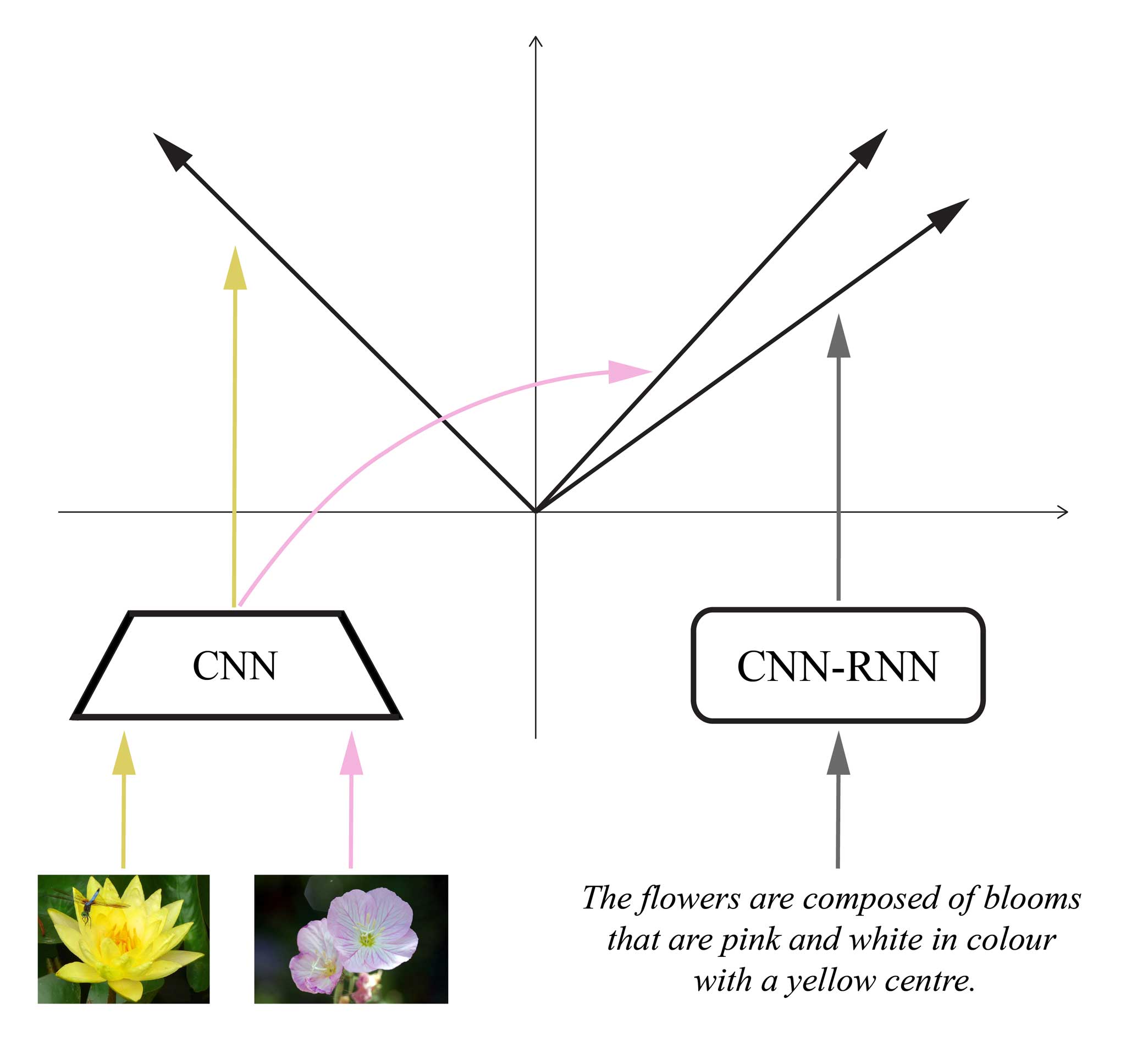}
\caption[Char-CNN-RNN encoder]{The char-CNN-RNN encoder maps images to a common embedding space. Images and descriptions which match are closer to each other. Here the embedding space is $\mathbb{R}^2$ to make visualisation easier. In practice, the preprocessed descriptions are in $\mathbb{R}^{1024}$.}
\label{fig:embeddings}
\end{figure}

The text embeddings are computed using the char-CNN-RNN encoder proposed in \cite{reed2016learning}. The encoder maps the images and the captions to a common embedding space such that images and descriptions which match are mapped to vectors with a high inner product. For this mapping, a Convolutional Neural Network (CNN) processes the images, and a hybrid Convolutional-Recurrent Neural Network (RNN) transforms the text descriptions (Figure \ref{fig:embeddings}).

A common alternative is Skip-Thought Vectors \cite{Kiros:2015:SV:2969442.2969607} which is a pure language-based model. The model maps sentences with similar syntax and semantics to similar vectors. Nevertheless, the char-CNN-RNN encoder is better suited for vision tasks as it uses the corresponding images of the descriptions as well. The embeddings are similar to the convolutional features of the images they correspond to, which makes them visually discriminative. This property reflects in a better performance when the embeddings are employed inside convolutional networks.

%% file: chapter2.tex
\chapter{State of the Art Models}

In this chapter, I first discuss in Section \ref{sec:method} the technical details of my implementation. Then, in Sections \ref{sec:gan-cls} and \ref{sec:stackgan} I present in depth the state of the art models introduced by September 2017. Section \ref{sec:other-models} briefly mentions other state of the art models which have been proposed since I started this project.

\section{Method}\label{sec:method}

All the images included in this report generated by the described models are produced by my implementation of these models. All the models which are discussed are implemented in python 3.6 using the GPU version of TensorFlow \cite{tensorflow2015-whitepaper}. 

TensorFlow is an open-source library developed by researchers from Google Brain and designed for high performance numerical and scientific computations. It is one of the most widely used libraries for machine learning research. TensorFlow offers both low level and high-level APIs which make development flexible and allow fast iteration. Moreover, TensorFlow makes use of the capabilities of modern GPUs for parallel computations to execute operations on tensors efficiently. 

I run all the experiments on a Nvidia 1080Ti which I acquired for the scope of this project. 

\section{GAN-CLS (Conditional Latent Space)} \label{sec:gan-cls}

Reed et al. \cite{reed2016generative} were the first to propose a solution with promising results for the problem of text to image synthesis. The problem can be divided into two main subproblems: finding a visually discriminative representation for the text descriptions and using this representation to generate realistic images.

In Section \ref{sec:text-processing} I briefly described how good representations for the text descriptions could be computed using the char-CNN-RNN encoder proposed in \cite{reed2016learning}. In section \ref{sec:cond-gans} I also explained how Conditional GANs could be used to generate images conditioned on some vector $\vec{c}$. GAN-CLS puts these two ideas together.

The functions $G(\vec{z})$ and $D(\vec{x})$ encountered in regular GANs become in the context of conditional GANs, $G(\vec{z}, \phi(\vec{t}))$ and $D(\vec{x}, \phi(\vec{t}))$, where $\phi: \Sigma^{*} \rightarrow \mathbb{R}^{N_{\phi}}$ is the char-CNN-RNN encoder, $\Sigma$ is the alphabet of the text descriptions, $\vec{t}$ is a text description treated as a vector of characters and $N_{\phi}$ is the number of dimensions of the embedding. The text embedding $\phi(\vec{t})$ is used as the conditional vector $\vec{c}$.

Before reading further, the reader is encouraged to have a look at Appendix \ref{appendix:nn} which includes a brief introduction to deep learning and explains the terms used to describe the architecture of the models. 

\subsection{Model Architecture} \label{sec:gan-cls-arch}

GAN-CLS uses a deep convolutional architecture for both the generator and the discriminator, similar to DC-GAN (Deep Convolutional-GAN) \cite{Radford2015UnsupervisedRL}.

In the generator, a noise vector $\vec{z}$ of dimension 128, is sampled from $\mathcal{N}(\vec{0}, \matr{I})$. The text $\vec{t}$ is passed through the function $\phi$ and the output $\phi(\vec{t})$ is then compressed to dimension 128 using a fully connected layer with a leaky ReLU activation. The result is then concatenated with the noise vector $\vec{z}$. The concatenated vector is transformed with a linear projection and then passed through a series of deconvolutions with leaky ReLU activations until a final tensor with dimension 64$\times$64$\times$3 is obtained. The values of the tensor are passed through a $\tanh$ activation to bring the pixel values in the range $[-1, 1]$.

\begin{figure}
    \centering
    \includegraphics[width=\textwidth]{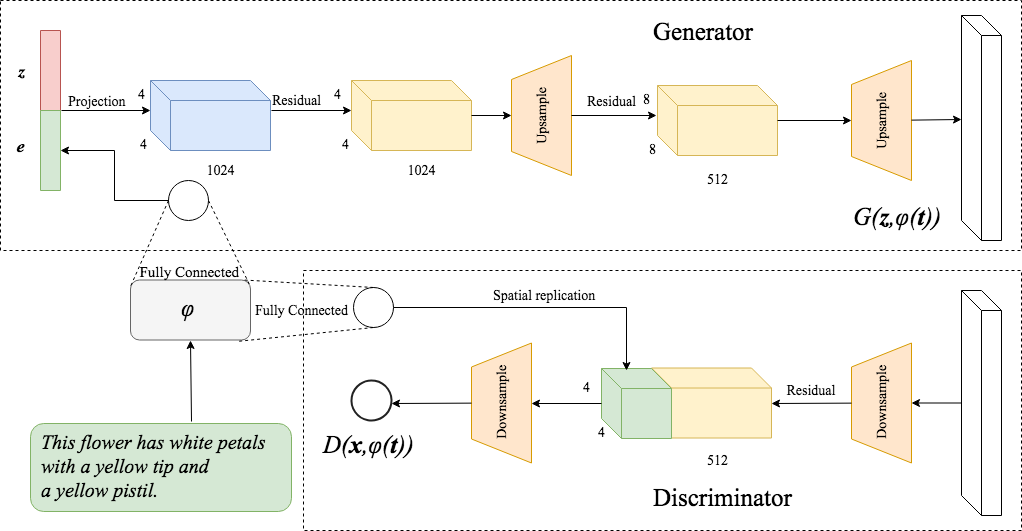}
    \caption[Architecture of Customised GAN-CLS]{Architecture of the customised GAN-CLS. Two fully connected layers compress the text embedding $\phi(\vec{t})$ and append it both in the generator and the discriminator. In the discriminator, the compressed embeddings are spatially replicated (duplicated) before being appended in depth.}
    \label{fig:gan-cls}
\end{figure}

In the discriminator, the input image is passed through a series of convolutional layers. When the spatial resolution becomes 4$\times$4, the text embeddings are compressed to a vector with 128 dimensions using a fully connected layer with leaky ReLU activations as in the generator. These compressed embeddings are then spatially replicated and concatenated in depth to the convolutional features of the network. The concatenated tensor is then passed through more convolutions until a scalar is obtained. To this scalar, a sigmoid activation function is applied to bring the value of the scalar in the range $[0, 1]$ which corresponds to a valid probability. 

The focus of the GAN-CLS paper is not on the details of the architecture of the discriminator and the generator. Thus, to obtain better results, I deviated slightly from the DC-GAN architecture, and I added one residual layer \cite{7780459} in the discriminator and two residual layers in the generator. These modifications increase the capacity of the networks and lead to more visually pleasant images. Figure \ref{fig:gan-cls} shows the architecture of the customised GAN-CLS.

\subsection{Adapting GAN Loss to Text-Image Matching}

The GAN-CLS critic has a slightly different loss function from the one presented in Equation \ref{eq:cg}. The goal of the modification is to better enforce the text-image matching by making the discriminator to be text-image matching aware. The critic cost function from Equation \ref{eq:cost-d-gan-cls} is used.
\begin{equation} \label{eq:cost-d-gan-cls}
\begin{split}
    L_D &= -\mathbb{E}_{(\rvec{X}, \rvec{E})\sim \mathbb{P}_{r-mat}}[\log(D(\vec{x}, \vec{e}))] - \frac{1}{2}(\mathbb{E}_{(\rvec{X}, \rvec{E})\sim \mathbb{P}_{ge}}[\log(1-D(\vec{x}, \vec{e}))] 
    \\ &+ \mathbb{E}_{(\vec{X}, \vec{E})\sim \mathbb{P}_{r-mis}}[\log(1 - D(\vec{x}, \vec{e}))])
\end{split}
\end{equation}
where $\mathbb{P}_{r-mat}$ is the joint distribution of image and text embeddings which match, $\mathbb{P}_{r-mis}$ is the joint distribution of image and text embeddings which mismatch, $\mathbb{P}_{ge}$ is the joint distribution of generated images and the text embeddings they were generated from and $\vec{e}=\phi(\vec{t})$ is a text embedding. In this way, the discriminator has a double functionality: distinguishing real and fake images but also distinguishing between the text-image pairs which match and those which mismatch. 

\begin{figure}[h]
    \centering
    \includegraphics[width=\textwidth]{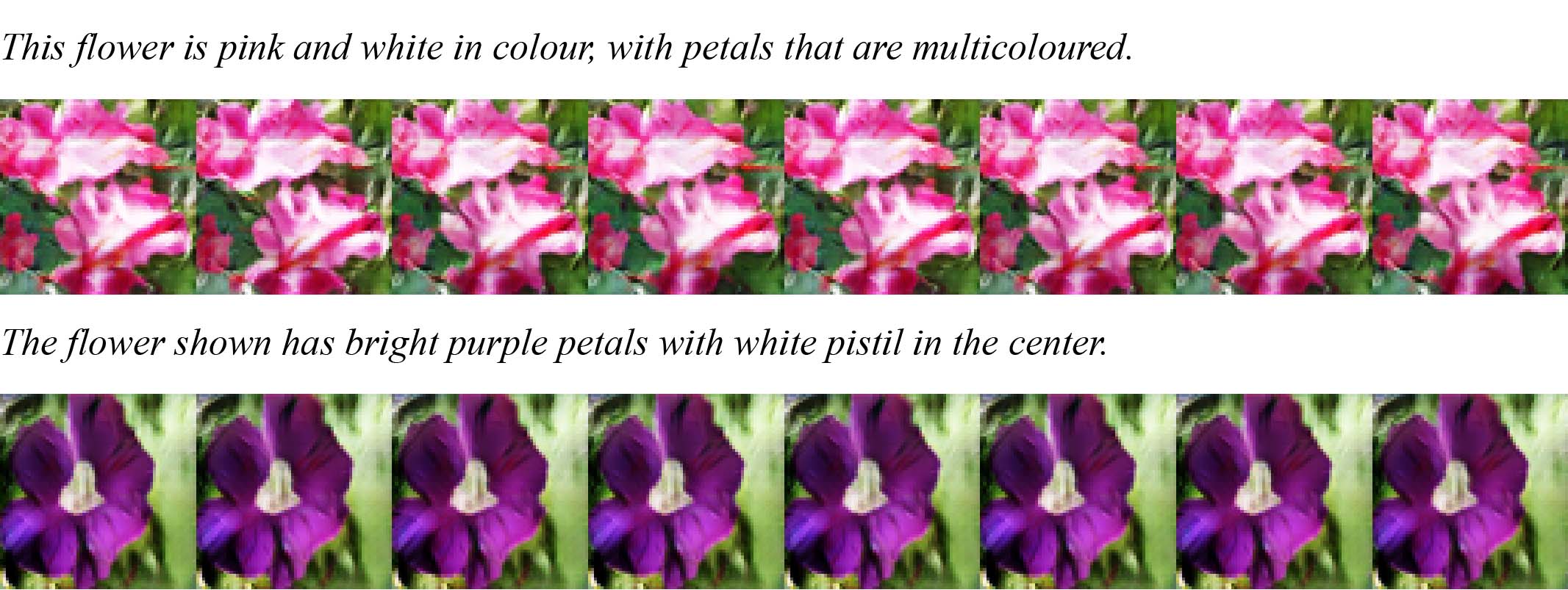}
    \caption[GAN-CLS Samples]{Samples generated from text descriptions from the test dataset. For each text description, the model generates multiple samples, each using a different input noise vector. As it can be seen the model ignores the input noise and the resulting images are extremely similar. The disappearance of the stochastic behaviour is a current research problem in Conditional GANs.}
    \label{fig:gan-cls-samples}
\end{figure}

\subsection{Training}

Adam optimiser \cite{journals/corr/KingmaB14} is used for training. Adam maintains a separate learning rate for each of the parameters of the model and uses a moving average of the first and second moment of the gradients in the computation of the parameter update. The usage of the gradient statistics makes the algorithm robust to gradient scaling and well suited for problems with noisy gradients such as this one. The parameters $\beta_1$ and $\beta_2$ control the decay rate of these moving averages. The learning rate is set to 0.0002 for both networks, $\beta_1 = 0.5$ and $\beta_2 = 0.9$. The model is trained for a total of 600 epochs with a batch size of 64. 

\subsection{Results}

Figure \ref{fig:gan-cls-samples} shows samples generated for the flowers dataset. All the shown samples are produced from descriptions from the test dataset.

\begin{figure}[h]
    \centering
    \includegraphics[width=\textwidth,height=0.8\textheight,keepaspectratio]{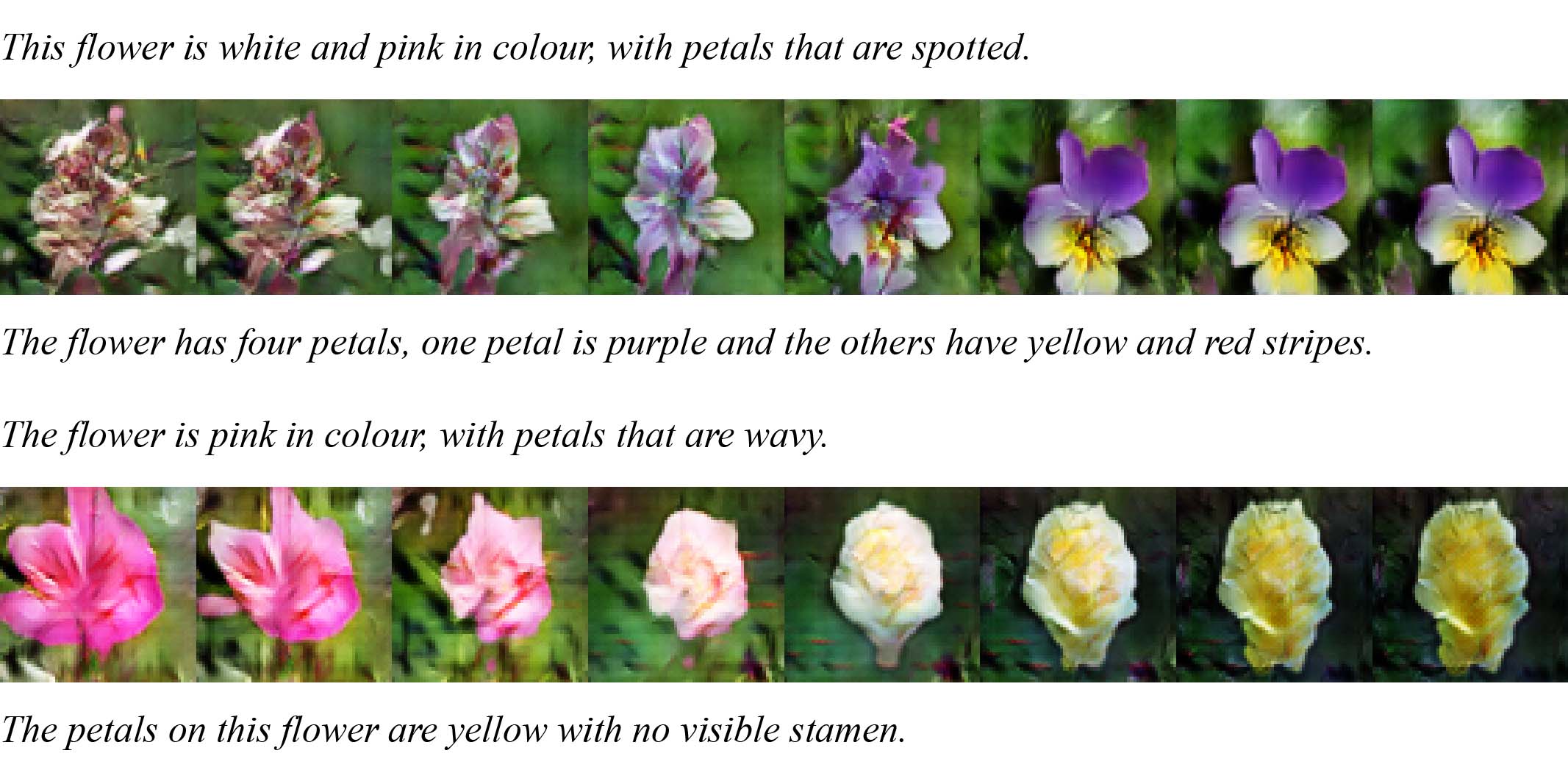}
    \caption[GAN-CLS Interpolations]{Interpolations in the conditional embedding space while maintaining the noise vector constant. The top description for each image corresponds to the image on the left and the bottom description corresponds to the image on the right. The images in between are generated from interpolations between these two descriptions.}
    \label{fig:gan-cls-interp}
\end{figure}

A common way to test that the models learns the visual semantics of the text descriptions and it does not merely memorise the description-image mappings is to generate images $G(\vec{z}, (1 - t) \vec{e}_1 + t \vec{e}_2)$ from interpolations between two text embeddings $\vec{e}_1$ and $\vec{e}_2$ where $t$ is increased from $0$ to $1$. If the model works, these transitions should be smooth. Figure \ref{fig:gan-cls-interp} shows images produced by GAN-CLS from such interpolations. 

\section{Stacked GANs} \label{sec:stackgan}

One would rarely see an artist producing a painting in full detail directly from the first attempt. By analogy, this is how GAN-CLS described in section \ref{sec:gan-cls} generates images. These architectures do not usually scale up well to higher resolutions. It would be desired to have a network architecture which is closer to the analogy of a painter who starts with the main shapes, colours, textures and then gradually adds details. 

StackGAN \cite{han2017stackgan} is such an architecture, and it uses two GANs. The first GAN, called Stage I, generates images from captions at a lower resolution of 64$\times$64 in a similar manner to GAN-CLS. The second GAN, called Stage II, has a generator which takes as input the image generated by the Stage I generator and produces a higher resolution 256$\times$256 image with more fine-grained details and better text-image matching. 

\subsection{Text Embedding Augmentation} \label{sec:cond-aug}

Besides this generation of images at multiple scales, the StackGAN paper proposes the augmentation of the conditional space. Because the number of text embeddings is small, they cover tiny, sparse regions in the embedding space clustered around their corresponding images. The model hardly understands the visual semantics of the embeddings at test time because these embeddings are usually far in the embedding space from the ones seen in training. Moreover, the low number of embeddings is likely to cause the model to overfit.

\begin{figure}
    \centering
    \includegraphics[width=\textwidth]{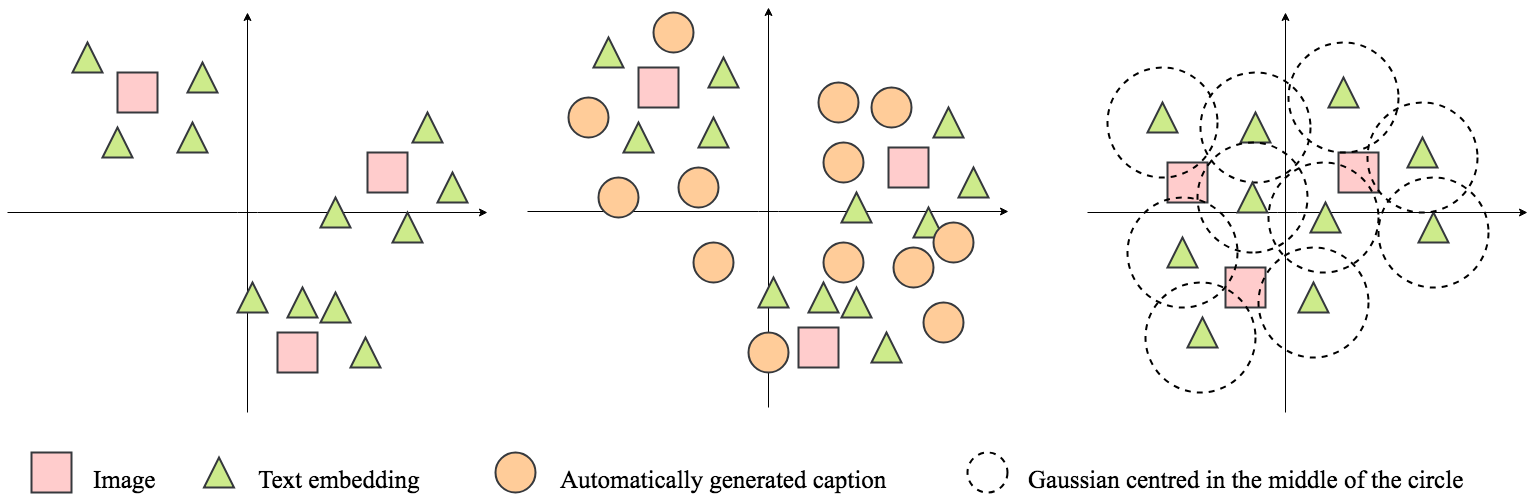}
    \caption[Augmentation in the conditional space]{Illustration of a simplified 2D conditional space before augmentation (left) and after adding two different augmentation strategies (middle and right). The image captioning system from \cite{Dong2017I2T2ILT} fills the embedding space with synthetic captions (middle). The conditional augmentation from StackGAN \cite{han2017stackgan} ensures a smooth conditional manifold by sampling from a Gaussian distribution for each text embedding (right)}
    \label{fig:cond_augl}
\end{figure}

Dong et al. \cite{Dong2017I2T2ILT} have also independently recognised this problem. They propose an image captioning system to fill the embedding space. Nevertheless, this is far from an ideal solution. The curse of dimensionality takes effect, and it is unfeasible to fill the space in such a manner. Moreover, the image captioning system adds significant computational costs.

StackGAN uses another approach inspired by another generative model, Variational Autoencoders (VAE) \cite{journals/corr/KingmaW13}. For a given text embedding $\phi(\vec{t})$, augmented embeddings can be sampled from a distribution $\mathcal{N}(\mu(\phi(\vec{t})), \Cov(\phi(\vec{t}))$. As in VAEs, to ensure that the conditional space remains smooth and the model does not overfit, a regularisation term enforces a standard normal distribution over the normal distributions of the embeddings. The regularisation term with hyper-parameter $\rho$ (Equality \ref{eq:cond-aug-reg}) consists of the $KL$ divergence between the normal distribution of the embeddings and the standard normal distribution.
\begin{equation} \label{eq:cond-aug-reg}
    L_G = - \mathbb{E}_{\rvec{X} \sim \mathbb{P}_g, \rvec{T} \sim \mathbb{P}_r}[\log(D(\vec{x})) + \rho KL\infdiv{\mathcal{N}(\vec{0}, \matr{I})}{\mathcal{N}(\mu(\phi(\vec{t})), \Sigma(\phi(\vec{t}))}] \text{, }
\end{equation}

\begin{figure}
    \centering
    \includegraphics[width=\textwidth]{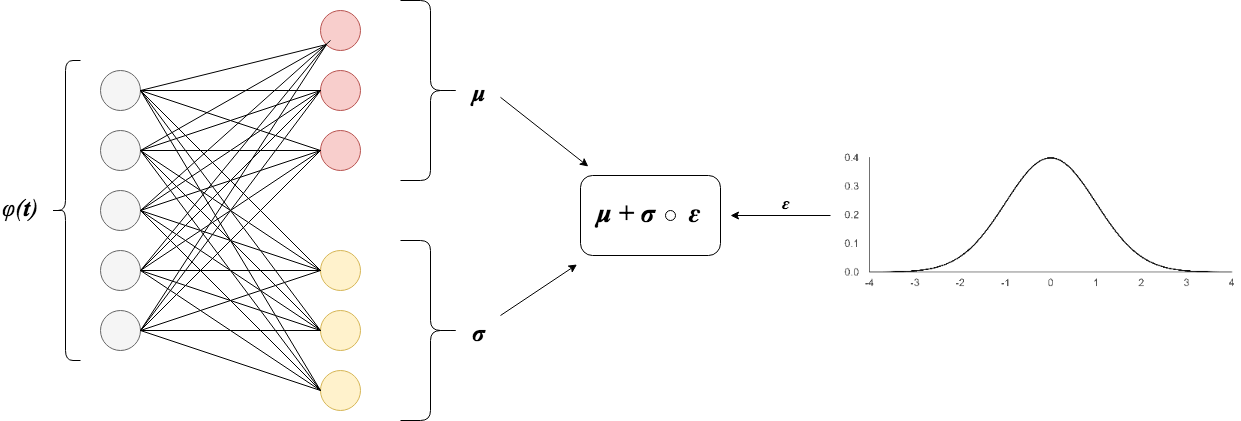}
    \caption[VAE reparametrisation trick]{The VAE reparametrisation trick. The network learns $\vec{\mu}$ and $\vec{\sigma}$ and uses a sampled $\vec{\epsilon}$ to compute an augmented embedding. The associated sampling noise causes improved image variation as the model generates different images for different samples of the same embedding.}
    \label{fig:cond-sample}
\end{figure}

The reparametrisation trick from VAEs is used to perform the sampling. With this trick, the network has the independence to learn the mean $\vec{\mu}$ and the standard deviation $\vec{\sigma}$ of the embedding. For an embedding $\vec{e} = \phi(\vec{t})$, a fully connected layer with leaky ReLU activations computes $\vec{\mu}$ and another computes $\vec{\sigma}$ and the sampled vector $\vec{\hat{e}}$ is obtained as shown in Equation \ref{eq:sample-cond} (Figure \ref{fig:cond-sample}).
\begin{equation} \label{eq:sample-cond}
    \vec{\hat{e}} = \vec{\mu} + \vec{\sigma} \circ \vec{\epsilon} \text{ , where } \epsilon \sim \mathcal{N}(\vec{0}, \matr{I}) \text{ and} \circ \text{is element-wise multiplication}
\end{equation}

\subsection{Model Architecture}

Figure \ref{fig:stack-gan} shows the full architecture of StackGAN. The architecture of the Stage I generator is identical to the one of the customised GAN-CLS (described in Section \ref{sec:gan-cls-arch}) with the addition of the conditioning augmentation (CA) module previously discussed. 

The Stage II generator starts by down-sampling the input 64$\times$64 image until it reaches a spatial resolution of 4$\times$4. To this 4$\times$4 block, the corresponding augmented text embedding is concatenated in depth to improve the text-image matching of Stage I. The concatenated block is passed through three residual layers and then up-sampled until a final tensor of 256$\times$256$\times$3 is obtained. In the end, $\tanh$ activation is applied to bring the output in $[-1, 1]$.

\begin{figure}[h]
    \centering
    \includegraphics[width=\textwidth]{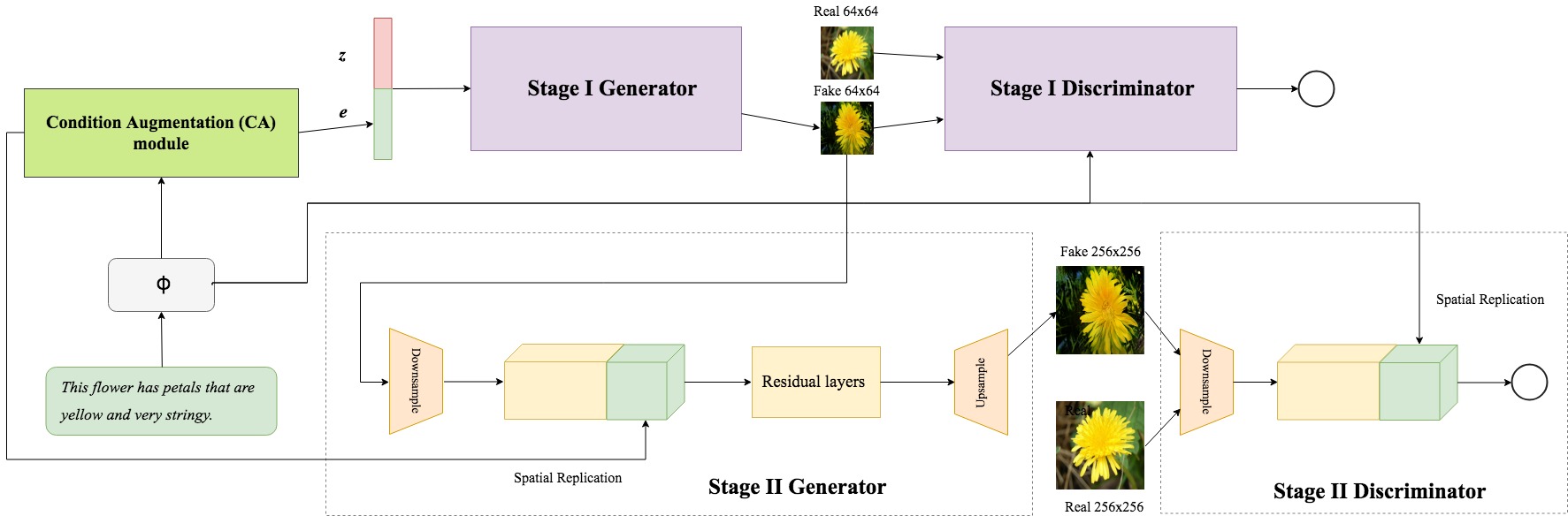}
    \caption[The architecture of StackGAN]{The architecture of StackGAN. The architecture of Stage I is identical to the customised GAN-CLS presented in the previous section. The Stage II generator takes as input and fine-tunes the image generated by Stage I. The generators of both stages use the augmented embeddings.}
    \label{fig:stack-gan}
\end{figure}

The Stage I discriminator is identical to the customised GAN-CLS discriminator previously discussed. The Stage II discriminator is also similar, with the exception that more down-sampling convolutional layers are used to accommodate for the higher resolution of the input.

As in the paper, ReLU activations are used for the generator and leaky ReLU activations for the discriminator. Batch normalisation is applied both in the generator and the discriminator.

\begin{figure}[h]
    \centering
    \includegraphics[width=\textwidth]{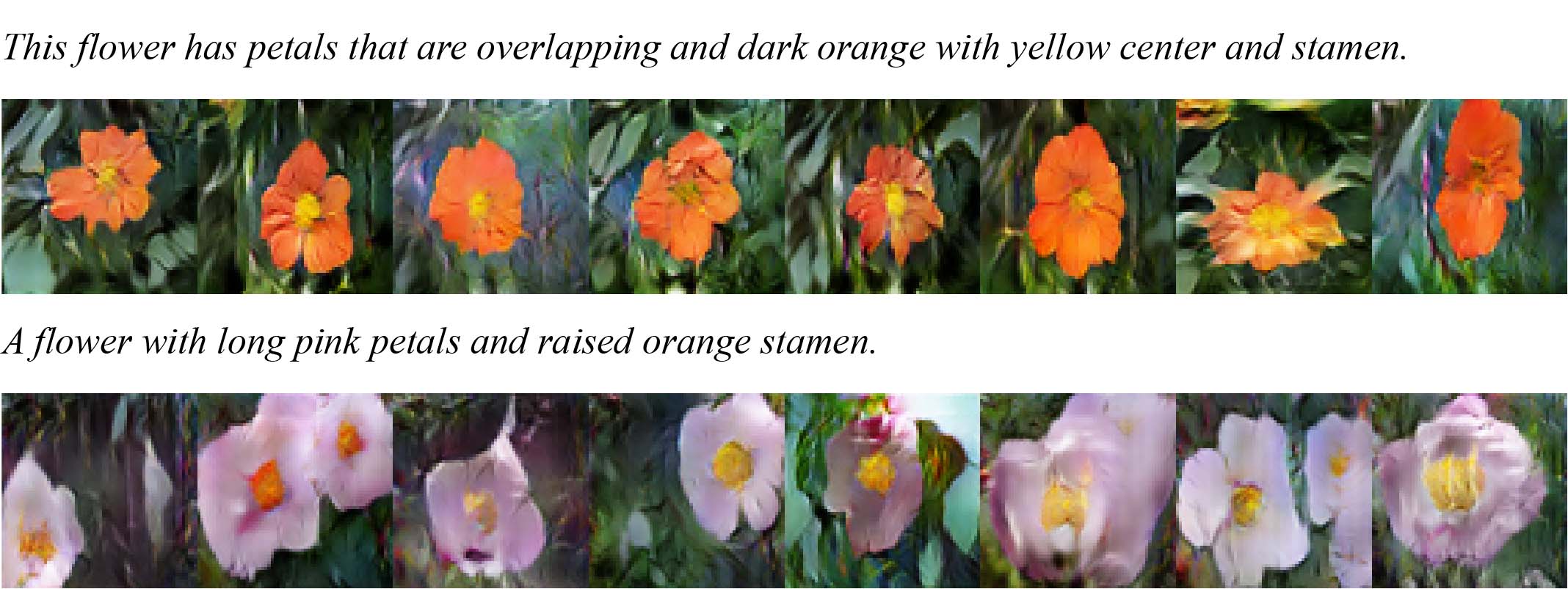}
    \caption[StackGAN Stage I Samples]{Samples generated by Stage I of StackGAN.}
    \label{fig:stageI-samples}
\end{figure}
\begin{figure}[h]
    \centering
    \includegraphics[width=\textwidth,height=0.9\textheight,keepaspectratio]{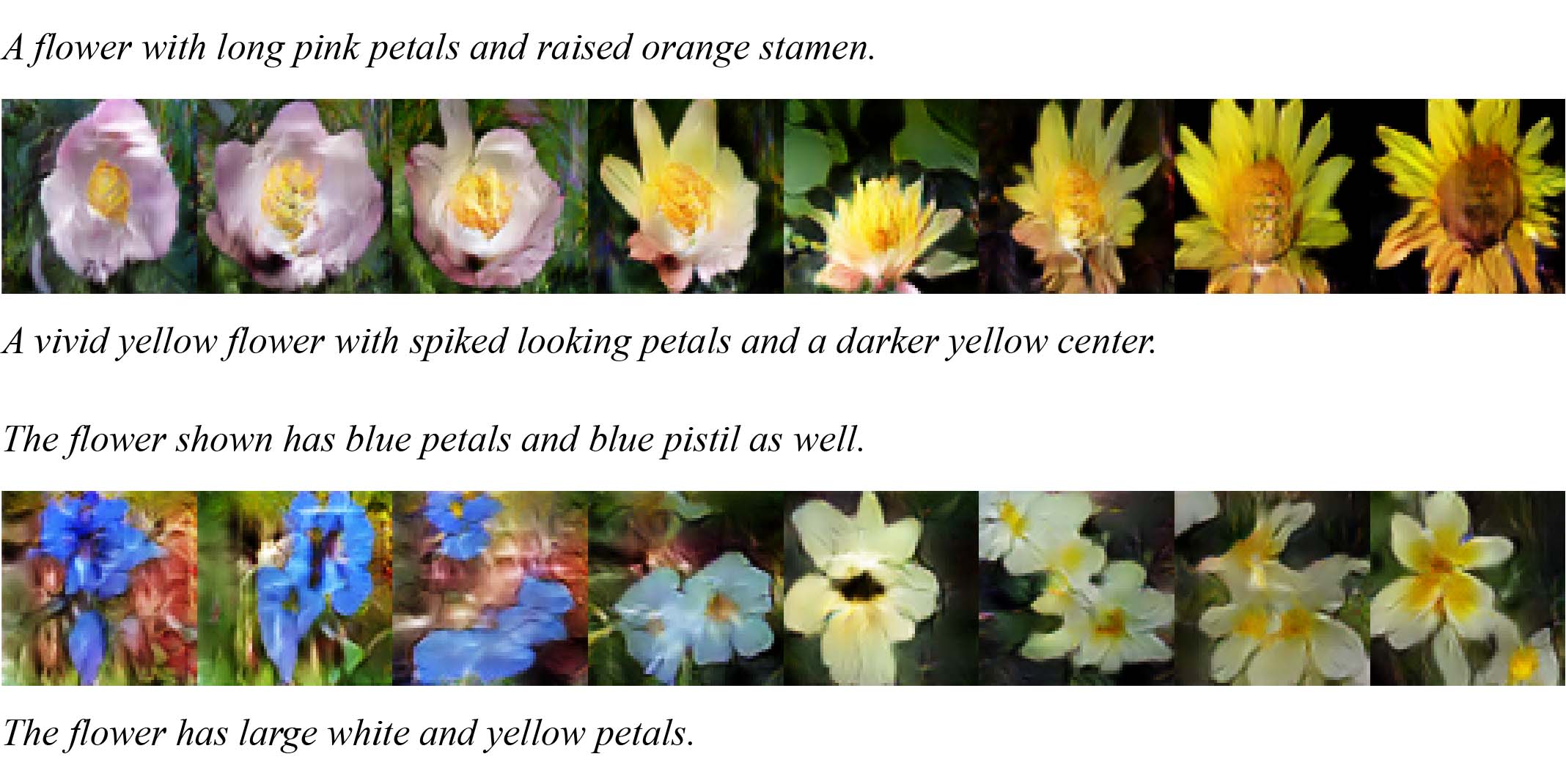}
    \caption[StackGAN Stage I interpolations]{Samples generated by Stage I of StackGAN from text embedding interpolations.}
    \label{fig:stage1-interp}
\end{figure}

\subsection{Training}

StackGAN uses the same discriminator loss function as the one in \ref{eq:cg} and the generator loss from \ref{eq:cond-aug-reg}.

For training, I used the Adam optimiser with a learning rate of 0.0002 for both networks, $\beta_1 = 0.5$ and $\beta_2 = 0.9$. I trained each of the stages for 600 epochs using a batch size of 64 for Stage I and a batch size of 32 for Stage II. When training Stage II, the parameters of Stage I are no longer trained. The learning rate is halved every 100 epochs as recommended in the paper. 

\begin{figure}[h]
    \centering
    \includegraphics[width=\textwidth]{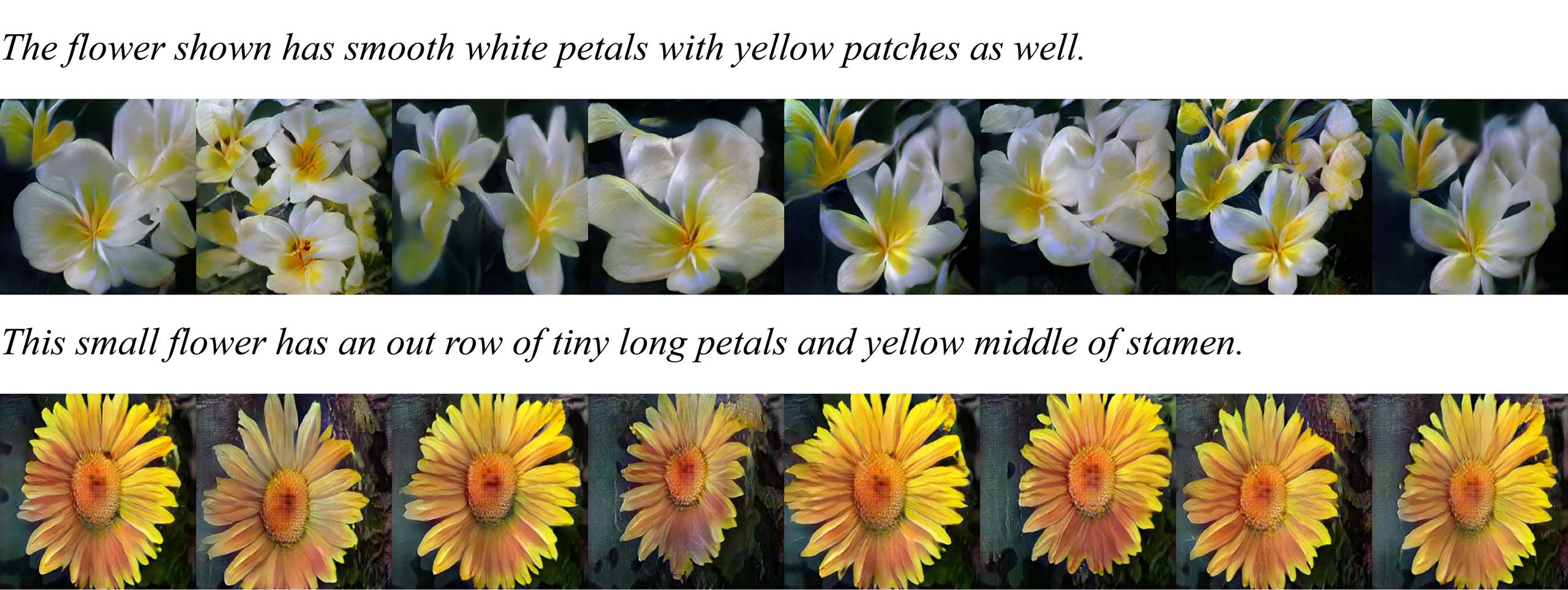}
    \caption[StackGAN Stage II Samples]{Samples generated by Stage II of StackGAN.}
    \label{fig:stageII-samples}
\end{figure}
\begin{figure}[h]
    \centering
    \includegraphics[width=\textwidth,height=0.9\textheight,keepaspectratio]{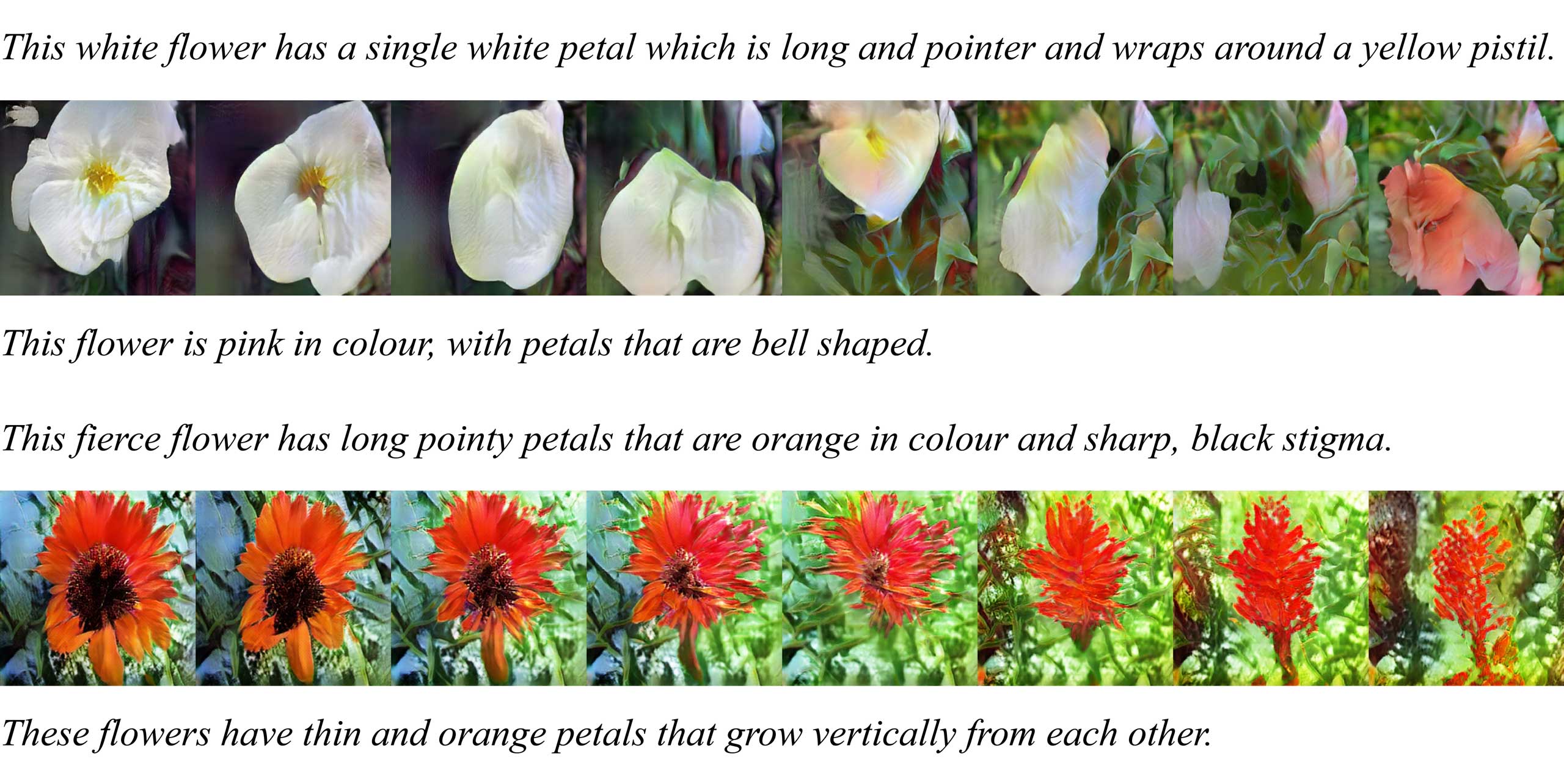}
    \caption[StackGAN Stage II interpolations]{Samples generated by Stage II of StackGAN from text embedding interpolations.}
    \label{fig:stage2-interp}
\end{figure}

\subsection{Results}

Figure \ref{fig:stageI-samples} shows samples generated by Stage I and Figure \ref{fig:stageII-samples} includes samples created by Stage II . The stochastic behaviour introduced by the augmentation of the text embeddings reflects in the higher image diversity of the generated images. Conditional interpolations for Stage I and Stage II are shown in Figures \ref{fig:stage1-interp} and \ref{fig:stage2-interp}. Figure \ref{fig:stackgan-stages} shows the images produced by the two stages for the same descriptions. Images generated by StackGAN on the birds dataset are included in Appendix \ref{appendix:birds}.

\begin{figure}[h]
    \centering
    \includegraphics[width=\textwidth,height=0.9\textheight,keepaspectratio]{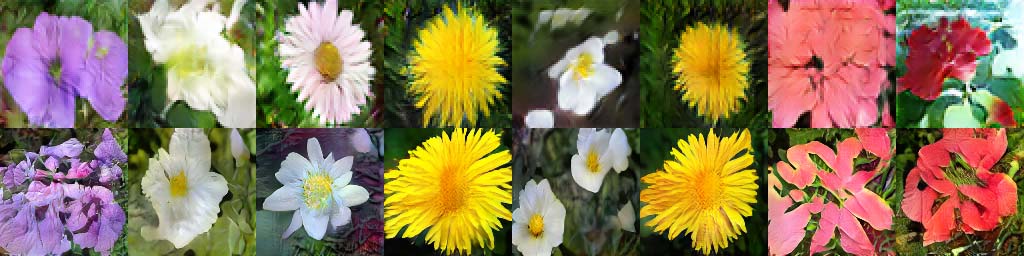}
    \caption[StackGAN Stages]{Images generated by Stage I (first row) and Stage II (second row) for the same text descriptions (one for each column). Stage II fine-tunes the images generated by Stage I.}
    \label{fig:stackgan-stages}
\end{figure}

\section{Other Models} \label{sec:other-models}

Other two state of the art models have been proposed since the start of this project: StackGAN-v2 \cite{DBLP:journals/corr/abs-1710-10916}, and more recently, AttnGAN \cite{DBLP:journals/corr/abs-1711-10485} developed by Microsoft Research in collaboration with other universities. StackGAN-v2, as the name suggests, is an improved version of StackGAN which uses multiple generators and discriminators in a tree-like structure. AttnGAN consists of an architecture similar to StackGAN-v2, but with an attention model \cite{DBLP:journals/corr/BaMK14, DBLP:journals/corr/MnihHGK14} on top of it. The attention model replicates the human attention mechanism and allows the network to focus on a single word from a sentence or a specific region of the image at a time. This ensures a granular image-word matching and not just a sentence level matching as it is the case with the other models discussed in this work.

%% file: chapter3.tex
\chapter{Research} \label{sec:research}

In this chapter, I propose new models which try to address some of the current research problems. In Section \ref{sec:wgan-cls} I propose Wasserstein GAN-CLS, a conditional Wasserstein GAN based on the recently introduced Wasserstein distance which offers convergence and stability guarantees. This model uses a novel loss function which achieves the text-image conditioning using a Wasserstein distance. In Section \ref{sec:pggan} I propose a conditional Progressive GAN inspired from \cite{pggan} which learns to generate images at iteratively increasing scale, and I show how the conditional Wasserstein loss improves this model.

\section{Wasserstein GAN-CLS} \label{sec:wgan-cls}

This section is more mathematical, in line with the firm theoretical arguments behind Wasserstein GANs. Additional explanations are included in Appendix \ref{appendix:wgan}.

The main problem of GANs is their instability during training. Perhaps counter-intuitively, as the discriminator becomes better, the generator's updates get worse. Arjovsky et al.\cite{Arjovsky2017TowardsPM} show that this problem is related to how the distances $d(\mathbb{P}_r, \mathbb{P}_g)$, which GANs commonly optimise, behave when the support of $\mathbb{P}_r$ and $\mathbb{P}_g$ are disjoint or lie on low dimensional manifolds. When that is the case, a perfect discriminator which separates them always exists. As the discriminator approaches optimality, the gradient of the generator becomes unstable if the generator uses the loss function $L_G$ from \ref{eq:cg}.

In many situations, it is likely that the two distributions lie on low dimensional manifolds. In the case of natural images, there is substantial evidence that the support of $\mathbb{P}_r$ lies on a low dimensional manifold \cite{NIPS2010_3958}. Moreover, Arjovsky et al. \cite{Arjovsky2017TowardsPM} prove that this is the case with $\mathbb{P}_g$ in the case of GANs. Thus, the choice of the distance $d(\mathbb{P}_r, \mathbb{P}_g)$ is crucial. One would like this function to be continuous and provide non-vanishing gradients that can be used for backpropagation even when this situation occurs.

Maximum likelihood models implicitly optimise $KL\infdiv{\mathbb{P}_r}{\mathbb{P}_g}$ (which is not a distance in the formal sense). GANs implicitly optimise the Jensen-Shannon divergence $JS\infdiv{\mathbb{P}_r}{\mathbb{P}_g}$, as shown in the proof of Theorem \ref{theo:1} from Appendix \ref{appendix:A}. Both of them are problematic. A simple example of two distributions whose supports are parallel lines \cite{Arjovsky2017TowardsPM} shows not only that these divergences (and others) are not differentiable, but they are not even continuous. In the next section, I discuss the Wasserstein distance which was proposed as a better choice for $d(\mathbb{P}_r, \mathbb{P}_g)$.

\subsection{Wasserstein GAN} \label{sec:wgan}

The Wasserstein distance, also known as the Earth Mover's (EM) distance, is theoretically proposed and analysed in GANs for the first time in \cite{Arjovsky2017TowardsPM}. In \cite{DBLP:conf/icml/ArjovskyCB17} it is shown in practice that a GAN which optimises the Wasserstein distance offers convergence and stability guarantees while producing good looking and more diverse images. This distance is given in Equation \ref{eq:wass}.
\begin{equation} \label{eq:wass}
    W(\mathbb{P}_r, \mathbb{P}_g) = \inf_{\gamma \in \Pi(\mathbb{P}_r, \mathbb{P}_g)} \mathbb{E}_{(\rvec{X},\rvec{Y}) \sim \gamma}[\norm{\vec{x} - \vec{y}}] \text{, }
\end{equation}
where $\Pi(\mathbb{P}_r, \mathbb{P}_g)$ is the set of joint distributions which have $\mathbb{P}_r$ and $\mathbb{P}_g$ as marginals and $\gamma_{\rvec{X}, \rvec{Y}}(\vec{x}, \vec{y}) \in \Pi(\mathbb{P}_r, \mathbb{P}_g)$ is one such distribution (see Figure \ref{fig:wdist} for an intuitive explanation). 

\begin{figure}[h]
    \centering
    \includegraphics[width=\textwidth]{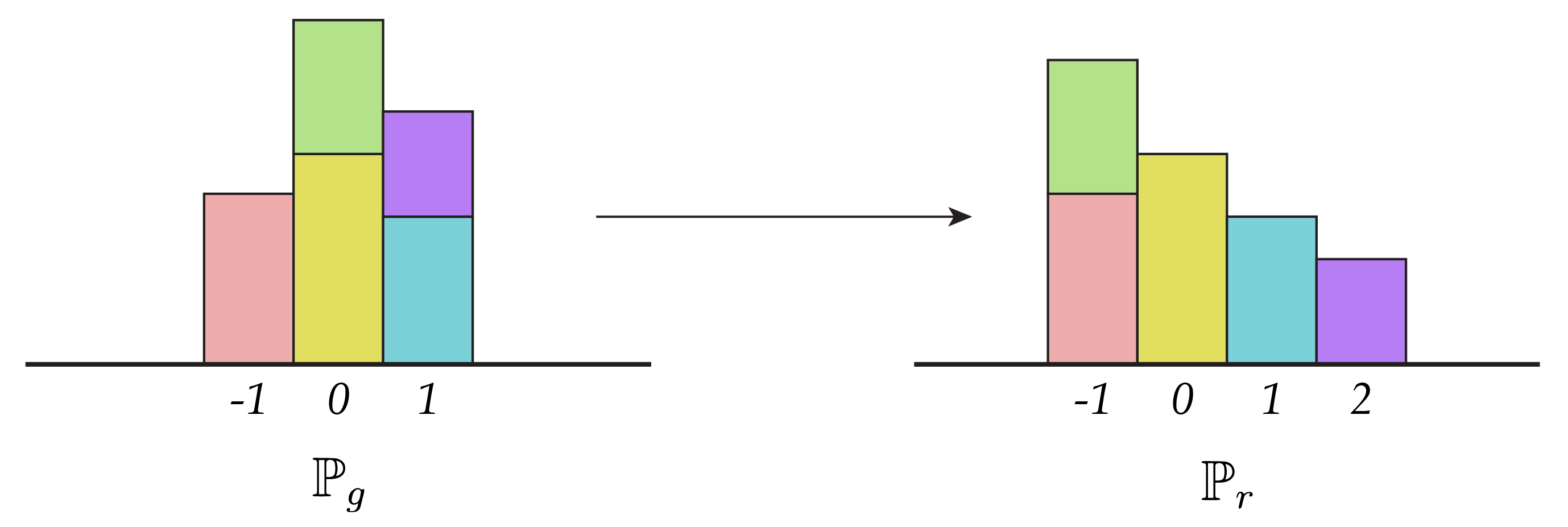}
    \caption[Wasserstein distance]{A better intuition for the Wasserstein distance can be developed by analysing a small discrete case. Given two discrete distributions $\mathbb{P}_r$ and $\mathbb{P}_g$, the Earth Mover's distance is the cost of the optimal plan to transport blocks of $\mathbb{P}_g$ to obtain $\mathbb{P}_r$ (or the other way around). A transport plan is optimal if it has minimum effort. The effort is proportional to the size of the blocks which are moved and the distance on which they have to be moved.}
    \label{fig:wdist}
\end{figure}

Of course, computing the Wasserstein distance in the form \ref{eq:wass} is intractable. Nevertheless, its dual form is tractable. This form is given by the Kantorovich-Rubinstein duality \cite{Villani2009} presented in Equation \ref{eq:wass-dual}.
\begin{equation} \label{eq:wass-dual}
    W(\mathbb{P}_r, \mathbb{P}_g) = \sup_{\norm{f}_{L} \leq 1} \mathbb{E}_{\rvec{X} \sim \mathbb{P}_r} [f(\vec{x})] - \mathbb{E}_{\rvec{X} \sim \mathbb{P}_g} [f(\vec{x})] \text{, }
\end{equation}
where the supremum is taken over the set of functions $f:\mathcal{X} \to \mathbb{R}$ which are 1-Lipschitz continuous (explained in Appendix \ref{appendix:wgan}). Instead of optimising relation \ref{eq:wass-dual} in function space, it can be optimised in the space of parametric functions using a neural network $D_{\vec{\omega}}: \mathcal{X} \to \mathbb{R}$. Equality \ref{eq:wass-dual} can be rewritten as:

\begin{equation}
    W(\mathbb{P}_r, \mathbb{P}_g) = \max_{\vec{\omega}, \norm{D_{\vec{\omega}}}_{L} \leq 1} \mathbb{E}_{\rvec{X} \sim \mathbb{P}_r} [D_{\vec{\omega}}(\vec{x})] - \mathbb{E}_{\rvec{X} \sim \mathbb{P}_g} [D_{\vec{\omega}}(\vec{x})]
\end{equation}

The only question which remains is how to enforce the Lipschitz constraint. It was originally suggested \cite{DBLP:conf/icml/ArjovskyCB17} to keep the weights $\vec{\omega} \in \mathcal{W}$, where $\mathcal{W}$ is a compact space such as $\mathcal{W} = [0.01, 0.01]^{N_{w}}$ and $N_{w}$ is the number of weights. Keeping the weights in such a small range indirectly constraints the rate of growth of the function which remains K-Lipschitz continuous over the course of training. The exact value of $K$ depends only on the choice of $\mathcal{W}$ and is independent of the values of $\vec{\omega}$. Nevertheless, this does not fully solve the problem because the small weights diminish the capacity of the neural network and also cause the training time to increase. A better solution was proposed \cite{gulrajani+al-2017-wasserstein-arxiv} which softens the constraint by appending it to the loss function as a regularisation term.

\begin{figure}[h] 
    \centering
    \includegraphics[width=\textwidth]{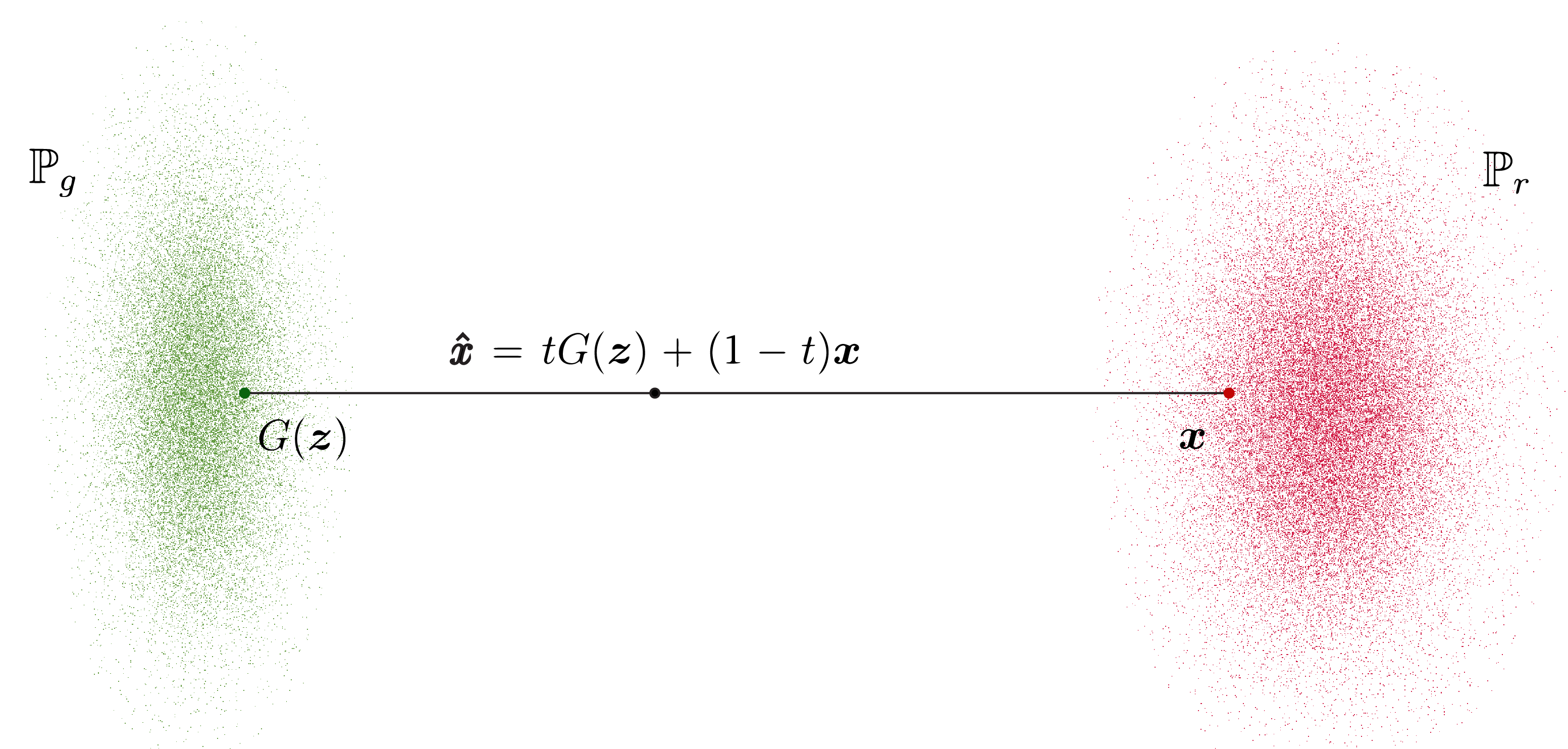}
    \caption[WGAN-GP linear interpolations]{Linear interpolation of a point from the dataset and a generated image. The gradient penalty ensures that the gradient norm remains close to one for such points between the two distributions.}
    \label{fig:wgangp_interp}
\end{figure}

It can be shown that a differentiable function is 1-Lipschitz if and only if its gradient norm is at most one almost everywhere. This motivates the loss function from Equation \ref{eq:wgan-gp} which adds a gradient penalty ($L_{GP}$) to penalise the network when the gradient norm goes far from one. 
\begin{equation} \label{eq:wgan-gp}
\begin{split}
        L_D &= \mathbb{E}_{\rvec{X} \sim \mathbb{P}_g} [D(\vec{x})] - \mathbb{E}_{\rvec{X} \sim \mathbb{P}_r}[D(\vec{x})] + \lambda L_{GP}
    \\ &= \mathbb{E}_{\rvec{X} \sim \mathbb{P}_g} [D(\vec{x})] - \mathbb{E}_{\rvec{X} \sim \mathbb{P}_r}[D(\vec{x})] + \lambda \mathbb{E}_{\hat{\rvec{X}} \sim \mathbb{P}_{\hat{\rvec{X}}}} [(\norm{\nabla D(\hat{\vec{x}})} - 1)^2] \text{, }
\end{split}
\end{equation}
where $\hat{\vec{x}}$ is a linear interpolation between a real and generated image: $\hat{\vec{x}} = t G(\vec{z}) + (1 - t) \vec{x}$ and $t$ is sampled from $U[0, 1]$. The model uses these interpolations because it is intractable to enforce the gradient constraint over the whole space $\mathcal{X}$. Instead, it is enforced only over the region between the two manifolds of the two distributions (Figure \ref{fig:wgangp_interp}). 

On the one hand, the discriminator is trained to better approximate the Wasserstein distance. The generator, on the other hand, tries to minimise $W(\mathbb{P}_r, \mathbb{P}_g)$, so the loss function from \ref{eq:wgan-min} is employed.
\begin{equation} \label{eq:wgan-min}
    L_G = - \mathbb{E}_{\rvec{X} \sim \mathbb{P}_g}[D(\vec{x})] \text{, }
\end{equation}

\subsection{Conditioning Wasserstein GAN} 

The loss function \ref{eq:wgan-gp} makes the discriminator distinguish between real and fake samples, but for text to image synthesis, it must also be text-image matching aware. By making the discriminator to also approximate $W(\mathbb{P}(\vec{x}, \vec{e})_{r-mat}, \mathbb{P}(\vec{x}, \vec{e})_{r-mis})$ between the joint distributions of matching and mismatching text-image pairs, the discriminator becomes matching aware. Based on this insight, I propose the loss function from \ref{eq:cond-wgan} for the discriminator. 
\begin{equation} \label{eq:cond-wgan}
\begin{split}
    L_D &= \{\mathbb{E}_{(\rvec{X}, \rvec{E}) \sim \mathbb{P}_{ge}} [D(\vec{x}, \vec{e})] - \mathbb{E}_{(\rvec{X}, \rvec{E}) \sim \mathbb{P}_{r-mat}}[D(\vec{x}, \vec{e})]\}
    \\ &+  \alpha\{\mathbb{E}_{(\rvec{X}, \rvec{E}) \sim \mathbb{P}_{r-mis}} [D(\vec{x}, \vec{e})] - \mathbb{E}_{(\rvec{X}, \rvec{E}) \sim \mathbb{P}_{r-mat}}[D(\vec{x}, \vec{e})]\}
    \\ &= \mathbb{E}_{(\rvec{X}, \rvec{E}) \sim \mathbb{P}_{ge}} [D(\vec{x}, \vec{e})] + \alpha \mathbb{E}_{(\rvec{X}, \rvec{E}) \sim \mathbb{P}_{r-mis}} [D(\vec{x}, \vec{e})] 
    \\ &- (1 + \alpha) \mathbb{E}_{(\rvec{X}, \rvec{E}) \sim \mathbb{P}_{r-mat}}[D(\vec{x}, \vec{e})] + \lambda L_{LP}
\end{split}
\end{equation}
where the parameter $\alpha$ controls the level of text-image matching.

Note that another regularisation term ($L_{LP}$), different from $L_{GP}$, is used to enforce the Lipschitz constraint. A potential problem of this loss function is that it can take values with a high magnitude on some datasets or architectures. Because nothing minimises $W(\mathbb{P}_{r-mat}, \mathbb{P}_{r-mis})$ as it is the case with $W(\mathbb{P}_{r-mat}, \mathbb{P}_{ge})$ which is being minimised by the generator, $W(\mathbb{P}_{r-mat}, \mathbb{P}_{r-mis})$ can theoretically take very high values. High values of this distance can damage the gradient penalty term whose proportion in the loss function will become so small that the gradient norm will get out of control. Theoretically, this can be fixed by simply increasing $\lambda$, but the regularisation of WGAN-GP from \ref{eq:wgan-gp} ($L_{GP}$) is not so robust \cite{petzka2018on} to changes in the values of the parameter $\lambda$. To address this, I use instead the regularisation term recently proposed in \cite{petzka2018on} which is called $L_{LP}$ (LP - Lipschitz Penalty). This term which does not penalise gradient norms less than one allows for larger values of $\lambda$ without harming the model. Moreover, empirical and theoretical evidence \cite{petzka2018on} shows that, under this softer regularisation term, convergence is faster and more stable. $L_{LP}$ is given by:
\begin{equation} \label{eq:L_LP}
    L_{LP} = \mathbb{E}_{(\hat{\rvec{X}}, \rvec{E}) \sim \mathbb{P}_{\eta}} [\max(0, \norm{\nabla_{\hat{\vec{x}}} D(\bar{\vec{x}}, \vec{e})} - 1)^2 + \max(0, \norm{\nabla_{\vec{e}} D(\hat{\vec{x}}, \vec{e})} - 1)^2]
\end{equation}
where I use $\mathbb{P}_{\eta}$ to denote the joint distribution of image text pairs $(\hat{\vec{x}}, \vec{e})$. $\hat{\vec{x}} = t G(\vec{z}, \vec{e}) + (1 - t) \vec{x}$ is a linear interpolation with $t$ sampled from $U(0, 1)$ and $\vec{e}$ is a matching text embedding of the image $\vec{x}$. Note that because $D(\hat{\vec{x}}, \vec{e})$ is also a function of the text embeddings $\vec{e}$ in this case, the Lipschitz constraint needs to be enforced with respect to the input $\vec{e}$ as well, not only $\hat{\vec{x}}$, hence the second term of the summation. 

Regarding the generator, I use the same cost function as the one in \ref{eq:wgan-min} with the addition of the text augmentation loss which softly maintains the standard normal distribution over the conditional latent space as described in Section \ref{sec:text-processing}. Thus, the loss of the generator is:
\begin{equation}
    L_G = - \mathbb{E}_{(\rvec{X}, \rvec{E}) \sim \mathbb{P}_g}[D(\vec{x}, \vec{e})] + \mathbb{E}_{\rvec{T} \sim \mathbb{P}_r}[\rho KL\infdiv{\mathcal{N}(\vec{0}, \matr{I})}{\mathcal{N}(\mu(\phi(\vec{t})), \Sigma(\phi(\vec{t}))}]
\end{equation}

\subsection{Architecture}

To make comparisons simpler, I keep the architecture of the generator identical to that of Stage I of StackGAN. In the case of the discriminator, I remove the batch normalisation for the gradient penalty to work. The gradient penalty assumes a unique gradient for each sample, and this assumption no longer holds in the presence of batch normalisation \cite{gulrajani+al-2017-wasserstein-arxiv}. 

Because, when the Wasserstein distance is used, the discriminator no longer needs to be crippled to keep the training balanced, I add one more convolutional layer in the discriminator after the text embedding concatenation, and I use the same number of convolutional filters as in the generator. 

\subsection{Training}

In the case of Wasserstein GANs, the closer the discriminator gets to optimality for a fixed $G$, the better the approximation of $W(\mathbb{P}_{r-mat}, \mathbb{P}_{ge})$ and  $W(\mathbb{P}_{r-mat}, \mathbb{P}_{r-mis})$ is. The generator's updates will also be better. That is why the discriminator is trained for $n_{critic}$ times for every generator update, where $n_{critic}$ is a hyper-parameter to be set at training time. A common value is $n_{critic} = 5$.

\begin{figure}[h]
    \centering
    \includegraphics[width=\textwidth]{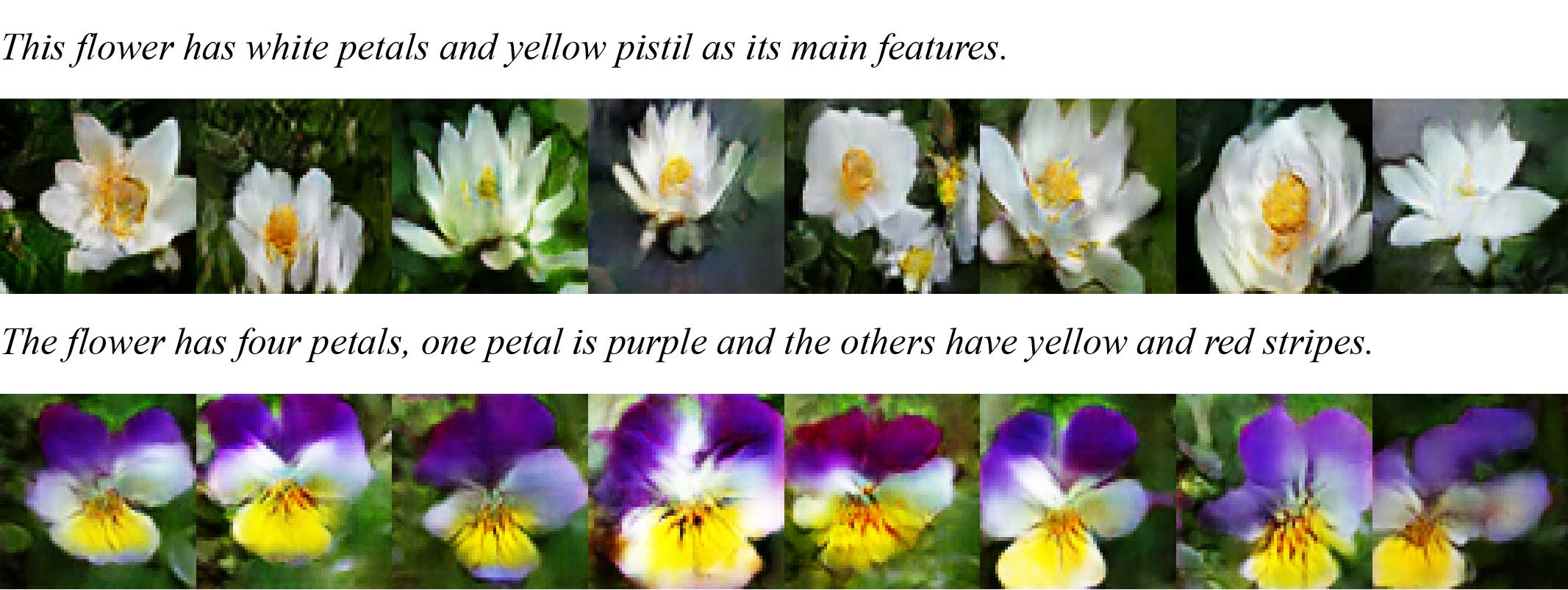}
    \caption[Conditional Wasserstein GAN samples]{Samples generated by WGAN-CLS from the same text descriptions but different noise vectors}
    \label{fig:wgan_samples}
\end{figure}

For faster training, I take a slightly different approach by setting $n_{critic} = 1$ in conjunction with the usage of the Two-Timescale Update Rule (TTUR) \cite{NIPS2017_7240}. TTUR refers to the usage of two different learning rates for the generator and the discriminator, which guarantees that the networks will reach a local Nash equilibrium. The convergence under TTUR is shown to be faster and the quality of the images higher than in the case of the classic method of training. Thus, I use the Adam optimiser again with a learning rate of 0.0003 for the critic and 0.0001 for the generator. In the case of the other parameters of Adam, I use $\beta_{1} = 0$ and $\beta_2 = 0.99$ for both the generator and the discriminator. The generator regularisation parameter $\rho$ is set to $10$, $\lambda_{LP} = 150$ and $\alpha = 1$. The training is performed for 120,000 steps with a batch size of 64.

\subsection{Results}

Figure \ref{fig:wgan_samples} shows samples generated by the Conditional Wasserstein GAN. Figure \ref{fig:wgan_interp} shows images generated from interpolations in the conditional space. 

\begin{figure}[h]
    \centering
    \includegraphics[width=\textwidth]{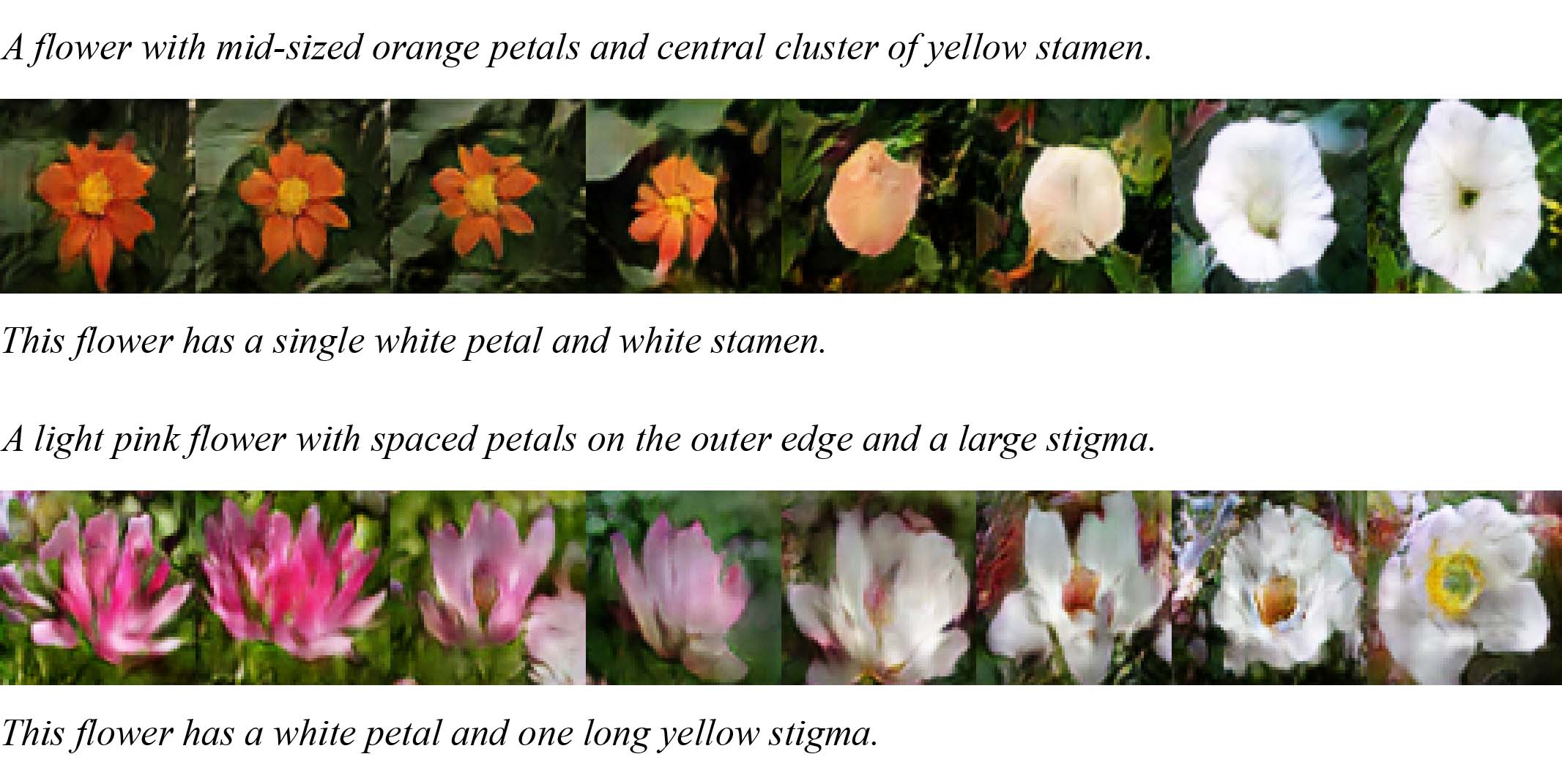}
    \caption[Conditional Wasserstein GAN interpolations]{Samples generated by the Conditional Wasserstein GAN from interpolations between two text embeddings.}
    \label{fig:wgan_interp}
\end{figure}

I subjectively assess that the quality of the generated images is comparable to the one of GAN-CLS and Stage I of StackGAN. My subjective evaluation is also confirmed by the Inception Scores of the models which are given in Chapter \ref{sec:eval}.

\section{Conditional Progressively Growing GANs} \label{sec:pggan}

In this section, I show how the recently introduced Progressive Growing GAN (PGGAN) \cite{pggan} can be turned into a conditional model for text to image synthesis. Moreover, I show how the Wasserstein critic loss I proposed in the previous section can improve this conditional model.

\subsection{Architecture and Training}

Because the training of this model is tightly integrated with its architecture, I treat them together in this section rather than separately.

The generator starts by concatenating a noise vector $\vec{z}$ with an augmented embedding $\vec{e}$. This concatenated vector is then projected into a tensor of dimension 4$\times$4$\times$512 which is then followed by two more convolutional layers with a filter size of 3$\times$3. Together, they constitute the first stage of the generator. The output of this stage is supplied as input to a stack of other stages, all separated by a nearest neighbour upsampling layer which upscales by a factor of two. All the generator stages excepting the first are composed of two convolutional layers with a kernel size of 3x3. In the end, the output tensor is passed through two more convolutional layers called the ``toRGB'' module which outputs the actual RGB image.

\begin{figure} 
    \centering
    \includegraphics[width=\textwidth]{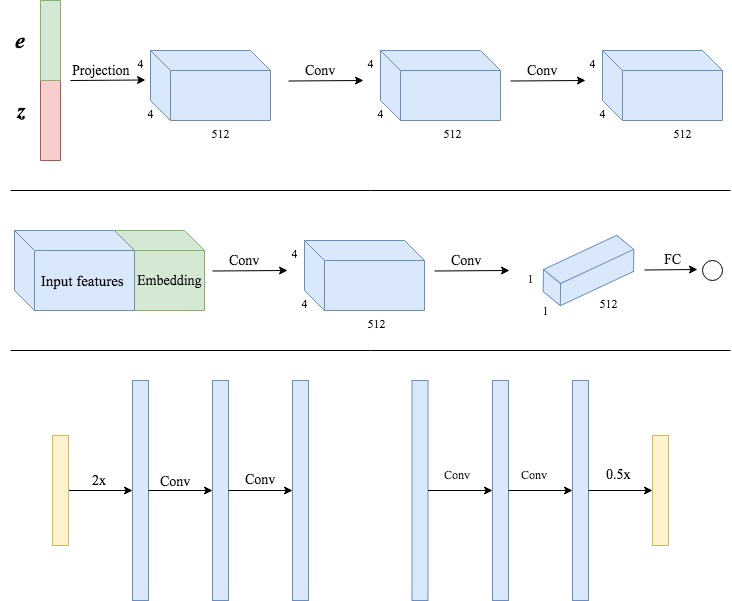}
    \caption[Conditional PGGAN stages]{Stages of the conditional PGGAN: (Top) First stage of the generator, (middle) First stage of the discriminator, (bottom left) the architecture of the other stages of the generator, (bottom right) the architecture of the other stages of the discriminator.}
    \label{fig:pggan-stages}
\end{figure}

The discriminator starts with a convolutional layer which produces a first set of convolutional features without affecting the spatial resolution. This layer is denoted as the ``fromRGB'' module. These features are given as input to a stack of stages which again are added step by step as training progresses concurrently with the generator stages. All stages have two convolutional layers, symmetric to the ones of the generator. The only exception is the first stage where the compressed embeddings are concatenated in depth to the input features of that stage similarly to the previous GANs. The concatenated block is processed by two more convolutional layers followed by a fully connected layer which produces the scalar discriminator output. The discriminator stages are separated by an average pooling layer which reduces the resolution by a factor of two. Figure \ref{fig:pggan-stages} shows the structure of all the stages.

\begin{figure}
    \centering
    \includegraphics[width=\textwidth]{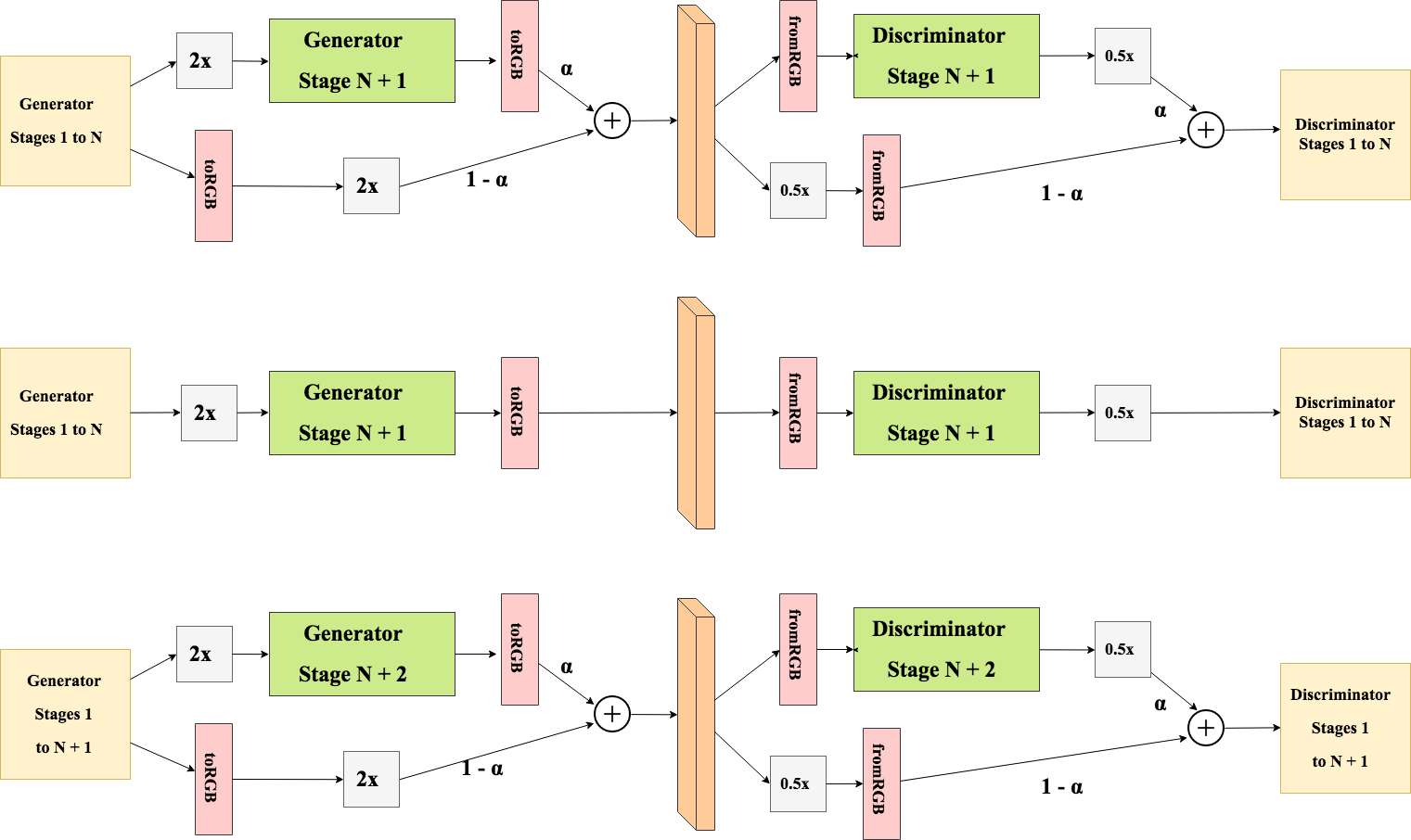}
    \caption[PGGAN interpolation]{Three consecutive training phases for the conditional PGGAN: the transition phase of an arbitrary $N + 1$ stage (top), the stabilisation phase of the same stage (middle), the transition phase of stage $N + 2$. The parameter $\alpha$ is linearly increased from zero to one during the transition phase. At the end of the transition phase when $\alpha = 1$, the new stages are fully attached to the previous stages. Next, the networks stabilise in their new configuration during the stabilisation phase. After the new stages are stabilised, the transition phase of the next stage begins.}
    \label{fig:pggan-train}
\end{figure}

The novelty of the model consists in the way the networks transform during training. The first stages of both networks are trained first using images of 4$\times$4 in resolution. Then the second stage is introduced concurrently for the discriminator and the generator as a residual layer with an associated weight $\alpha = 0$ to avoid perturbing the network. This stage doubles the resolution to 8$\times$8. The addition of any new stage starts with a transition phase. During a transition phase, $\alpha$ linearly increases to one, and the model smoothly learns to adapt to the new stages and the enlarged image size. This weight has the effect of interpolating between the scaled output of the previous stage and the output of the new stage in the case of the generator. For the critic, it is an interpolation of the inputs. When the transition phase is over, a stabilisation phase follows to stabilise the network in its new configuration with $\alpha=1$. Each transition and stabilisation phase lasts until the discriminator sees 600,000 real images and 600,000 generated images. This process of iteratively adding stages repeats until the desired resolution is reached or as longs as the GPU memory and the resolution of the images in the dataset allow it. When scaling up to a new stage, all the stages are trained, including the previous ones. The training process is also shown and further explained in Figure \ref{fig:pggan-train}. More details are given in Appendix \ref{appendix:cpggan}.

As shown in \cite{pggan}, this method of training is significantly faster than training the complete network from scratch because the majority of the training time is spent at the lower stages.

\subsection{The Need for a Stable Loss Function}

This architecture and unusual method of training do not work with any loss function given the instability of GANs. The PGGAN paper \cite{pggan} empirically shows PGGAN working with a least squares loss and a Wasserstein distance loss. While the least squares loss \cite{lsgan} is empirically known to be more stable than the classic GAN loss, the Wasserstein loss has technical reasons behind which guarantee its stability. 

For the least squares loss, as with the Wasserstein loss, the discriminator no longer outputs a probability, but an arbitrary real number. The generator and the discriminator optimise for making this real number close to some predefined labels $a, b, c$. The general form of the least squares loss is as follows:
\begin{equation}
\begin{split}
     L_D &= \mathbb{E}_{\rvec{X} \sim \mathbb{P}_r}[(D(\vec{x}) - b)^2] + \mathbb{E}_{\rvec{X} \sim \mathbb{P}_g}[(D(\vec{x}) - a)^2]
     \\ L_G &= \mathbb{E}_{\rvec{X} \sim \mathbb{P}_g}[(D(\vec{x}) - c)^2] 
\end{split}
\end{equation}

\begin{figure}[h]
    \centering
    \includegraphics[width=\textwidth]{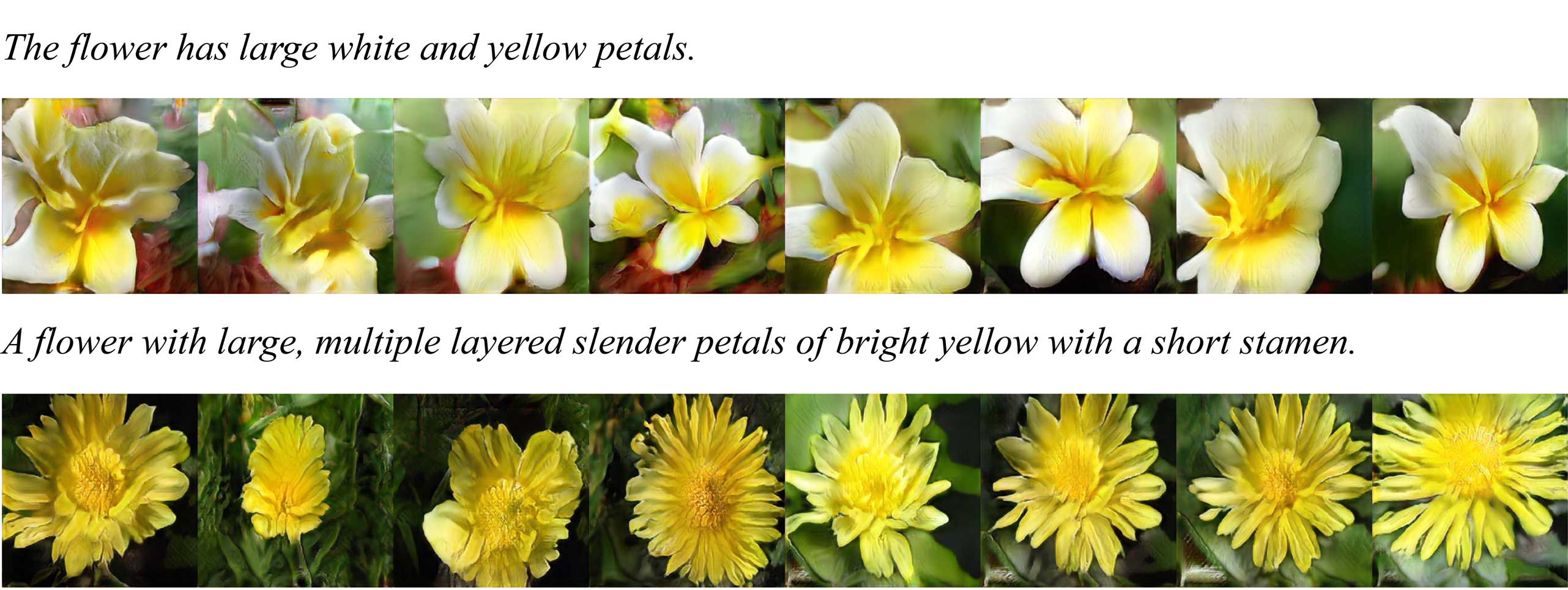}
    \caption[Least Squares CPGGAN samples]{Samples generated by the Conditional Least Squares PGGAN.}
    \label{fig:lspggan_samples}
\end{figure}

As shown in \cite{lsgan}, when $b - c = 1$ and $b - a = 2$, minimising these cost functions is equivalent to minimising the Pearson $\mathcal{X}^2$ divergence between $\mathbb{P}_r + \mathbb{P}_g$ and $2\mathbb{P}_g$. This justifies the choice of labels $a = -1, b = 1, c = 0$ which I use for my experiments. It is trivial to adapt this loss function to make the discriminator matching aware as follows:
\begin{equation}
\begin{split}
     L_D &= \mathbb{E}_{(\rvec{X}, \rvec{E}) \sim \mathbb{P}_{r-mat}}[(D(\vec{x}, \vec{e}) - b)^2] + \mathbb{E}_{(\rvec{X}, \rvec{E}) \sim \mathbb{P}_{ge}}[(D(\vec{x}, \vec{e}) - a)^2] 
     \\ &+ \mathbb{E}_{(\rvec{X}, \rvec{E}) \sim \mathbb{P}_{r-mis}}[(D(\vec{x}, \vec{e}) - a)^2]
     \\ L_G &= \mathbb{E}_{(\rvec{X}, \rvec{E}) \sim \mathbb{P}_{ge}}[(D(\vec{x}, \vec{e}) - c)^2] 
\end{split}
\end{equation}

\begin{figure}[h]
    \centering
    \includegraphics[width=\textwidth]{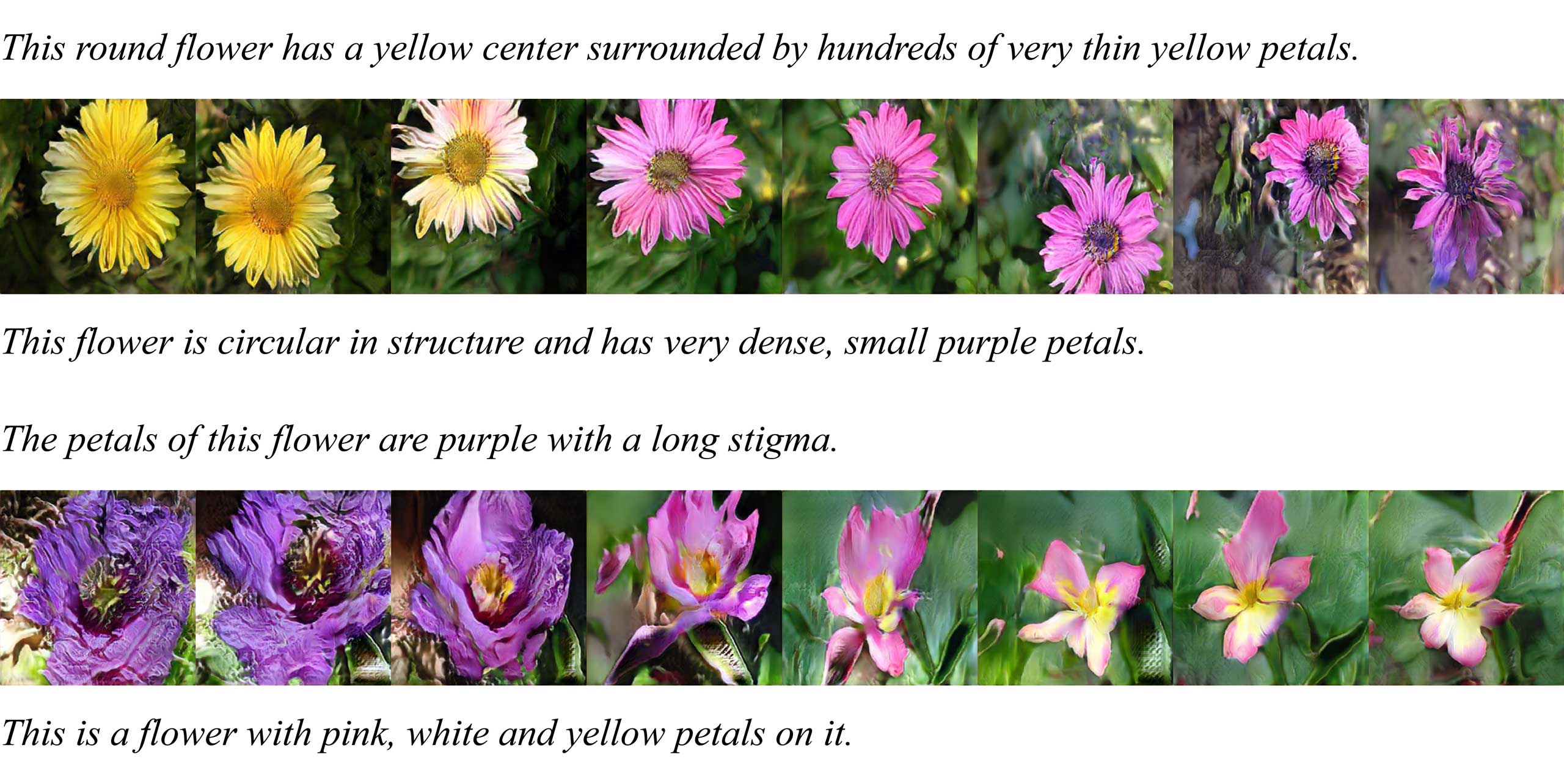}
    \caption[Least Squares CPGGAN interpolations]{Samples generated by the Conditional Least Squares PGGAN from text embedding interpolations.}
    \label{fig:lspggan_interp}
\end{figure}

Now, the discriminator will push towards $a$ not only the synthetic images but also images which do not match their description.

Nevertheless, I find the least squares loss to be unstable when the network reached the high-resolution stages, which is consistent with the findings from \cite{pggan}. As in the paper, I introduce multiplicative Gaussian noise between the layers of the discriminator to eliminate the instability. This hack does not address the cause of the problem, which is the loss function. The Conditional Progressive Growing GAN is a perfect use case for the Wasserstein based loss I proposed in Section \ref{sec:wgan-cls} because it is guaranteed to be stable. Results for both of these losses are discussed in the next section.

\subsection{Results}

Figure \ref{fig:lspggan_samples} includes 256$\times$256 samples generated by the Conditional Least Squares PGGAN (CLSPGGAN). Figure \ref{fig:lspggan_interp} includes images generated by the same model from text embedding interpolations. Figures \ref{fig:wpggan_samples} and \ref{fig:wpggan_interp} include the equivalent images generated by the Conditional Wasserstein PGGAN (CWPGGAN) which use the Wasserstein loss I proposed in section \ref{sec:wgan-cls}.

Images generated by CWPGGAN on the birds dataset are included in Appendix \ref{appendix:birds}.

\begin{figure}[h]
    \centering
    \includegraphics[width=\textwidth]{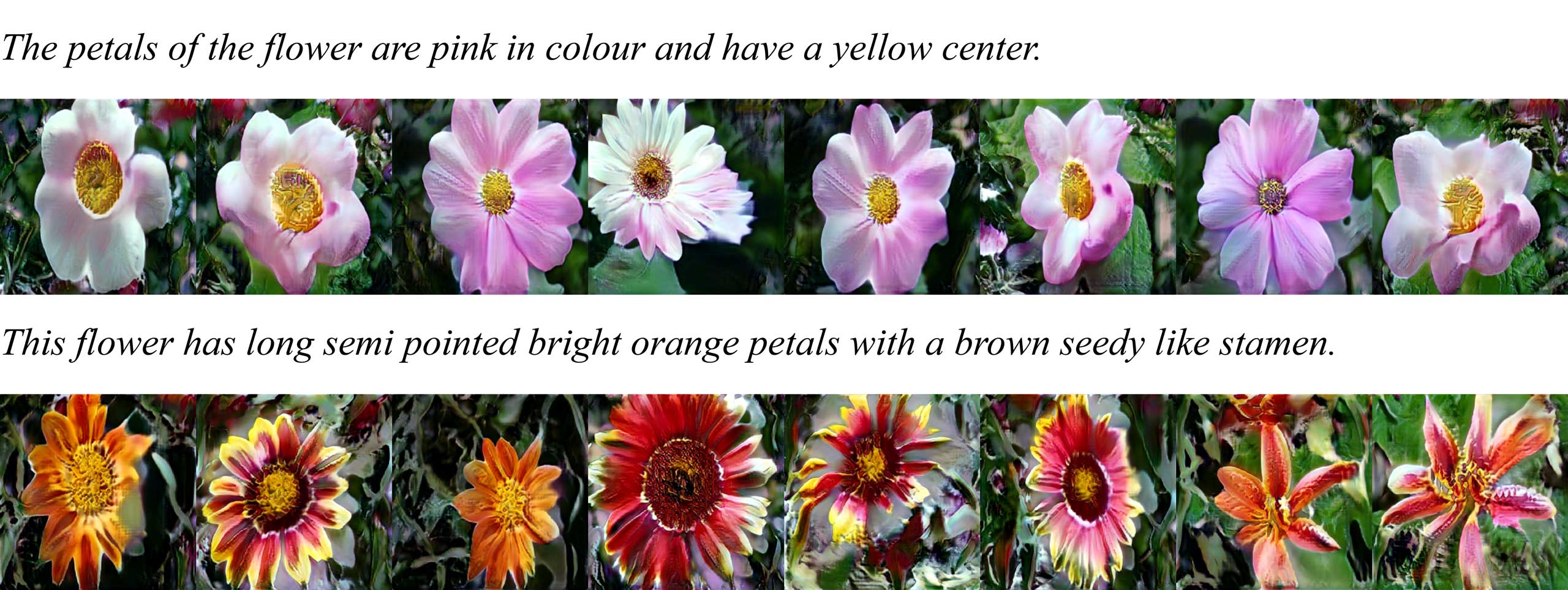}
    \caption[Conditional Wasserstein PGGAN samples]{Samples generated by the Conditional Wasserstein PGGAN}
    \label{fig:wpggan_samples}
\end{figure}
\begin{figure}[h]
    \centering
    \includegraphics[width=\textwidth]{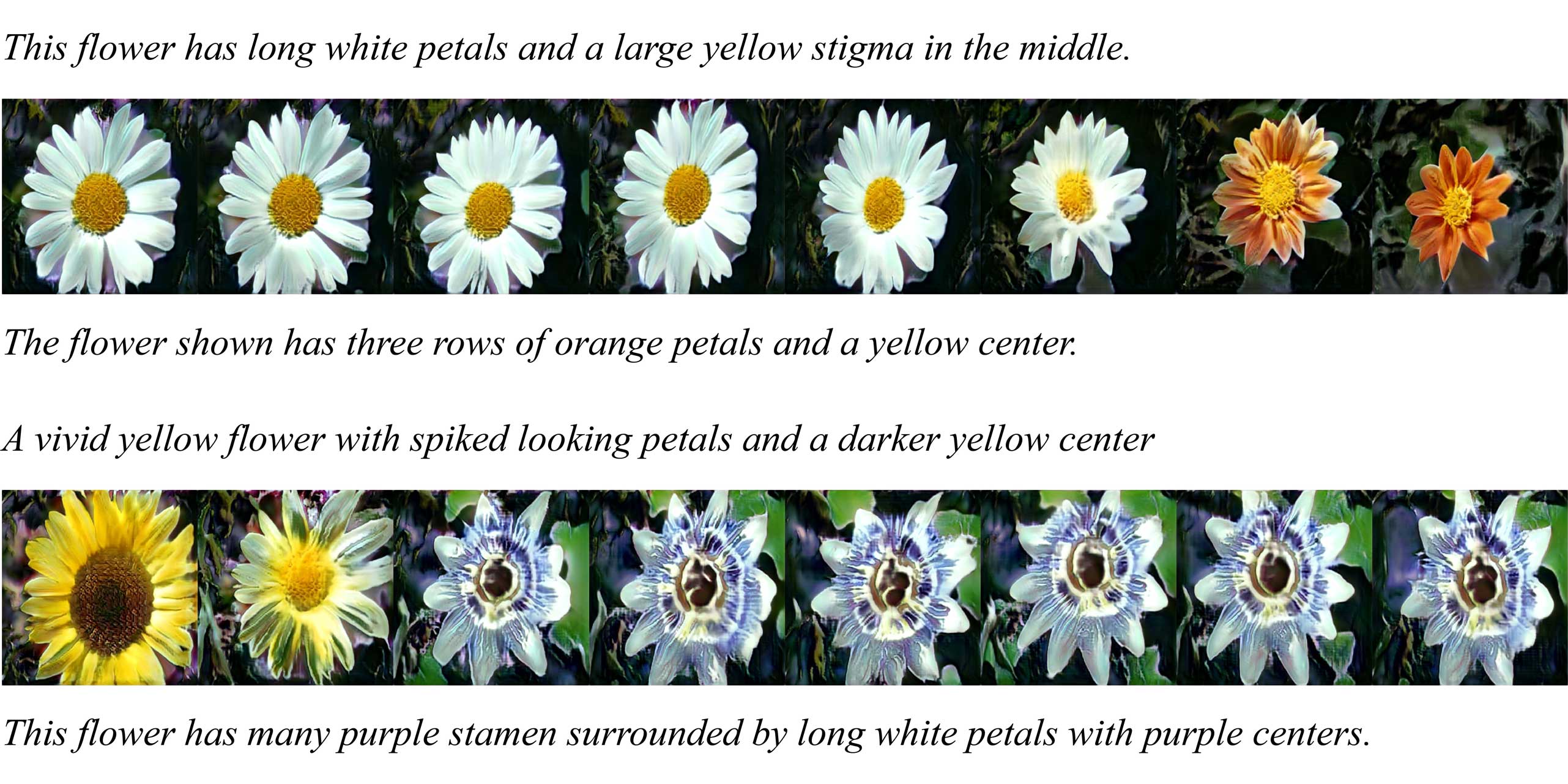}
    \caption[Conditional Wasserstein PGGAN interpolations]{Samples generated by the Conditional Wasserstein PGGAN from interpolations between two text embeddings.}
    \label{fig:wpggan_interp}
\end{figure}

Figures \ref{fig:lspggan_stages} and \ref{fig:wpggan_stages} show the images generated by each stage of Conditional PGGAN for the Least Squares loss and the Wasserstein loss respectively.

\begin{figure}[h]
    \centering
    \includegraphics[width=\textwidth]{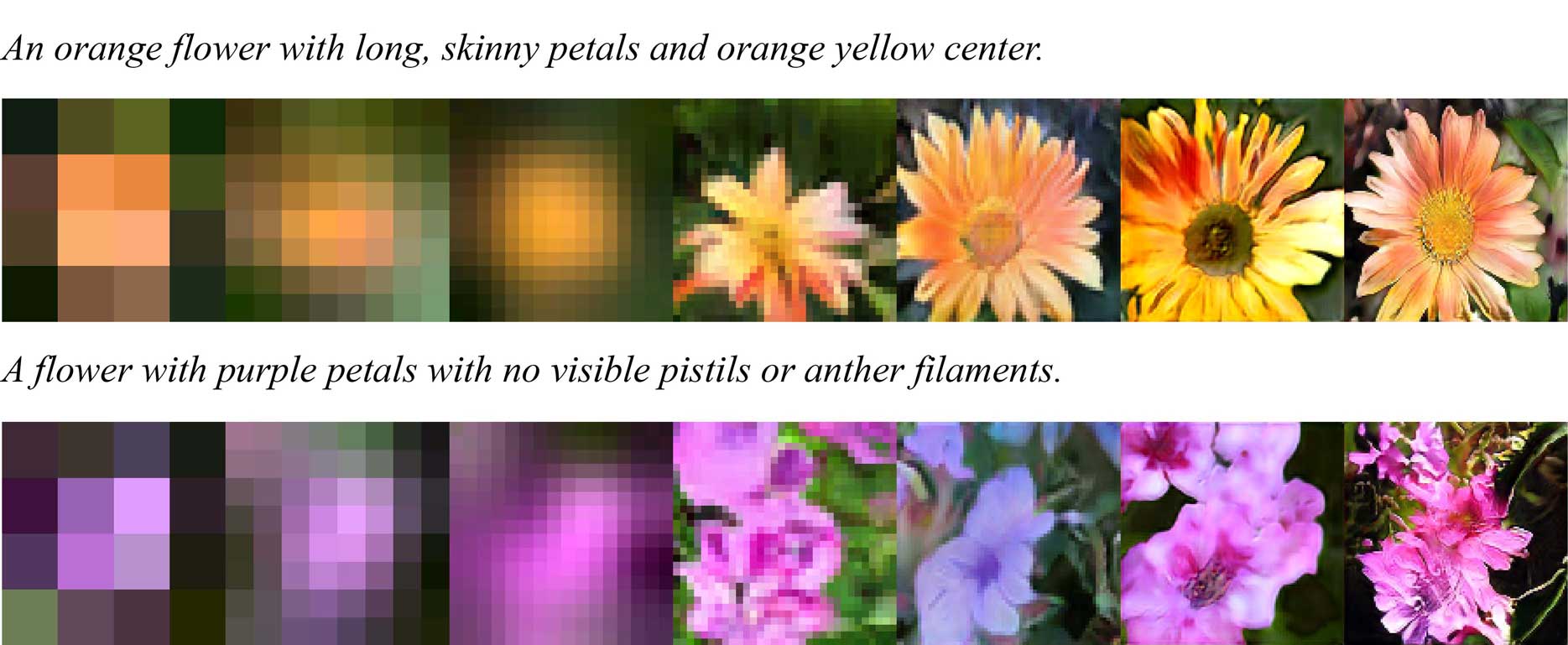}
    \caption[Conditional Least Squares PGGAN stages]{The image generated by each stage of the Least Squares Conditional PGGAN for the same text description. The images range from resolutions 4$\times$4 to 256$\times$256. Each stage doubles the resolution.}
    \label{fig:lspggan_stages}
\end{figure}
\begin{figure}[h]
    \centering
    \includegraphics[width=\textwidth]{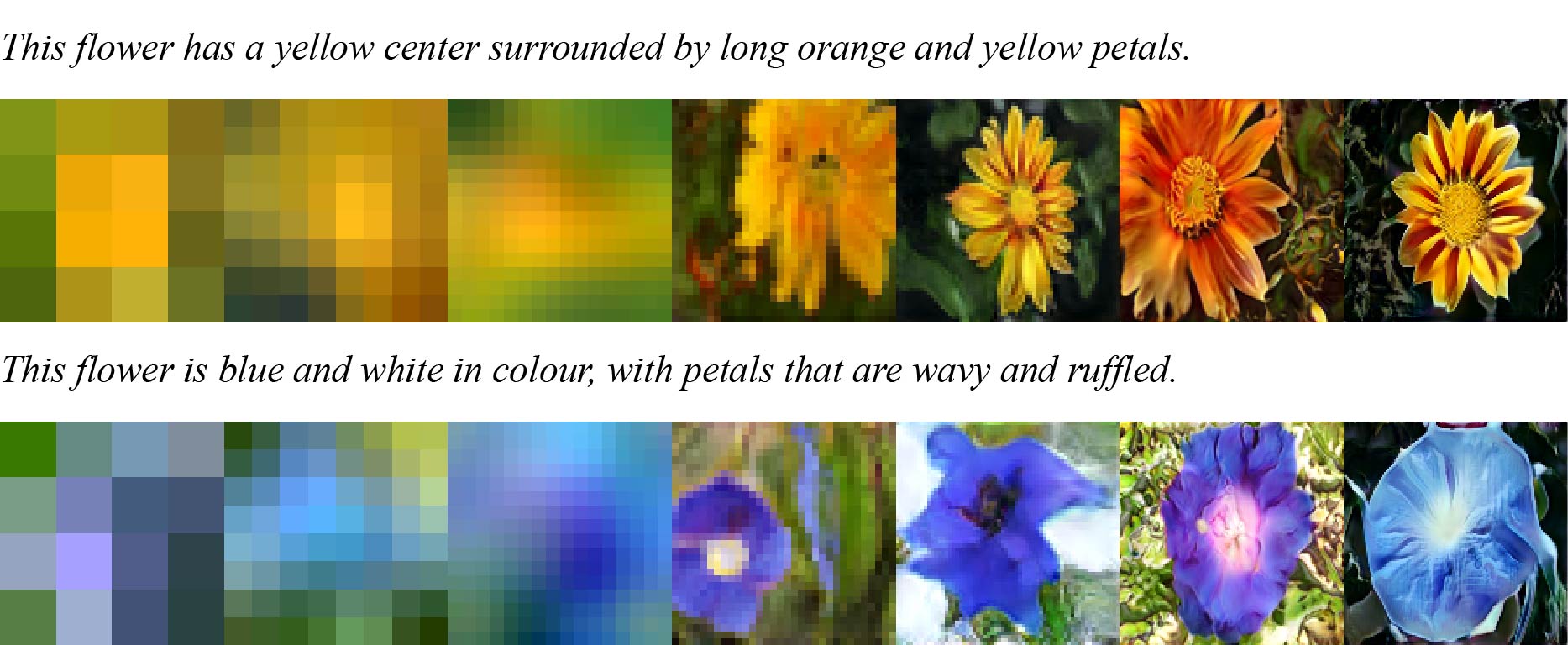}
    \caption[Conditional Wasserstein PGGAN stages]{The image generated by each stage of the Wasserstein Conditional PGGAN for the same text description. The images range from resolutions 4$\times$4 to 256$\times$256. Each stage doubles the resolution.}
    \label{fig:wpggan_stages}
\end{figure}

%% file: chapter4.tex
\chapter{Evaluation and Comparisons} \label{sec:eval}

In this chapter, I present the Inception Score \cite{Salimans2016ImprovedTF}, the current standard technique for evaluating GANs. Then, I show the score for all of the models introduced in this work as well as a side by side comparison of the images they produce.

\section{The Inception Score}

The evaluation of generative models is a current area of research. Because most generative models maximise the likelihood of the data, they are evaluated using the average log-likelihood as a metric. As previously discussed, GANs depart from this approach and thus perform better, but at the same time, this also makes their evaluation harder. A recently proposed way of evaluating GANs which generate images is the Inception Score \cite{Salimans2016ImprovedTF}. 

\begin{figure}
    \centering
    \includegraphics[width=\textwidth]{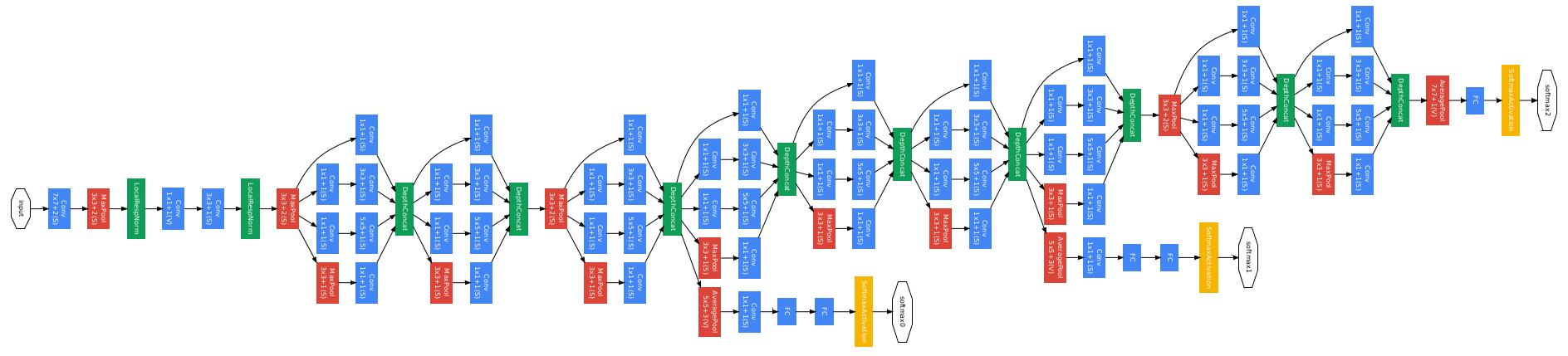}
    \caption[Inception architecture]{The architecture of the inception network (image taken from \cite{7298594inception}). The bigger blue blocks are convolutions, the smaller blue blocks are fully connected layers, the red blocks are pooling layers, and the yellow blocks are softmax layers which convert the layer input values in a valid probability distribution. The two bottom branches which separate from the main part of the network are auxiliary classifiers, which are used for better gradient propagation.}
    \label{fig:inception}
\end{figure}

The name of the score comes from Google's Inception classifier \cite{7298594inception} (Figure \ref{fig:inception}). Treating images as a random vector $\bm{X}$ and the image labels as a random variable $Y$, the Inception network produces a distribution $P_{Y \vert \bm{X}}$ where $P_{Y \vert \bm{X}}(y \vert \vec{x})$ is the probability assigned to image $\vec{x}$ to belong to class $y$. An Inception network is trained to produce such probabilities for the classes from the test dataset the GAN will be evaluated on. This assumes a dataset divided into classes. Then, the trained network classifies the images generated by the model being evaluated. The score is a function of the distribution of the predicted classes. There are two desired outcomes:

\begin{enumerate}
    \item The object in any image $\vec{x}$ should be undoubtedly identifiable. In other words, the conditional distribution $P_{Y \vert \bm{X}}$ should have low entropy.
    
    \item All the generated images should be as diverse as possible. That is, the images should not belong to just a small subset of classes but all the classes in the dataset. Equivalently, the distribution $P_Y$ should have high entropy.
\end{enumerate}

These two aspects motivate the form of the Inception Score from \ref{eq:inception} because if they hold, then the $KL$ divergence between the two mentioned distributions is high. The exponential function is used only for aesthetic reasons to bring the values of the score in a higher range of values. 

\begin{equation} \label{eq:inception}
    IS(G) = \exp(\mathbb{E}_{\bm{X} \sim P_g}[KL\infdiv{P_{Y \vert \bm{X}}(y \vert \vec{x})}{P_{Y}(y)}])
\end{equation}

The Inception score was shown to correlate well with human evaluation of image quality \cite{Salimans2016ImprovedTF}. This is the reason I chose not to conduct a human assessment for evaluating the models presented in this work. 

For training I use the Inception-V3 architecture \cite{Szegedy2016RethinkingTI}, a variant of the architecture shown in Figure \ref{fig:inception}. Instead of fully training the network, I only fine tune it. I train only the ``Logits'' and ``Mixed\_7c'' variable scopes and for the other layers, I use the publicly available weights trained on ImageNet \cite{imagenet_cvpr09}. This follows the approach from StackGAN \cite{han2017stackgan}. For computing the Inception Score, I use a group of $50,000$ generated images which are split randomly into ten equal sets as recommended in \cite{Salimans2016ImprovedTF}. The inception score is computed for each of these sets and the mean value together with the standard deviation are reported.

\subsection{Evaluation of Text-Image Matching}

The Inception score in its default form measures only the image quality, but it can also be used as an implicit measure of text-image matching. The Inception network is trained on the test dataset which (very importantly) contains classes disjoint from those in the training dataset. The generated images which are evaluated are produced exclusively from text descriptions from the test dataset. Because the training and test datasets contain disjoint classes, neither the text descriptions nor the images from the test dataset (or similar ones) are seen by the model in training. To generate high Inception Scores, the model must create images similar to the ones from the test dataset. The only possibility for the model to do this is to learn the visual semantics of the text descriptions correctly and to generate high-quality images which respect those descriptions. Thus, the reported Inception-Score is a measure of both image quality and text-image matching. The StackGAN paper \cite{han2017stackgan} uses the same approach in its evaluation.

\section{Inception Score Results}

I include the Inception Score means and standard deviations for all models on the Oxford-102 flowers dataset in Table \ref{tab:inception-scores}. The results show that the proposed Conditional Wasserstein GAN obtains comparable results to other state of the art models which produce 64$\times$64 images while maintaining the same generator as Stage I of StackGAN. Moreover, the Conditional Wasserstein GAN achieves this score in conditions of guaranteed training stability which is very important. The proposed Conditional Progressive Growing GAN achieves a better score than the other models for both resolutions on the flowers dataset. Moreover, the model obtains the best score in combination with the Wasserstein loss I proposed in Section \ref{sec:wgan-cls}.

\begin{table}[h]
    \centering
    \begin{tabular}{| l | l | l |}
    \hline
    Model & Resolution & Score \\ \hline
    Customised GAN-CLS & 64$\times$64 & 3.11 $\pm$ 0.03 \\ 
    StackGAN Stage I & 64$\times$64 & 3.42 $\pm$ 0.02 \\ 
    WGAN-CLS$^*$ & 64$\times$64 & 3.11 $\pm$ 0.02 \\ 
    WGAN-CLS with TTUR$^*$ & 64$\times$64 & 3.20 $\pm$ 0.01 \\ 
    CLSPGGAN$^*$ & 64$\times$64 & 3.44 $\pm$ 0.04 \\ 
    CWPGGAN$^*$ & 64$\times$64 & \textbf{3.70 $\pm$ 0.03} \\ \hline
    StackGAN Stage II & 256$\times$256 & 3.71 $\pm$ 0.04 \\ 
    CLSPGGAN$^*$ & 256$\times$256 & 3.76 $\pm$ 0.03 \\ 
    CWPGGAN$^*$ & 256$\times$256 & \textbf{3.86 $\pm$ 0.02} \\ \hline

    \end{tabular}
    \caption[Flowers dataset Inception Scores]{Inception Scores for the Oxford-102 flowers dataset. Models marked with $^*$ are the models proposed in this report.}
    \label{tab:inception-scores}
\end{table}




On the birds dataset, I run limited experiments for CWPGGAN and StackGAN (Appendix \ref{appendix:birds}). To quickly evaluate CWPGGAN against the other models, including the recently introduced StackGAN-v2 and AttnGAN, I used directly the scores given in the AttnGAN paper \cite{DBLP:journals/corr/abs-1711-10485} for the birds dataset. The flowers dataset is not used in the paper. 

Thus, I computed the Inception score of CWPGGAN using the same (publicly available) Inception network used in the evaluation part of the AttnGAN paper. The score obtained by CWPGGAN, as well as the score of the other models, are given in Table \ref{tab:inception-scores-birds-attngan}.

\begin{table}[h]
    \centering
    \begin{tabular}{| l | l | l |}
    \hline
    Model & Resolution & Score \\ \hline
    GAN-INT-CLS & 64$\times$64 & 2.88 $\pm$ 0.04 \\ 
    CWPGGAN$^*$ & 64$\times$64 & \textbf{3.18 $\pm$ 0.03} \\ \hline
    StackGAN & 256$\times$256 & 3.70 $\pm$ 0.04 \\ 
    StackGAN-v2 & 256$\times$256 & 3.82 $\pm$ 0.06 \\ 
    CWPGGAN$^*$ & 256x356 & 4.09 $\pm$ 0.03 \\ 
    AttnGAN & 256$\times$256 & \textbf{4.36 $\pm$ 0.04} \\ \hline

    \end{tabular}
    \caption[Birds dataset Inception scores]{Inception scores for the CUB-200-2011 birds dataset using the Inception network used for evaluation in \cite{DBLP:journals/corr/abs-1711-10485}. The Inception scores of all models, excepting CWPGGAN, are taken directly from \cite{DBLP:journals/corr/abs-1711-10485}. Models marked with $^*$ are the models proposed in this report.}
    \label{tab:inception-scores-birds-attngan}
\end{table}

CWPGGAN boosts by $7.07\%$ the best Inception Score on the birds dataset of the models which use only the sentence-level visual semantics. Moreover, CWPGGAN has the second best Inception Score for 256$\times$256 images out of all the existent state of the art models. The score of CWPGGAN and the quality of the images it produces is particularly impressive given the that it does not use any word-level visual semantics such as AttnGAN. This score is also achieved in conditions of guaranteed stability given by the proposed loss function in WGAN-CLS. Because AttnGAN is composed of a StackGAN-v2 with an attention model on top, these results are an indication for future research that CWPGGAN equipped with a similar attention model could produce even higher scores.

The results also prove that the proposed Wasserstein loss makes possible the usage of innovative architectures and training techniques which would not work with the standard loss function used by the existent text to image models.

\section{Side by Side Comparison of the Models}

Figures \ref{fig:img1_comp64}, \ref{fig:img2_comp64} and \ref{fig:img3_comp64} show a side by side comparison of the models which generate images with resolution 64$\times$64. Figures \ref{fig:img1_comp256}, \ref{fig:img2_comp256} and \ref{fig:img3_comp256} include a side by side comparison of the models which generate images with resolution 256$\times$256.

\begin{figure}[h]
    \centering
    \includegraphics[width=\textwidth]{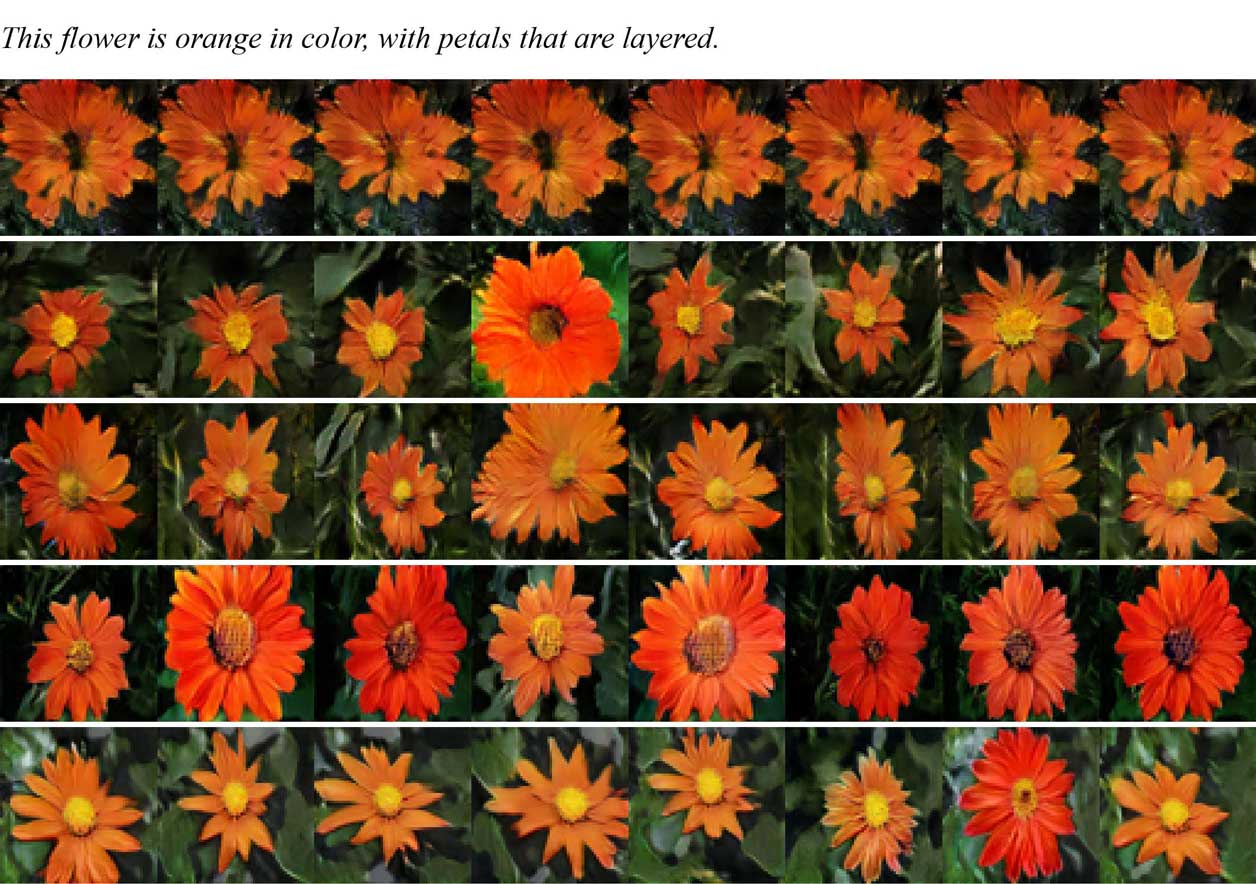}
    \caption[Side by side comparison of 64$\times$64 models (I)]{Each row contains 64$\times$64 images generated by a different model from the top text description. The order is: GAN-CLS (first row), WGAN-CLS (second row), StackGAN Stage I (third row), Conditional Least Squares PGGAN (forth row), Conditional Wasserstein PGGAN (fifth row).}
    \label{fig:img1_comp64}
\end{figure}

\begin{figure}[h]
    \centering
    \includegraphics[width=\textwidth]{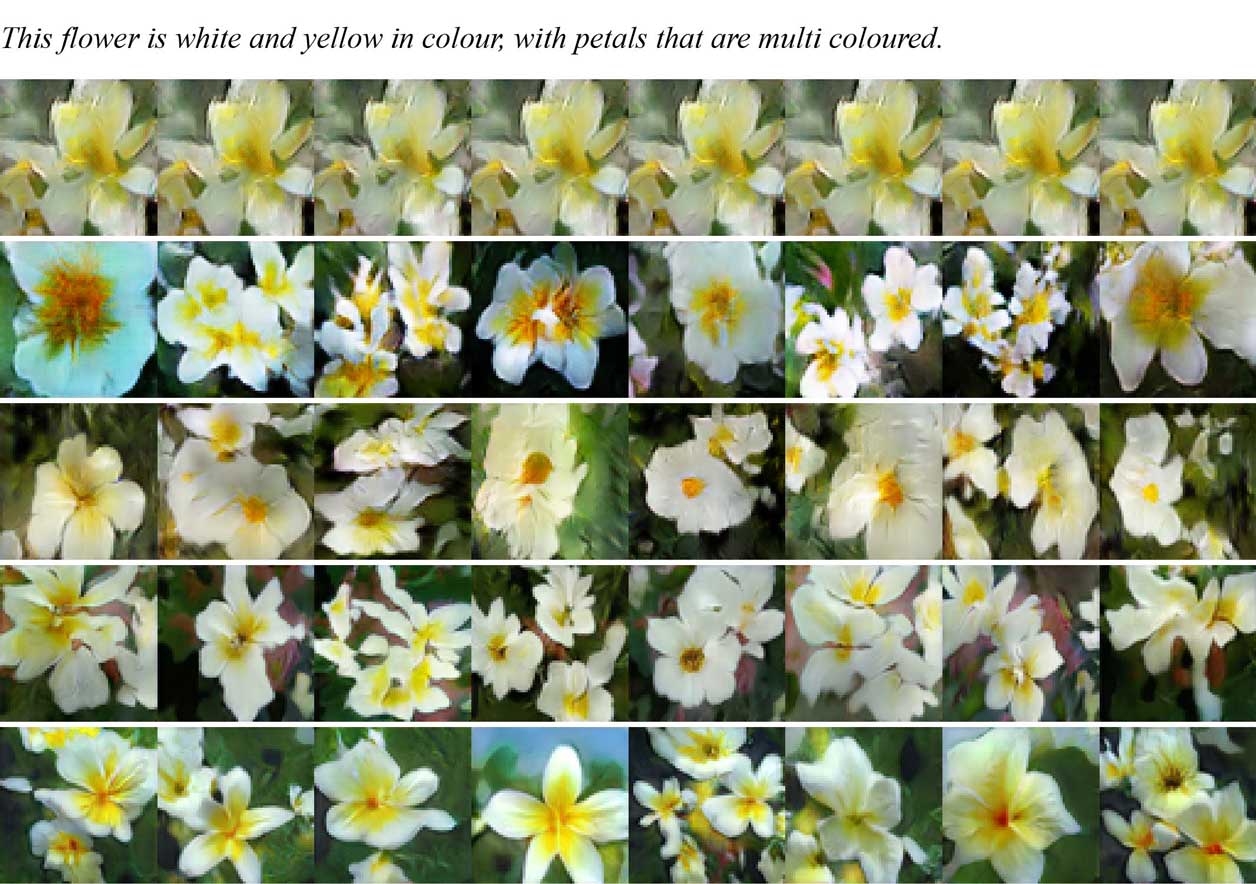}
    \caption[Side by side comparison of 64$\times$64 models (II)]{Each row contains 64$\times$64 images generated by a different model from the top text description. The order is: GAN-CLS (first row), WGAN-CLS (second row), StackGAN Stage I (third row), CLSPGGAN (forth row), CWPGGAN (fifth row).}
    \label{fig:img2_comp64}
\end{figure}

The lack of diversity of the images produced by GAN-CLS is evident. All the other models create a variety of images for the same text descriptions thanks to the condition augmentation module they are all equipped with. The quality of the images generated by WGAN-CLS is subjectively better than the one of GAN-CLS and comparable to the one of Stage I of StackGAN. The CPGGANs (64$\times$64) generate more structurally coherent images than the other models. The Wasserstein based CPGGAN generates even more diverse images, but the text-image matching of its images is slightly worse than the one of the other models. Figure \ref{fig:img3_comp64}, where CWPGGAN generates a few flowers which do not contain any shade of pink is one such example. 

\begin{figure}[h]
    \centering
    \includegraphics[width=\textwidth]{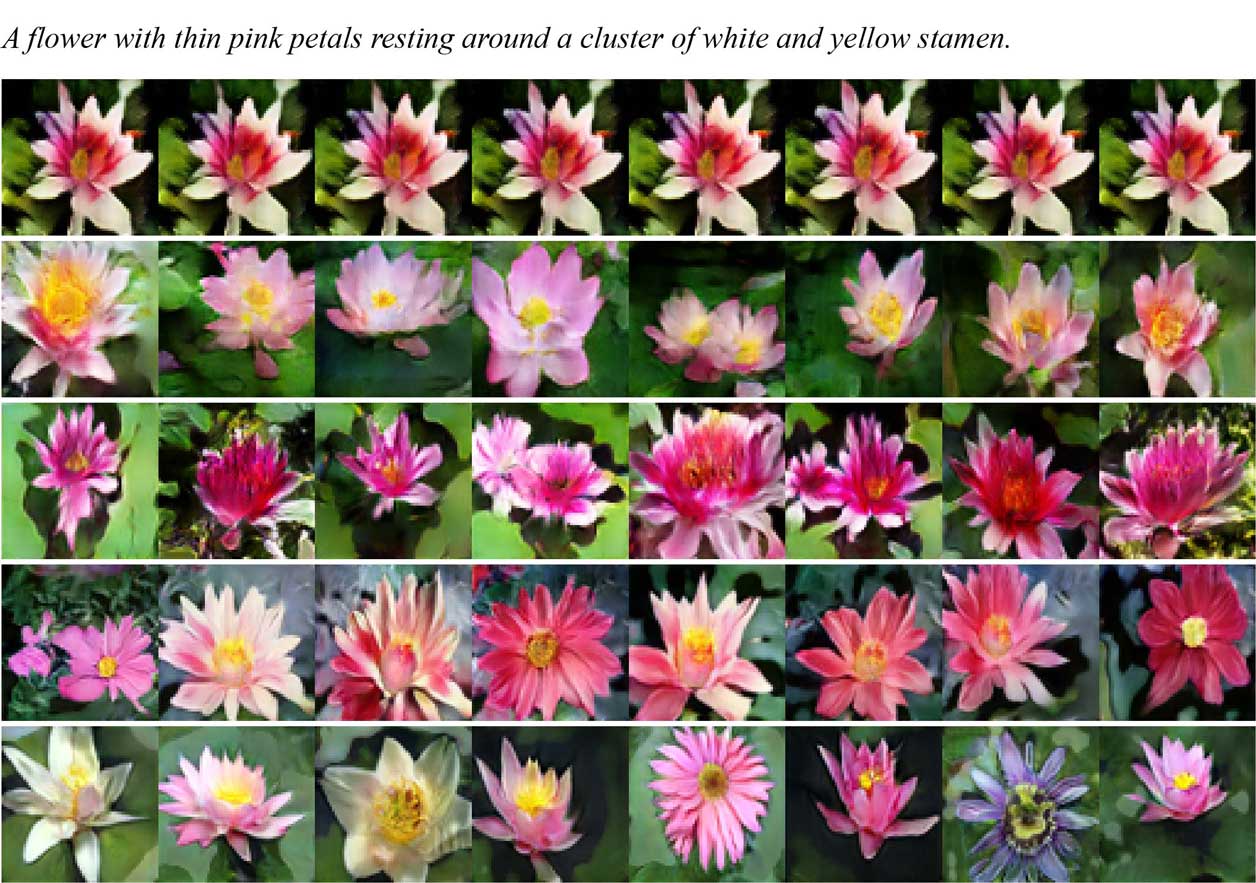}
    \caption[Side by side comparison of 64$\times$64 models (III)]{Each row contains 64$\times$64 images generated by a different model from the top text description. The order is: GAN-CLS (first row), WGAN-CLS (second row), StackGAN Stage I (third row), Conditional Least Squares PGGAN (forth row), Conditional Wasserstein PGGAN (fifth row).}
    \label{fig:img3_comp64}
\end{figure}

\begin{figure}[h]
    \centering
    \includegraphics[width=\textwidth]{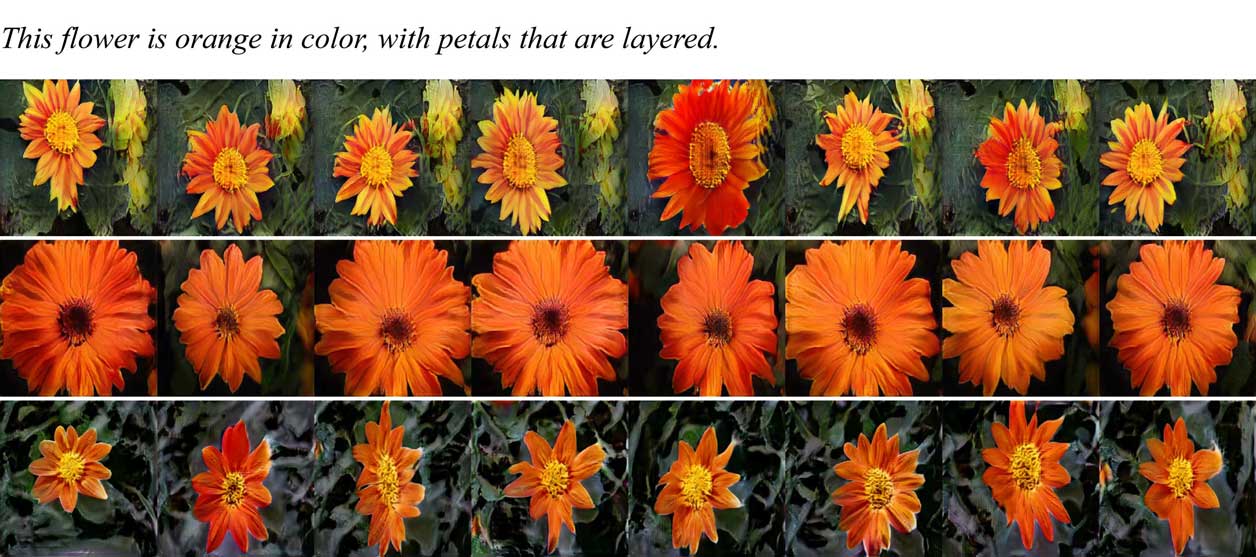}
    \caption[Side by side comparison of 256$\times$256 models (I)]{Each row contains 256$\times$256 images generated by a model from the top text description. The order is: StackGAN Stage II (first row), Conditional Least Squares PGGAN (second row), Conditional Wasserstein PGGAN (third row).}
    \label{fig:img1_comp256}
\end{figure}

The slightly worse text-image matching becomes more visible on the 256$\times$256 version of CWPGGAN (see \ref{fig:img3_comp256}). Nevertheless, the images are subjectively better than the images of the other models, which is also confirmed by the Inception Score. Note that, in the case of 256$\times$256 images, the CLSPGGAN generates slightly unrealistic textures (Figure \ref{fig:img1_comp256}) or images which lack local coherence (Figure \ref{fig:img3_comp256}). I believe this is due to the Gaussian noise hack which was used to fix its instability.

To test that the models do not simply memorise the images from the dataset and that they produce new images, a nearest neighbour analysis is given in Figure \ref{fig:neighb64} for 64$\times$64 images and Figure \ref{fig:neighb256} for 256$\times$256 images.

\begin{figure}[h]
    \centering
    \includegraphics[width=\textwidth]{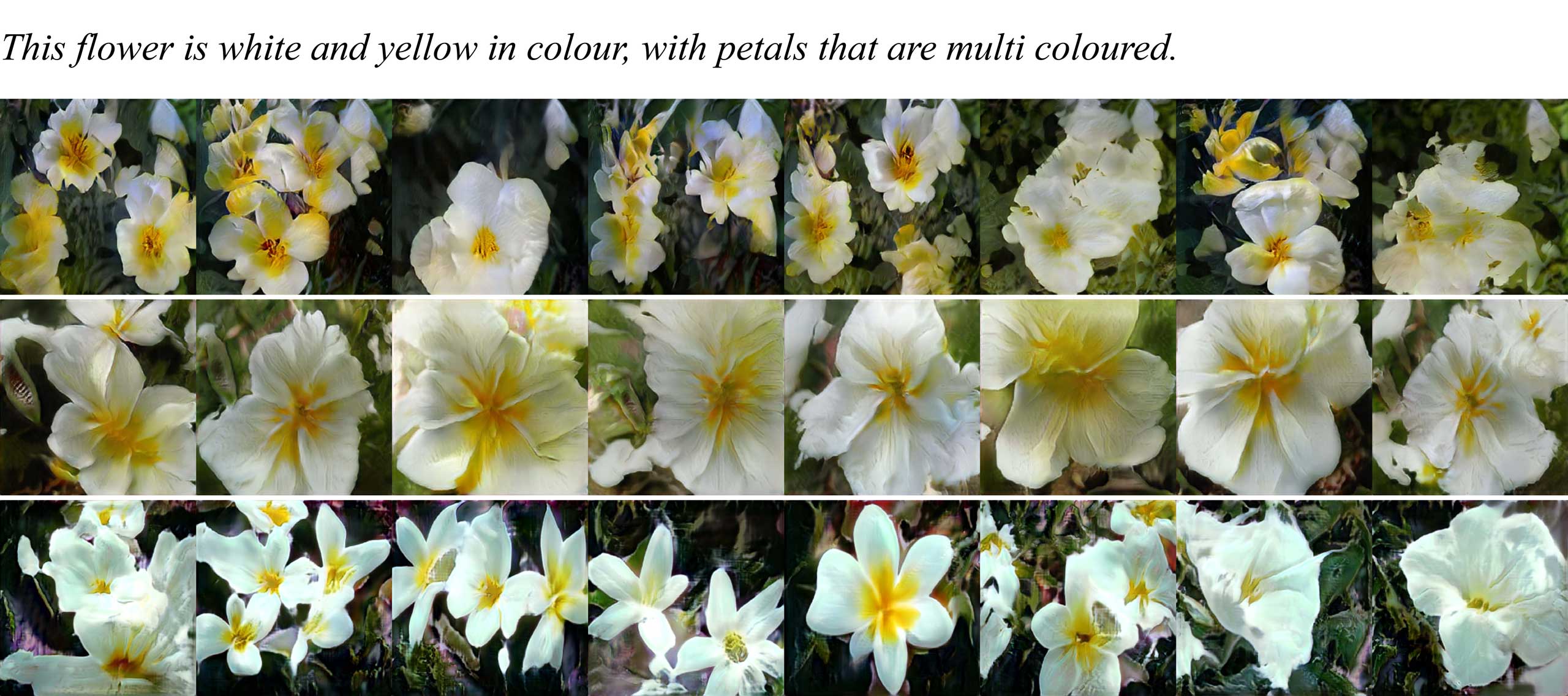}
    \caption[Side by side comparison of 256$\times$256 models (II)]{Each row contains 256$\times$256 images generated by a model from the top text description. The order is: StackGAN Stage II (first row), Conditional Least Squares PGGAN (second row), Conditional Wasserstein PGGAN (third row).}
    \label{fig:img2_comp256}
\end{figure}

\begin{figure}[h]
    \centering
    \includegraphics[width=\textwidth]{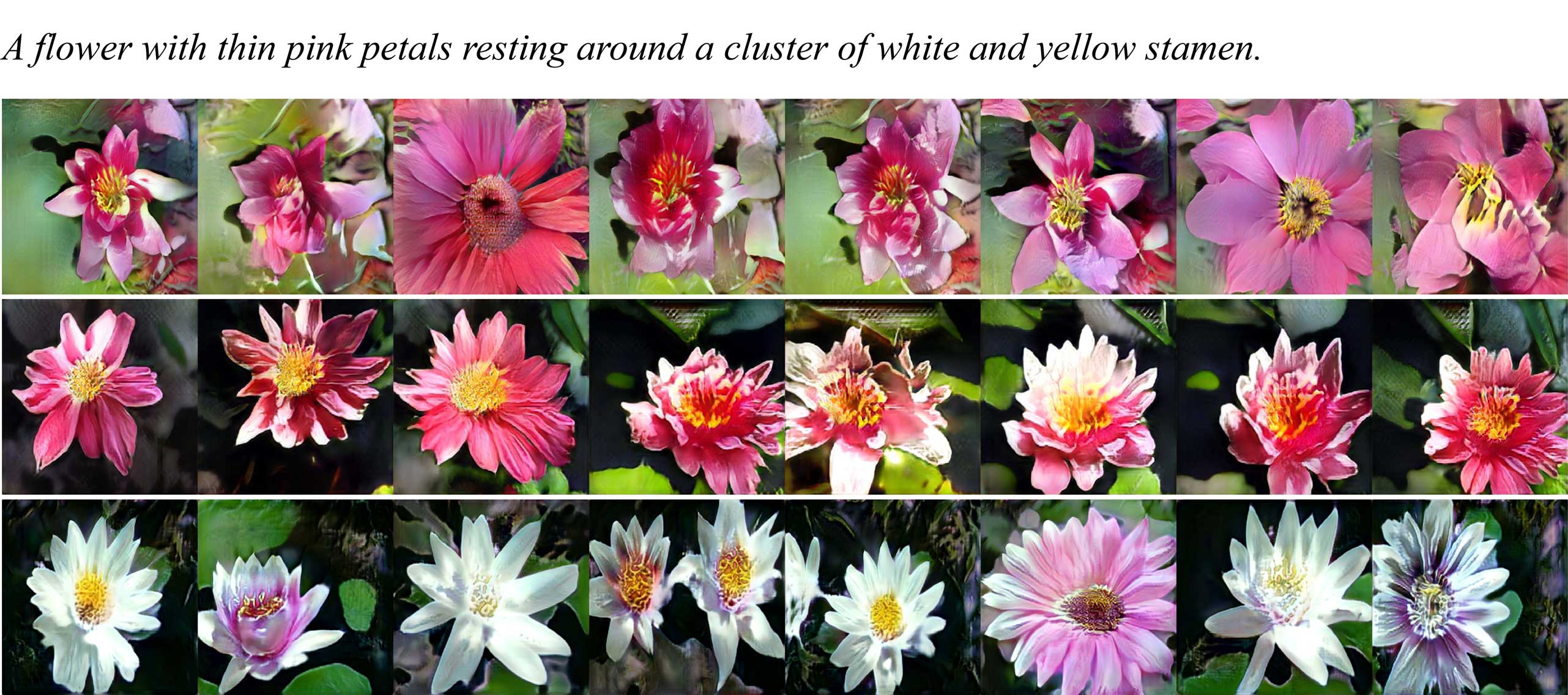}
    \caption[Side by side comparison of 256$\times$256 models (III)]{Each row contains 256$\times$256 images generated by a model from the top text description. The order is: StackGAN Stage II (first row), Conditional Least Squares PGGAN (second row), Conditional Wasserstein PGGAN (third row).}
    \label{fig:img3_comp256}
\end{figure}

\begin{figure}[ht]
    \centering
    \includegraphics[width=\textwidth]{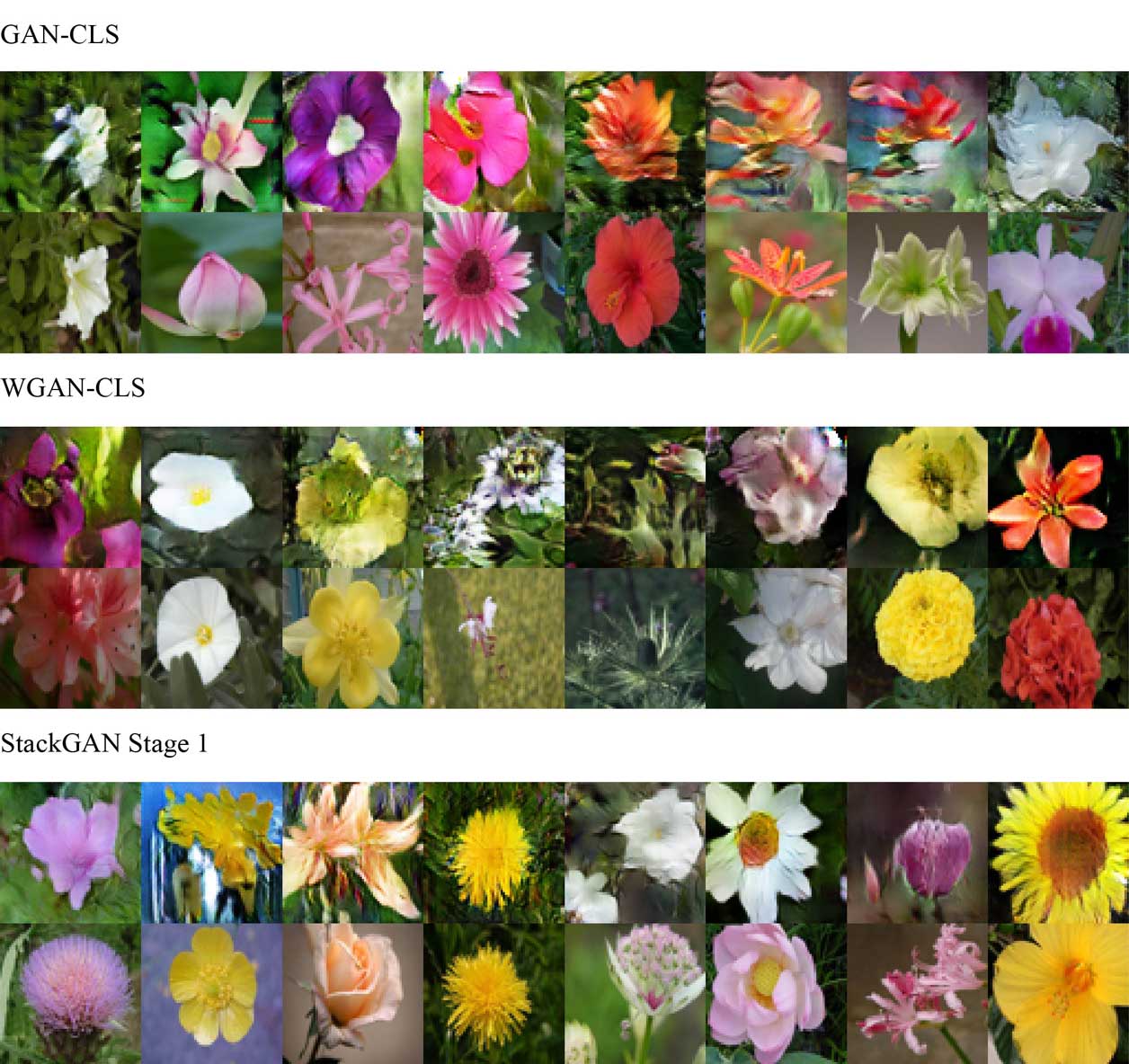}
    \caption[Nearest neighbour analysis for 64$\times$64 images]{For each model, the first row contains 64$\times$64 images produced by the model and the second row contains the nearest neighbour from the training dataset.}
    \label{fig:neighb64}
\end{figure}

\begin{figure}[ht]
    \centering
    \includegraphics[width=\textwidth]{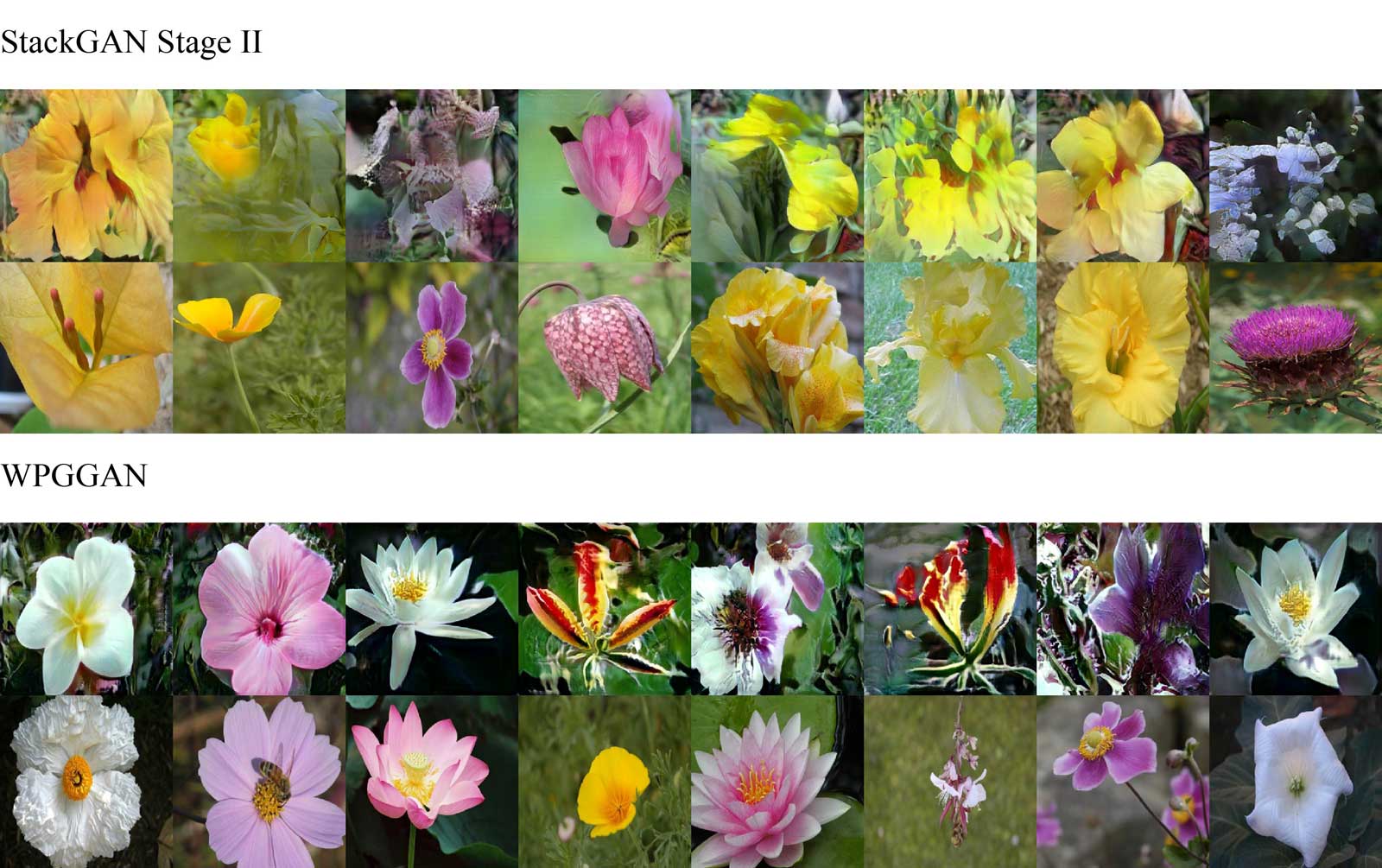}
    \caption[Nearest neighbour analysis for 256$\times$256 images]{For each model, the first row contains 256$\times$256 images produced by the model and the second row contains the nearest neighbour from the training dataset.}
    \label{fig:neighb256}
\end{figure}

A comparison between StackGAN and Conditional Wasserstein PGGAN is provided in Appendix \ref{appendix:birds}.

%% file: chapter5.tex
\chapter{Reflection and Conclusion} \label{sec:conclusions}

This chapter contains a reflection on the planning and management of this project and ends with a conclusion of the present work.

\section{Planning and Management}

The plan of my project is divided into two main parts:

\begin{enumerate}
    \item The first part, covering the first semester, was concerned with background reading, the understanding of the existent state of the art models, the reproduction of their results and the identification of their limitations and consequently of possible directions of research.
    \item The second part, covering the second semester, was concerned with finding solutions for the identified research problems. 
\end{enumerate}

\begin{table}[h]
    \centering
    \begin{tabular}{| l | c | c |}
        \hline 
        Milestone & Planned Weeks & Actual Weeks \\ \hline
        Background reading on GANs & 1-4 & 1-4 \\ \hline
        Reproduce GAN-CLS results \cite{reed2016generative} & 5-6 & 5-7 \\ \hline
        Reproduce StackGAN results \cite{han2017stackgan} & 7-10 & 7-11 \\ \hline
        Reproduce I2T2I results \cite{Dong2017I2T2ILT} & 11-13 & - \\ \hline
        Implement Inception Score evaluation & 14 & 12 \\ \hline
    \end{tabular}
    \caption[First semester milestones]{The milestones of the first part of the project. The plan is based on the weeks of the academic year.}
    \label{tab:milestones}
\end{table}

Table \ref{tab:milestones} includes the milestones for the first part of the project together with their timeline. Out of these milestones, I decided to skip the reproduction of the results of the I2T2I paper for two reasons. The StackGAN paper proposes a more elegant solution for textual data augmentation as discussed in Section \ref{sec:cond-aug}. The implementation of an image captioning system would have required significantly more background reading in the area of language models which is vast. Instead, I decided to start the research part of the project earlier, before the start of the second semester. After the first semester, I identified three research directions summarised in Table \ref{tab:research_directions}.

\begin{table}[h]
    \centering
    \begin{tabular}{| l | c | c |}
        \hline 
        Research direction & Planned Weeks & Actual Weeks \\ \hline
        Stable Conditional GAN & 15-20 & 13-21 \\ \hline
        Conditional GAN operating on multiple resolutions & 21-25 & 21-26 \\ \hline
        Explicit evaluation of text-image matching & 26 & 26 \\ \hline
    \end{tabular}
    \caption[Second semester research directions]{The identified research directions for the second semester. The plan is based on weeks of the academic year.}
    \label{tab:research_directions}
\end{table}

Out of these research directions, I obtained good results for the first two points on the list as described in Chapter \ref{sec:research}.

Due to the significant training time the presented models take and the limited computing resources (one Nvidia 1080Ti) I decided to focus my experiments on the flowers dataset and not on both the flowers and birds datasets as I originally intended. The focus on a slightly smaller and less complicated dataset such as Oxford-102 offered more time for testing ideas and rigorous evaluation. Nevertheless, I run a few experiments on the birds dataset, and the results can be found in Appendix \ref{appendix:birds}.




\section{Conclusion}

In this work, I present Generative Adversarial Networks and their application in the problem of text to image synthesis. I explain how the current state of the art models work at the intersection between Computer Vision and Natural Language Models and I reproduce the results of the papers which introduce them. Moreover, I bring my contribution to the field by proposing a novel Conditional Wasserstein GAN (WGAN-CLS) which enables conditional generation of images with a stable training procedure. The images this model generates are comparable to the current state of the art models. I show how this conditional Wasserstein loss function can be used in a more advanced model: the proposed Conditional Progressive Growing GAN. Other classical GAN loss functions would not work on such a model because of their instability during training. I show that Conditional Progressive Growing GANs, with the novel conditional Wasserstein loss, produce better results than the current state of the art models which use only sentence-level visual semantics.

%% file: appendix1.tex
\chapter{Proof of Theorem~\ref{theo:1}}\label{appendix:A}

\fta*
\begin{proof}
First, we am going to find in function space the optimal strategy of the discriminator $D^{*}: \mathcal{X} \to [0, 1]$ for a fixed strategy $G: \mathcal{Z} \to \mathcal{X}$ of the generator. For this, we assume that $\mathbb{P}_r(\vec{x})$ and $\mathbb{P}_g(\vec{x})$ are non zero everywhere in order to avoid undefined behaviour. The proof uses variational calculus for finding function $D$ which maximises the functional $F[x, D] = V(D, G)$. For an introduction to calculus of variations please consult \cite{gelfand2012calculus}.

\begin{align*}
    V(D, G) &= \mathbb{E}_{\rvec{X} \sim \mathbb{P}_r}[\log(D(\vec{x}))] + \mathbb{E}_{\rvec{Z} \sim \mathbb{P}_{\rvec{Z}}}[\log(1 - D(G(\vec{z})))] 
    \\ &= \int_{\mathcal{X}}{\mathbb{P}_r(\vec{x})\log(D(\vec{x}))d\vec{x}} + \int_{\mathcal{Z}}{\mathbb{P}_{\rvec{Z}}(\vec{z})(\log(1 - D(G(\vec{z})))d\vec{z}} \tag{by the definition of expectation} 
    \\ &= \int_{\mathcal{X}}{\mathbb{P}_r(\vec{x})\log(D(\vec{x}))d\vec{x}} + \int_{\mathcal{X}}{\mathbb{P}_g(\vec{x})(\log(1 - D(\vec{x}))d\vec{x}} \tag{by the  Radon–Nikodym theorem of measure theory} 
    \\ &= \int_{\mathcal{X}}{[\mathbb{P}_r(\vec{x})\log(D(\vec{x})) + \mathbb{P}_g(\vec{x})(\log(1-D(\vec{x}))]d\vec{x}} \tag{by the additive property of integration} 
    \\ &= \int_{\mathcal{X}}{[\mathbb{P}_r(\vec{x})\log(D(\vec{x})) + \mathbb{P}_g(\vec{x})( \log(1-D(\vec{x}))]d\vec{x}}
    \\ &= \int_{\mathcal{X}}{L(\vec{x}, D)}
\end{align*}

We denote here the integrand by $L(x, D)$. To find $D(\vec{x})$ which maximises this functional, we apply the Euler-Lagrange equation from variational calculus, which can be further simplified because $D^\prime$ does not show up in $L$. For the next part, we will not include the argument $x$ of the functions $\mathbb{P}(\vec{x})$ and write them as $\mathbb{P}$ to save space.

\begin{align*}
    \pdv{L}{D} - \dv{}{x} \pdv{L}{D^\prime} = 0 &\Leftrightarrow \pdv{L}{D} - 0 = 0
    \\ &\Leftrightarrow \pdv{}{D}[\mathbb{P}_r\log(D) + \mathbb{P}_g(\log(1 - D)] = 0
    \\ &\Leftrightarrow \frac{\mathbb{P}_r}{D} - \frac{\mathbb{P}_g}{1 - D} = 0
    \\ &\Leftrightarrow \mathbb{P}_r(1 - D) = \mathbb{P}_gD
    \\ &\Leftrightarrow \mathbb{P}_r = (\mathbb{P}_g + \mathbb{P}_r)D
    \\ &\Leftrightarrow D = \frac{\mathbb{P}_r}{\mathbb{P}_g + \mathbb{P}_r}
\end{align*}

Now we can write the cost function of $G$ as $C(G) = V(D^{*}, G)$.
\begin{equation}\label{eq:A1}
\begin{split}
    C(G) = \mathbb{E}_{\rvec{X} \sim \mathbb{P}_r}\log(\frac{\mathbb{P}_r}{\mathbb{P}_r + \mathbb{P}_g}) + \mathbb{E}_{\rvec{X} \sim \mathbb{P}_g}\log(\frac{\mathbb{P}_g}{\mathbb{P}_r + \mathbb{P}_g})
\end{split}
\end{equation}

We now show that the global minimum of $C(G)$ is obtained when $\mathbb{P}_g = \mathbb{P}_r$. When this equation holds, $D = \frac{1}{2}$. Plugging this in $C(G)$ we obtain:
\begin{equation}\label{eq:A2}
\begin{split}
    C(G) &= \mathbb{E}_{\rvec{X} \sim \mathbb{P}_r}\log(\frac{1}{2}) +  \mathbb{E}_{\rvec{X} \sim \mathbb{P}_g}\log(\frac{1}{2})
    \\ &= -\mathbb{E}_{\rvec{X} \sim \mathbb{P}_r}\log(2) -  \mathbb{E}_{\rvec{X} \sim \mathbb{P}_g}\log(2) = -\log(4)
\end{split}
\end{equation} 
By adding equations \ref{eq:A1} and \ref{eq:A2} we see that $-\log(4)$ is indeed the minimum value of $C(G)$
\begin{align*}
    C(G) -\mathbb{E}_{\rvec{X} \sim \mathbb{P}_r}\log(2) - \mathbb{E}_{\rvec{X} \sim \mathbb{P}_g}\log(2) &= \mathbb{E}_{\rvec{X} \sim \mathbb{P}_r}\log(\frac{\mathbb{P}_r}{\mathbb{P}_r + \mathbb{P}_g})  \\ &+  \mathbb{E}_{\rvec{X} \sim \mathbb{P}_g}\log(\frac{\mathbb{P}_g}{\mathbb{P}_r + \mathbb{P}_g}) \\ &- \log(4)
\end{align*} 
By leaving only $C(G)$ on the left hand side we obtain: 
\begin{align*}
        C(G) &= -\log(4) + KL\infdiv{\mathbb{P}_r}{\frac{\mathbb{P}_r + \mathbb{P}_g}{2}} + KL\infdiv{\mathbb{P}_g}{\frac{\mathbb{P}_r + \mathbb{P}_g}{2}} 
        \\ &= -\log(4) + JS\infdiv{\mathbb{P}_r}{\mathbb{P}_g}
\end{align*}
Where $KL$ and $JS$ are the Kullback-Leibler divergence and the Jensen–Shannon divergence respectively. Because the Jensen–Shannon divergence is a metric it has the non-negativity and identity of indiscernibles properties. Thus, its minimum of 0 is realised only when $\mathbb{P}_r = \mathbb{P}_g$.

\end{proof}

Note that as in maximum likelihood models, no knowledge of the unknown $\mathbb{P}_r$ is needed to optimise $JS\infdiv{\mathbb{P}_r}{\mathbb{P}_g}$ (only samples). The expectation over $\rvec{X} \sim \mathbb{P}_r$ can be approximated using a sample mean.

%% file: appendix2.tex
\chapter{Notions of Deep Learning} \label{appendix:nn}

The goal of this chapter is to explain some of the Deep Learning notions and terminology used throughout this report. On the one hand, these notions are very commonly used and cannot be avoided. On the other hand, describing them in the body of the report would distract the reader from the main ideas.

\section{Neural Networks}

Neural Networks can be viewed as universal parametric function approximators. A function $f: A \to B$ can be approximated by a parametric function $f_{\vec{\theta}}: A \to B$ where $\vec{\theta}$ are the parameters of the network. A key idea in deep learning is that learning complicated functions can be done by using the composition of multiple but simple non linear functions. The stack of layers of a network is a composition of such functions. Assuming $f_{\vec{\theta}}$ has $n$ layers, these can be denoted by $f^{(1)}_{\vec{\theta}}, \dots, f^{(n)}_{\vec{\theta}}$. Then, $f_{\vec{\theta}} = f^{(n)}_{\vec{\theta}}(\dots(f^{(1)}_{\vec{\theta}}))$

\subsection{Backpropagation}

In order to approximate $f$, a loss function $L$ which describes how far the approximation $f_{\vec{\theta}}$ is from $f$ is used. The approximation can be improved by decreasing $L$. The minimisation of $L$ offers a way to adjust the parameters $\vec{\theta}$ to improve the approximation. This process is called back-propagation. For any parameter $\vec{\theta}_i$ of the network, the partial derivative $\pdv{L}{\vec{\theta}_{i}}$ can be computed using the (multivariate) chain rule. For this partial derivative to exist, it is required that the functions each layer implements are differentiable (almost everywhere). Thus, the parameters can be updated using the following procedure: $\vec{\theta}_{i} \gets \vec{\theta}_i - \alpha \pdv{L}{\vec{\theta}_{i}}$ where $\alpha$ is the learning rate. The minus sign is introduced because the parameter must be moved in the opposite direction of the sign of the derivative to approach the minimum of $L$.

Most often, mini-batch gradient descent is used. The network does not take as input a single example, but rather a batch of samples and the loss is computed with respect to this batch. When the parameter update is performed, it is calculated using the average derivative of that parameter where the average is taken over all the examples in the mini-batch. Thus, bigger mini-batches help reduce the variance of the updates but introduce additional computational cost at the same time.

\subsection{Activation Functions}

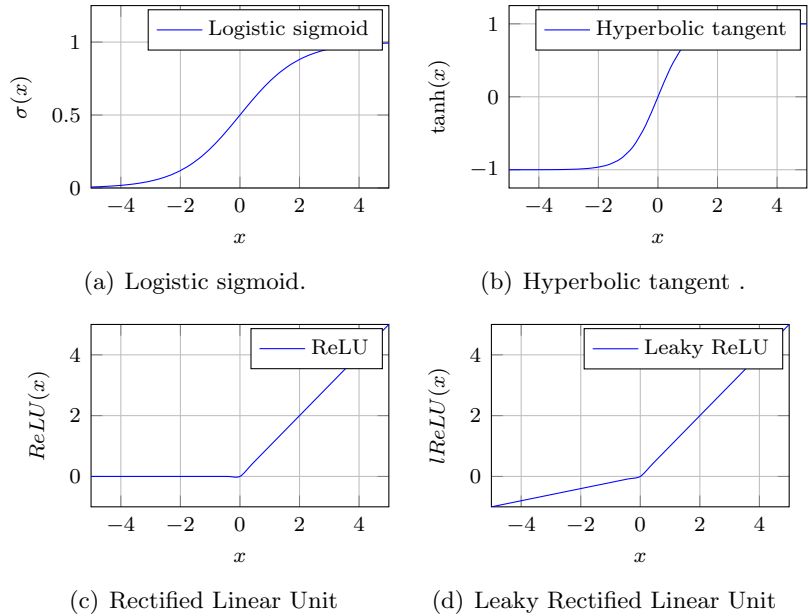
\begin{figure}[h]
    \centering
    \subfigure[Logistic sigmoid.]{
            \begin{tikzpicture}
            \begin{axis}[width=5.5cm,height=4cm,ylabel=$\sigma(x)$,xlabel=$x$,ymin=0,ymax=1.25,xmin=-5,xmax=5]
                \addplot[blue,smooth] {1/(1+exp(-x))};
                \addlegendentry{Logistic sigmoid}
            \end{axis}
        \end{tikzpicture}
    }
    \subfigure[Hyperbolic tangent .]{
        \begin{tikzpicture}
            \begin{axis}[width=5.5cm,height=4cm,ylabel=$\tanh(x)$,xlabel=$x$,ymin=-1.25,ymax=1.25,xmin=-5,xmax=5]
                \addplot[blue,smooth] {tanh(x)};
                \addlegendentry{Hyperbolic tangent}
            \end{axis}
        \end{tikzpicture}
    }\\
    \subfigure[Rectified Linear Unit]{
            \begin{tikzpicture}
            \begin{axis}[width=5.5cm,height=4cm,ylabel=$ReLU(x)$,xlabel=$x$,ymin=-1,ymax=5,xmin=-5,xmax=5]
                \addplot[blue,smooth] {max(0, x)};
                \addlegendentry{ReLU}
            \end{axis}
        \end{tikzpicture}
    }
    \subfigure[Leaky Rectified Linear Unit]{
            \begin{tikzpicture}
            \begin{axis}[width=5.5cm,height=4cm,ylabel=$lReLU(x)$,xlabel=$x$,ymin=-1,ymax=5,xmin=-5,xmax=5]
                \addplot[blue,smooth] {max(0.2*x, x)};
                \addlegendentry{Leaky ReLU}
            \end{axis}
        \end{tikzpicture}
    }
        \caption[Activation functions]{The commonly used activation functions. Softmax is not depicted it here because it is a multivariate vector valued function and it is harder to visualise.}
        \label{fig:act-fct}
\end{figure}

Usually, $f^{(i)}_{\vec{\theta}} = h(g^{(i)}_{\vec{\theta}})$ where $g^{(i)}_{\vec{\theta}}$ is a simple (linear) function followed by a non linearity $h(x)$. $h(x)$ is called an activation function. Without the activation functions, the network would not gain any additional capacity because the composition of multiple linear functions is still a linear function. In other words, multiple linear layers stacked together have the same capacity as a network with a single linear layer.

The activation functions (Figure \ref{fig:act-fct}) commonly used in practice and in this report are:

\begin{itemize}
    \item The sigmoid function $\sigma(x) = \frac{1}{1 + e^{-x}}$. It is also called the logistic function. It can be used to bring a variable in the range $[0, 1]$.
    \item The hyperbolic tangent activation function $\tanh(x) = 2\sigma(2x) - 1$ is a re-scaled sigmoid which brings the values in the range $[-1, 1]$. It is generally used when generating images to bring the values of the pixels in the range $[-1, 1]$.
    \item The softmax function takes as input a vector $\vec{x}$ and outputs a vector $\xi(\vec{x})$ whose values are in range $[0, 1]$ and sum up to one. It is usually used in classifiers to obtain a valid probability distribution over the classes some input could belong to. It is defined as  $\xi(\vec{x})_j = \frac{\vec{x}_j}{\Sigma_{k}\vec{x}_k}$ 
    \item A very popular and simple activation is the Rectified Linear Unit: $ReLU(x) = max(0, x)$ \cite{DBLP:journals/corr/HeZR015}. It is used in the intermediate layers to introduce non linearity in the model. Because for $x > 0$ the derivative is constant and non-zero, this activation prevents the gradient from saturating. 
    \item A generalisation of ReLUs are Leaky ReLUs $lReLU(x) = max(kx, x)$ where $k \in [0, 1]$ is usually close to $0$. $k = 0.2$ is a common value.
\end{itemize}

\section{Normalisation Techniques}

\subsection{Batch Normalisation}

Batch Normalisation \cite{Ioffe:2015:BNA:3045118.3045167} is a widely used technique for speeding up the learning and improving the performance of deep learning models. Ideally, it is desired that the input to a model to be whitened, to have zero mean and unit variance. Whitening the data was shown decades ago \cite{LeCun:1998:EB:645754.668382} to improve the speed of the training. Nevertheless, for deep learning, it is not enough because between layers inputs which are not normalised appear. A layer could supply to the next layer inputs with high variance and a mean far from zero. This phenomenon is called internal covariance shift. The fix is to whiten the data given as input to every layer using the batch statistics as in Equation \ref{eq:batch-norm}. 

\begin{equation} \label{eq:batch-norm}
    \hat{x}_{i} = \frac{x_{i} - \mu}{\sqrt{\sigma^2 + \epsilon}} \text{ and } y_{i} = \gamma \hat{x}_{i} + \beta
\end{equation}

Here, $x_{i}$ is an activation for the $i^{th}$ example in the minibatch and $\hat{x}_{i}$ is its whitened version. $\mu$ and $\sigma^2$ are the mean and variance of that activation over the entire batch. The trainable parameters $\gamma$ and $\beta$, ensure that this transformation is also able to represent the identity transform. They act as denormalisation parameters and can reverse the whitening if needed. 

Batch normalisation is based on the assumption that the batch statistics approximate the dataset statistics. Thus, the disadvantage of batch normalisation is that for small batch sizes the approximation is not so good and the performance drops. For more details, please check \cite{Ioffe:2015:BNA:3045118.3045167}.

GANs are empirically known to be more stable on architectures which use batch normalisation in the generator and the discriminator. 

\subsection{Layer Normalisation}

Layer normalisation performs the same type of whitening as batch normalisation with the exception that the normalisation is performed over all the hidden units in a layer and not by using the mini-batch statistics.

Because this normalisation technique is independent of the size of the mini-batch, it has the advantage that it does not impose a lower bound on the batch size. Nevertheless, layer normalisation brings only marginal improvements in convolutional layers and is better suited for Recurrent Neural Networks and fully connected layers. 

\section{Convolutional Layers}

Convolutional Layers are the building block of Convolutional Neural Networks. They operate on third order tensors, or informally, a 3D array of values and produce as output another 3D array (not necessarily with the same dimensions). Images are one such tensor with dimensions width $\times$ height $\times$ 3 in the case of RGB images. Convolutional layers are composed of a number $f$ of filters which can be adjusted. A filter has a reduced spatial resolution such as 3x3, and its depth is always equal to the depth of the input tensors. Each filter is convoluted with regions of the input tensor by sliding it across the width and height of the tensor. The result of each such convolution operation is a scalar. After sliding the filter over the spatial dimensions of the tensor, the output scalars form together a matrix. Each of the $f$ filters produces one such matrix. All these matrices stacked together form the output tensor. Convolutional layers also have other parameters besides the number of filters and the size of the filter. One of them is the stride which determines by how many units the filter is moved in each direction during sliding. Another one is the amount of zero padding which refers to padding the borders of the input tensor with zeros. These four hyper-parameters: the number of filters, the filter size, the stride and the amount of zero padding are used to manipulate the exact shape of the output tensor (Figure \ref{fig:convolution}).

\begin{figure}[h]
    \centering
    \includegraphics[width=\textwidth]{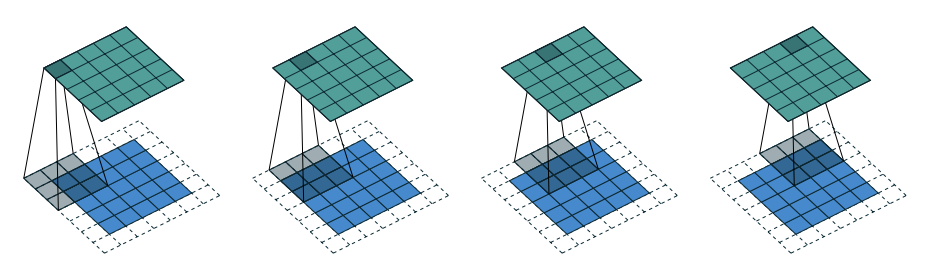}
    \caption[Convolution]{A convolutional filter with padding, stride one and filter size of 3x3. Image is taken from \cite{dumoulin2016guide}.}
    \label{fig:convolution}
\end{figure}

Each of the filters of a convolutional layer tries to learn a useful visual feature such various types of edges. The filters from the deeper levels of the network recognise more complex structures from the input image.

\begin{figure}[h]
    \centering
    \includegraphics[width=\textwidth]{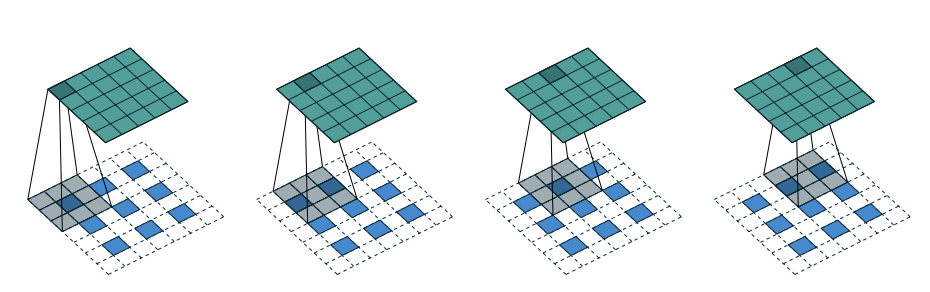}
    \caption[Transposed convolution]{The transposed convolution operation (also called deconvolution) performing upsampling. The resolution of the output (5x5) is higher than the one of the input (3x3). Image is taken from \cite{dumoulin2016guide}.}
    \label{fig:deconvolution}
\end{figure}

Convolutional layers can reduce or maintain the spatial resolution of the input tensor. Nevertheless, sometimes a reverse operation is needed to perform upsampling. A deconvolutional layer does precisely that. Deconvolutional layers can be thought of (and inefficiently implemented) as regular convolutional layers with the exception that the input pixels are moved apart, and zeros are inserted between them (see Figure \ref{fig:deconvolution}).

For more details on convolutional layers and their arithmetic, please consult \cite{dumoulin2016guide}.

\section{Residual Layers}

A good rule of thumb in deep learning is that more layers do not always translate to better performance. In fact, only increasing the depth of deep learning models has been shown to cause a decrease in the performance of the network \cite{resnet}. 

\begin{figure}[h]
    \centering
    \includegraphics[width=\textwidth]{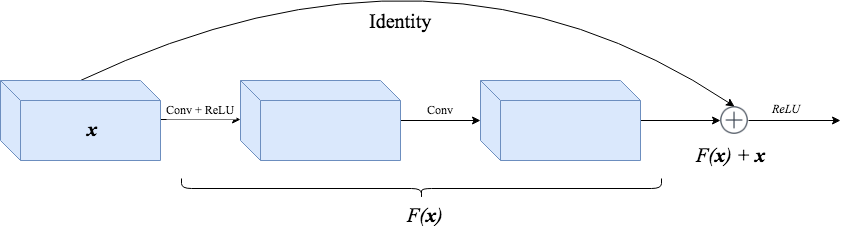}
    \caption[Residual layer]{A residual layer with two intermediate convolutional layers. The curved arrow represents the identity skip connection. The output of the two layers $F(\vec{x})$ and the input $\vec{x}$ are added at the end. An activation functions is applied after the addition.}
    \label{fig:res-layer}
\end{figure}

Residual Layers \cite{resnet} have eliminated this problem and led to better results. Figure \ref{fig:res-layer} shows the architecture of a residual layer. In this architecture, the network has to learn a function $F$ with respect to the identity mapping, rather than the zero mapping. This approach has two advantages:

\begin{enumerate}
    \item If the identity mapping is needed, the network can easily represent it by setting the value of the two intermediate weight layers to zero.
    \item The shortest path from any layer to the output layer is shorter than in a normal network. The identity skip connection helps prevent the gradient from becoming almost zero, especially for the first layers which are very far from the output.
\end{enumerate}

For more details please check \cite{resnet, He2016IdentityMI, Veit2016ResidualNB}.

%% file: appendix3.tex
\chapter{Additional Explanations of Wasserstein GANs} \label{appendix:wgan}

I include here a few clarifications and intuitive explanations for the technical terms used in Section \ref{sec:wgan-cls}.

\section{Manifolds}

I will not formally introduce manifolds here as it would be hard to present them rigorously in such a limited space. They are in essence generalisations in $n$ dimensions of surfaces. They are spaces with the property that the neighbourhood of any point resembles a Euclidean space. A typical example is that of Earth (the surface of a sphere) which locally (on Earth) looks flat, but the global structure is different. Low dimensional manifolds can be described by fewer dimensions than the space they live in. For example, a sheet of paper can be defined by two coordinates even if it is stretched in strange ways and it is perceived as part of the 3D world. It fundamentally remains a 2D object.

\section{Lipschitz Functions}

Lipschitz functions are functions whose rate of growth is limited. Formally, given two metric spaces $(\mathcal{A}, d_{\mathcal{A}})$ and $(\mathcal{B}, d_{\mathcal{B}})$, a function $f:\mathcal{A} \to \mathcal{B}$ is K-Lipschitz continuous for some constant $K$ if for any $x_1, x_2 \in \mathcal{A}$, relation \ref{eq:lipschitz} holds. 

\begin{equation} \label{eq:lipschitz}
    d_{\mathcal{B}}(f(x_1), f(x_2)) \leq K \cdot d_{\mathcal{A}}(x_1, x_2)
\end{equation}

This work is concerned with Euclidean spaces, so $d_{\mathcal{A}}$ and $d_\mathcal{B}$ are the usual Euclidean distances in their respective spaces. 

%% file: appendix4.tex
\chapter{Samples Generated on the Birds Dataset} \label{appendix:birds}

This section includes samples of generated images produced by Conditional Wasserstein PGGAN and StackGAN on the birds dataset as well as side by side comparisons between them. The birds dataset is more difficult than the flowers dataset. The dataset contains more images and more categories. The birds have different posses and are usually surrounded by vegetation which sometimes obstructs the view. Moreover, the descriptions are more detailed and refer to fine-grained details from the images. The higher difficulty of the dataset better reflects the difference between models.

\begin{figure}[h]
    \centering
    \includegraphics[width=\textwidth]{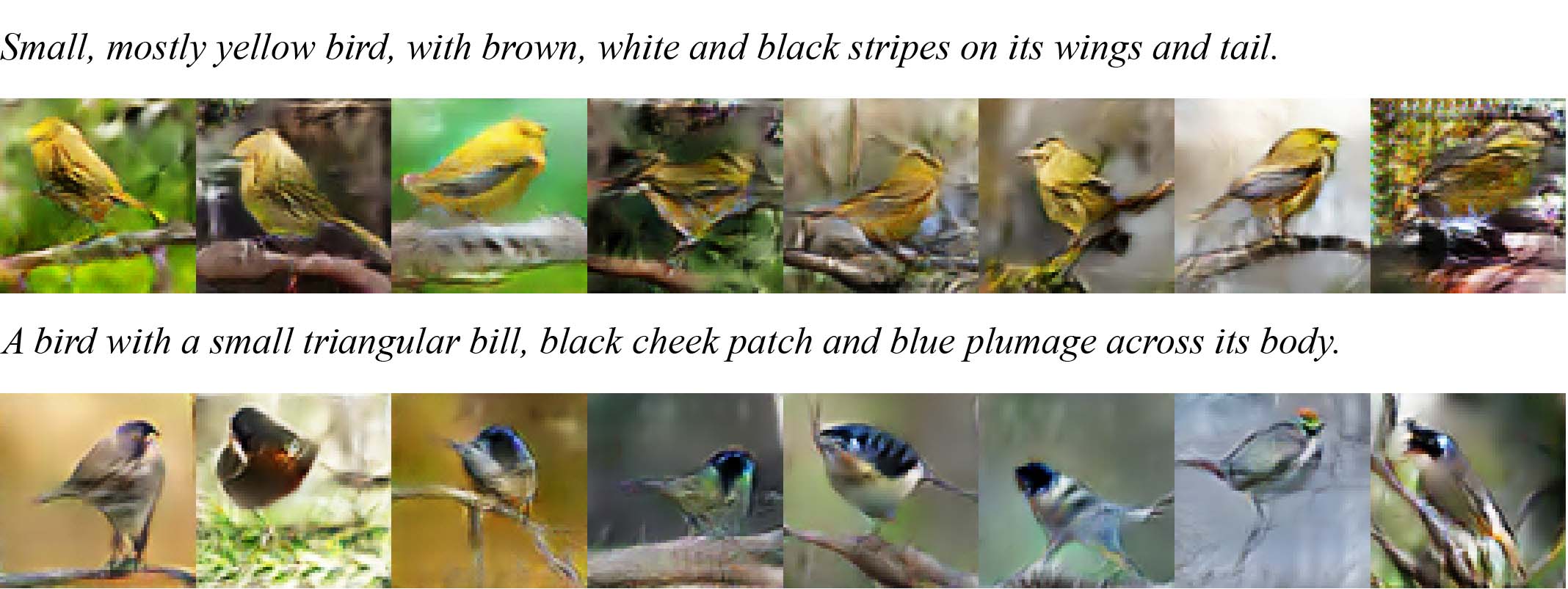}
    \caption[StackGAN Stage I birds dataset samples]{Samples generated by Stage I of StackGAN on the birds dataset.}
\end{figure}

\begin{figure}[h]
    \centering
    \includegraphics[width=\textwidth]{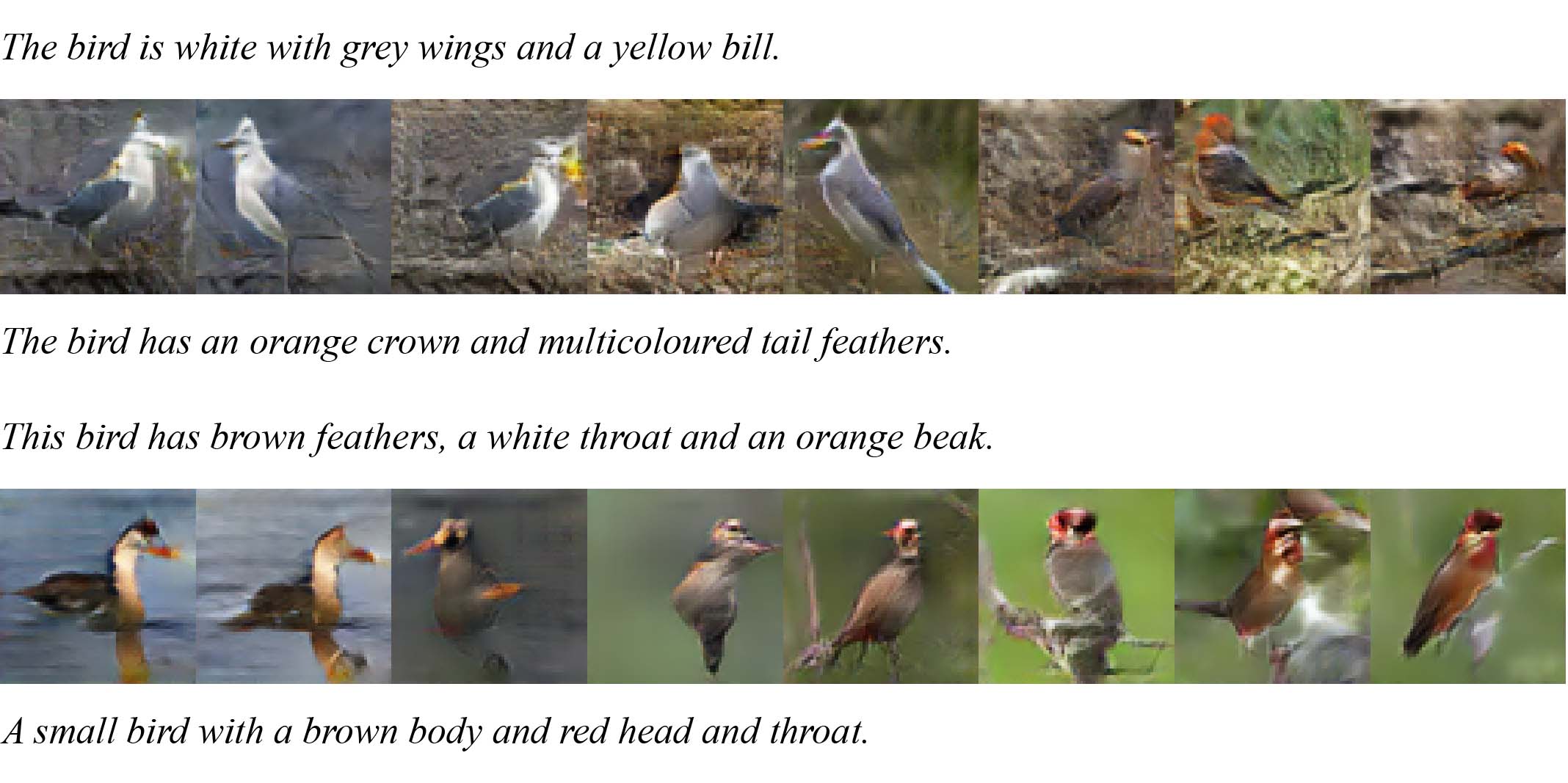}
    \caption[StackGAN Stage I birds dataset interpolations]{Samples generated by Stage I of StackGAN from interpolations of text embedding from the birds dataset.}
\end{figure}

\begin{figure}[h]
    \centering
    \includegraphics[width=\textwidth]{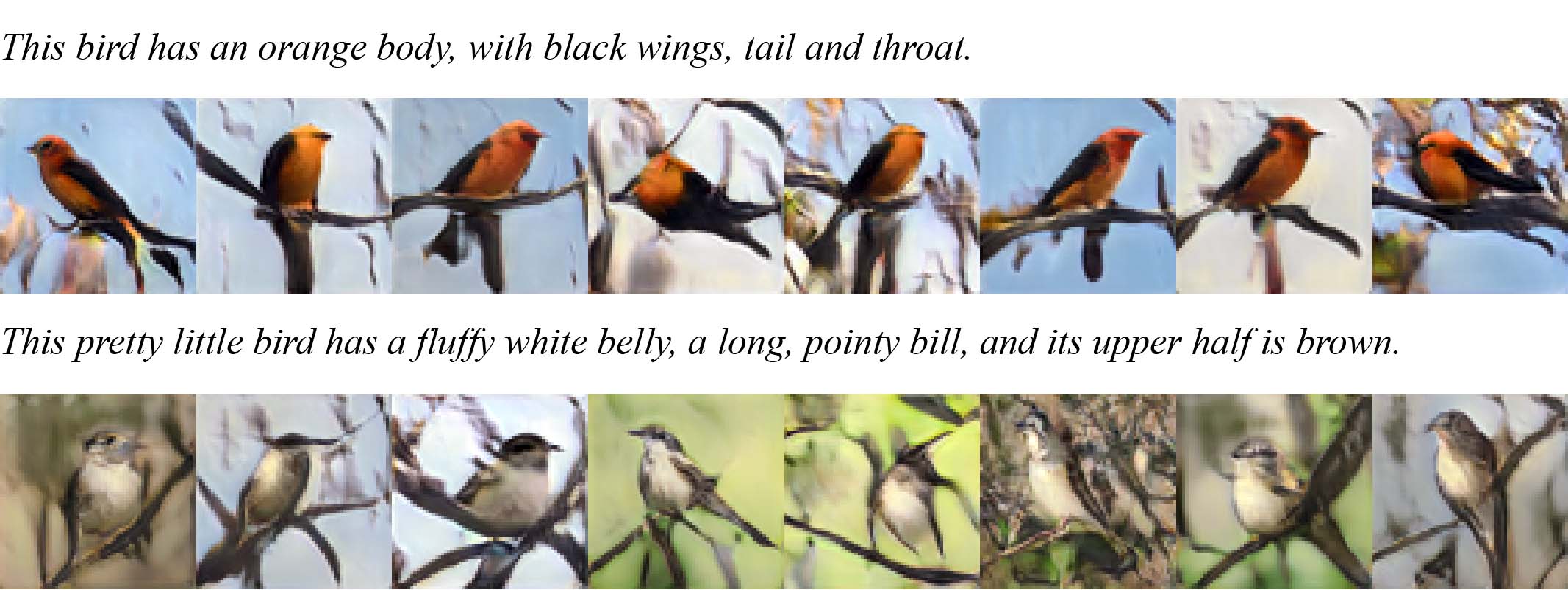}
    \caption[CWPGGAN birds dataset 64$\times$64 samples]{Samples (64$\times$64) generated by CWPGGAN on the birds dataset.}
\end{figure}

\begin{figure}[h]
    \centering
    \includegraphics[width=\textwidth]{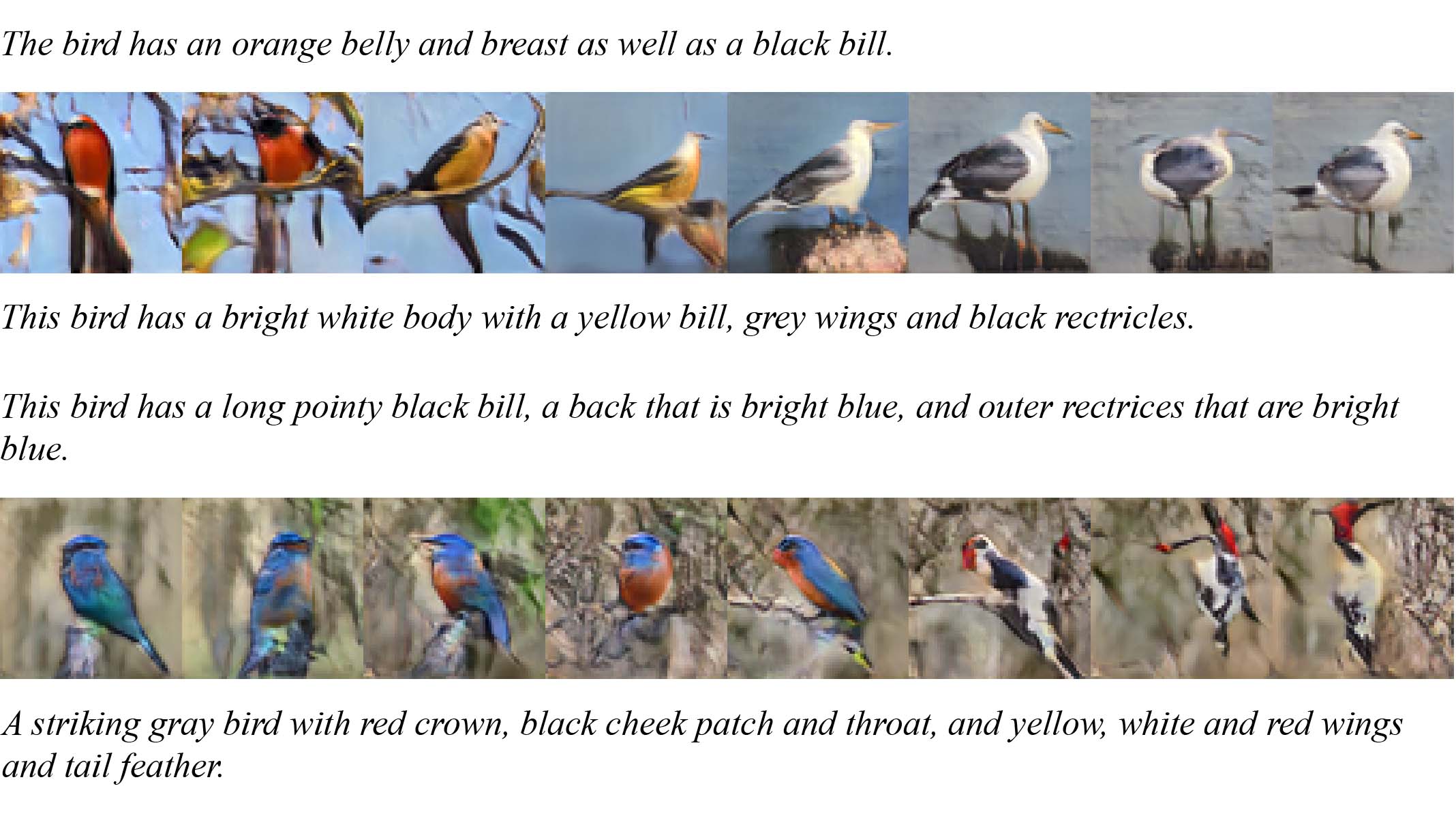}
    \caption[CWPGGAN birds dataset 64$\times$64 interpolations]{Samples (64$\times$64) generated by CWPGGAN from interpolations of text embedding from the birds dataset.}
\end{figure}

\begin{figure}[h]
    \centering
    \includegraphics[width=\textwidth]{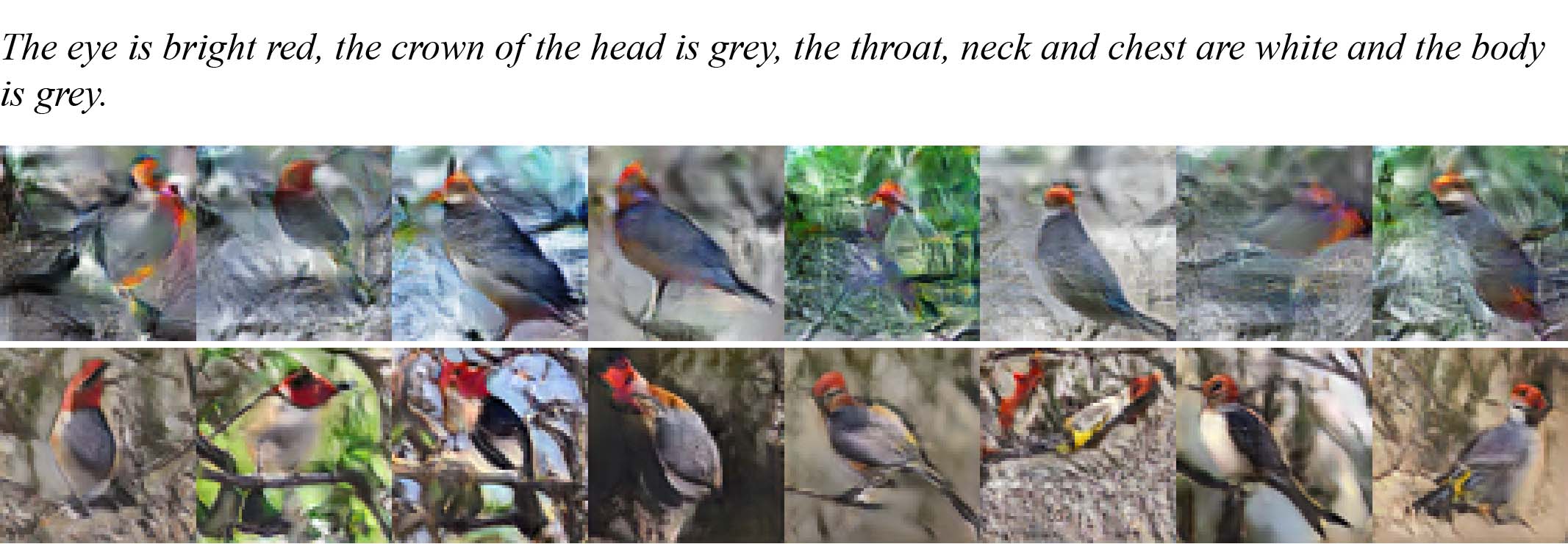}
    \caption[Comparison of 64$\times$64 models on the birds dataset (I)]{Comparison of 64$\times$64 models on the birds dataset: Stage I of StackGAN (first row), CWPGGAN (second row).}
\end{figure}

\begin{figure}[h]
    \centering
    \includegraphics[width=\textwidth]{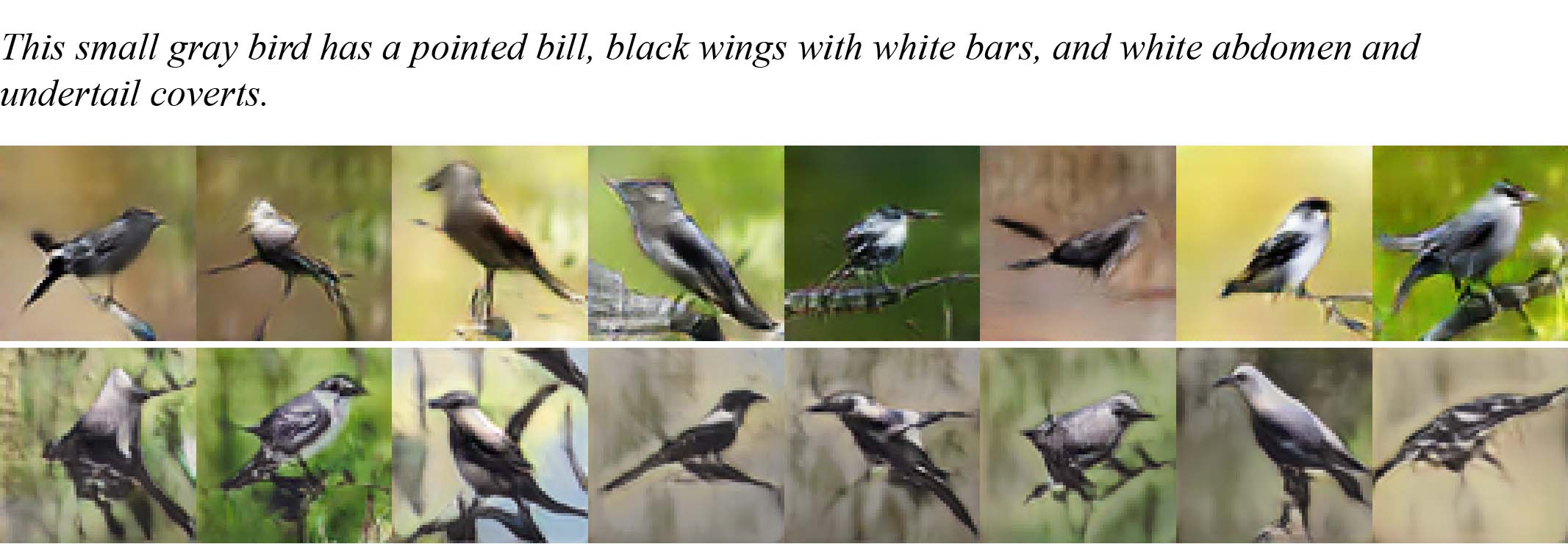}
    \caption[Comparison of 64$\times$64 models on the birds dataset (II)]{Comparison of 64$\times$64 models on the birds dataset: Stage I of StackGAN (first row), CWPGGAN (second row).}
\end{figure}

\begin{figure}[h]
    \centering
    \includegraphics[width=\textwidth]{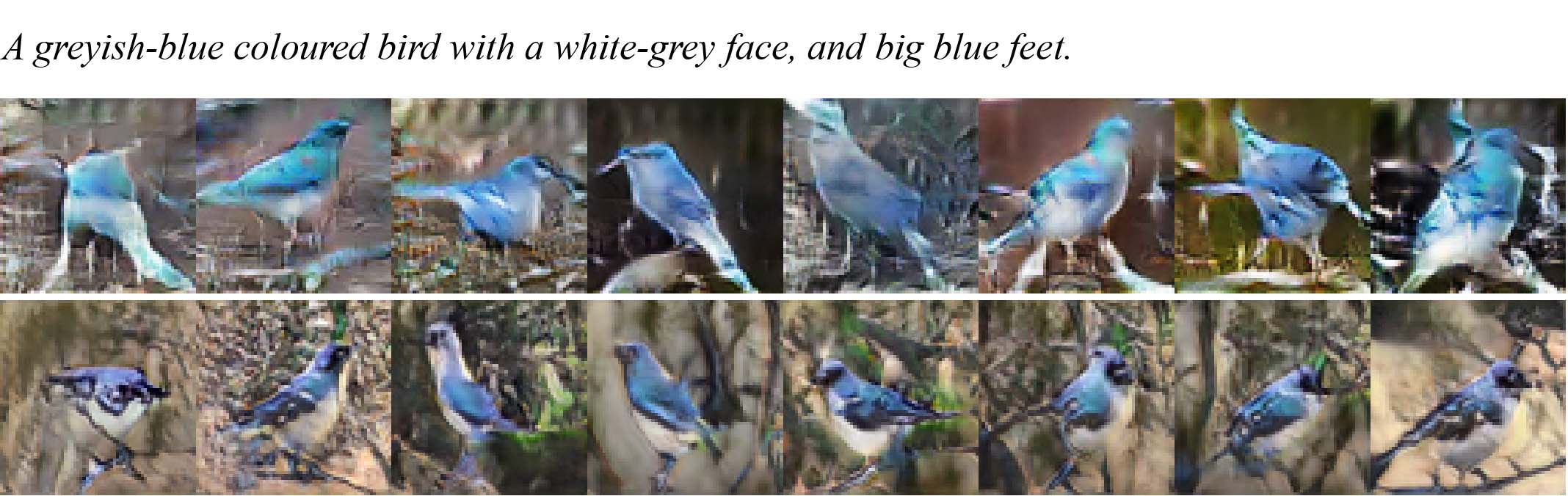}
    \caption[Comparison of 64$\times$64 models on the birds dataset (III)]{Comparison of 64$\times$64 models on the birds dataset: Stage I of StackGAN (first row), CWPGGAN (second row).}
\end{figure}

\begin{figure}[h]
    \centering
    \includegraphics[width=\textwidth]{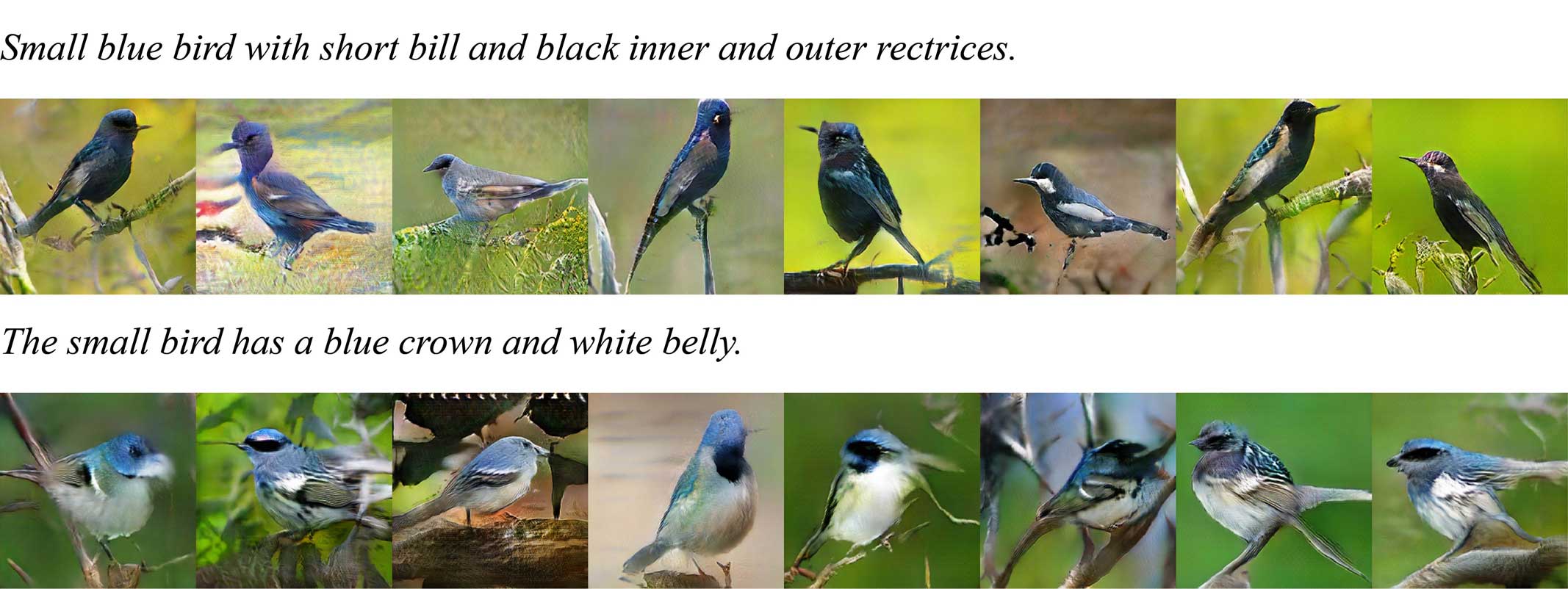}
    \caption[StackGAN Stage II birds dataset samples]{Samples (256$\times$256) generated by Stage II of StackGAN on the birds dataset.}
\end{figure}

\begin{figure}[h]
    \centering
    \includegraphics[width=\textwidth]{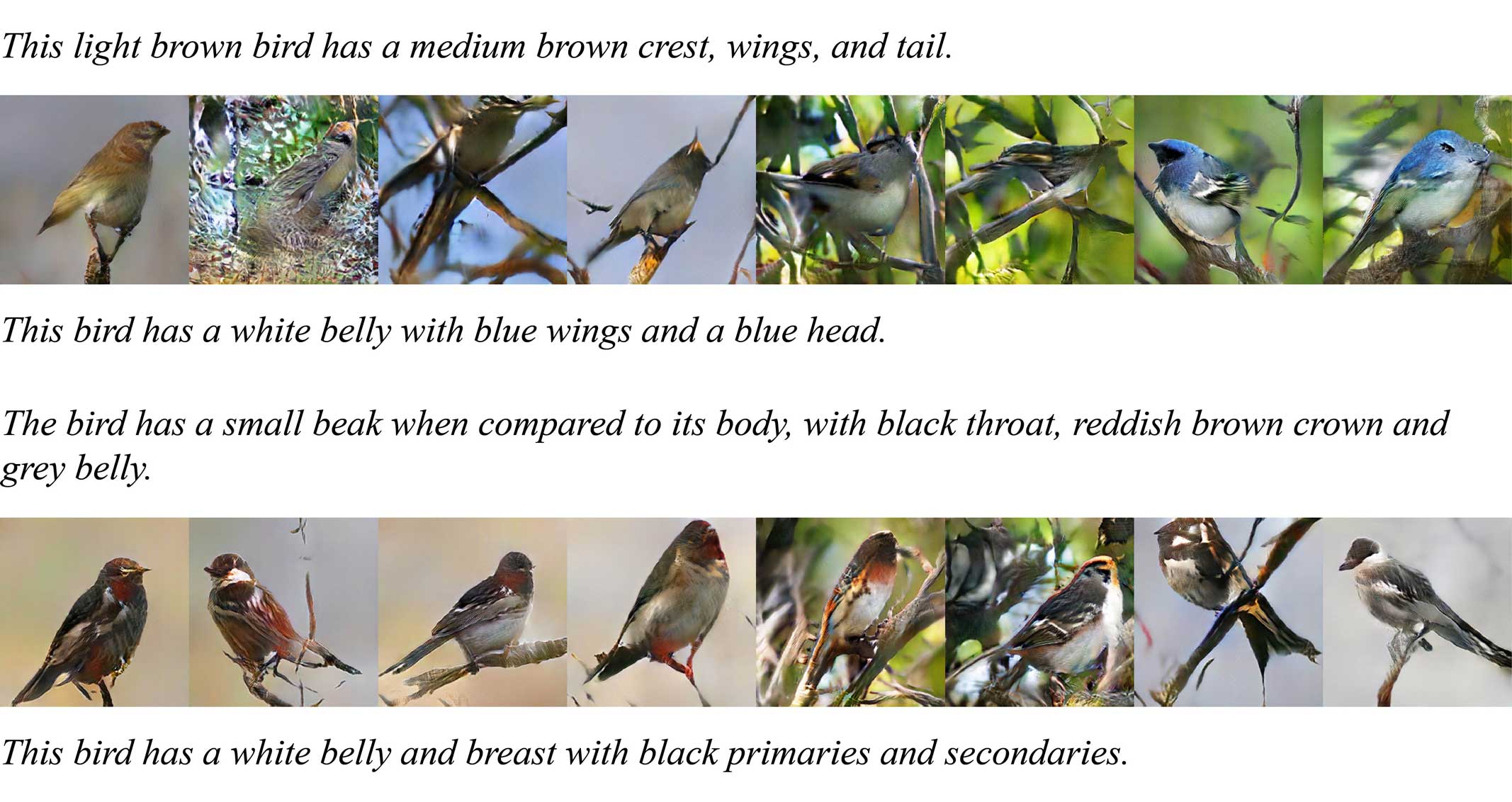}
    \caption[StackGAN Stage II birds dataset interpolations]{Samples generated by Stage II of StackGAN from interpolations of text embedding from the birds dataset.}
\end{figure}

\begin{figure}[h]
    \centering
    \includegraphics[width=\textwidth]{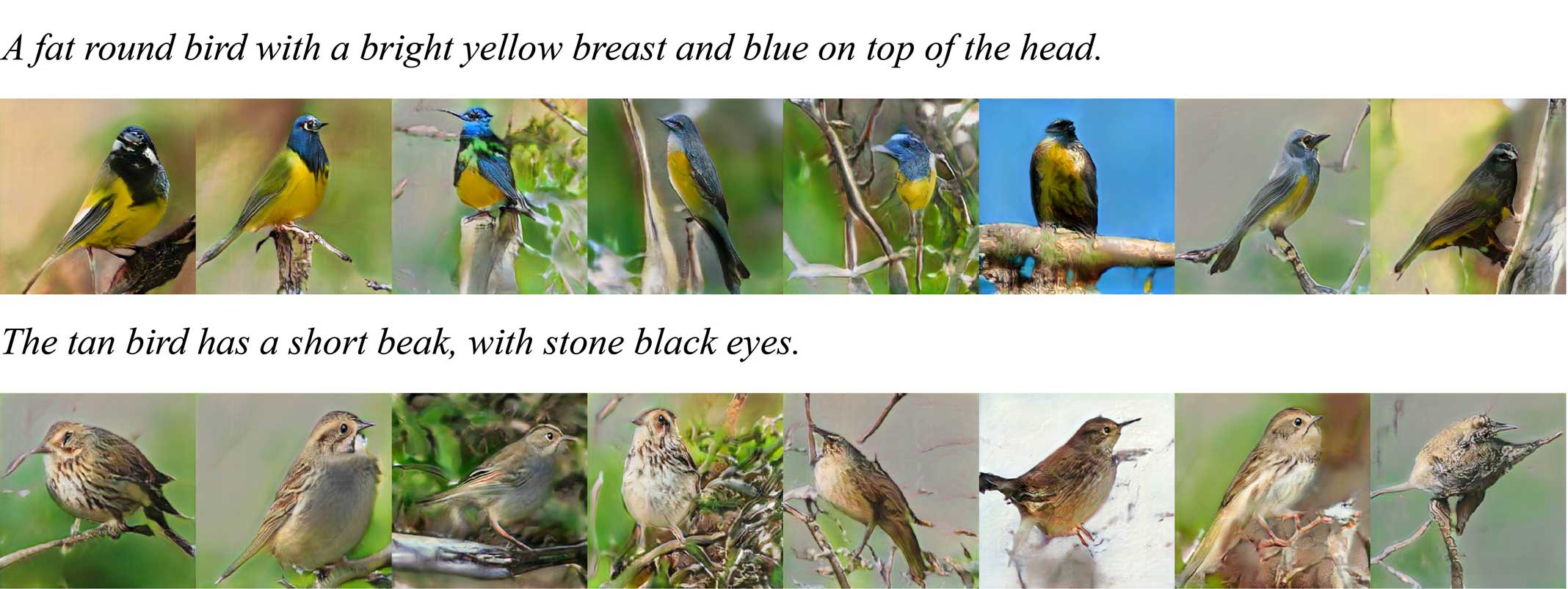}
    \caption[CWPGGAN birds dataset 256$\times$256 samples]{Samples (256$\times$256) generated by CWPGGAN on the birds dataset.}
\end{figure}

\begin{figure}[h]
    \centering
    \includegraphics[width=\textwidth]{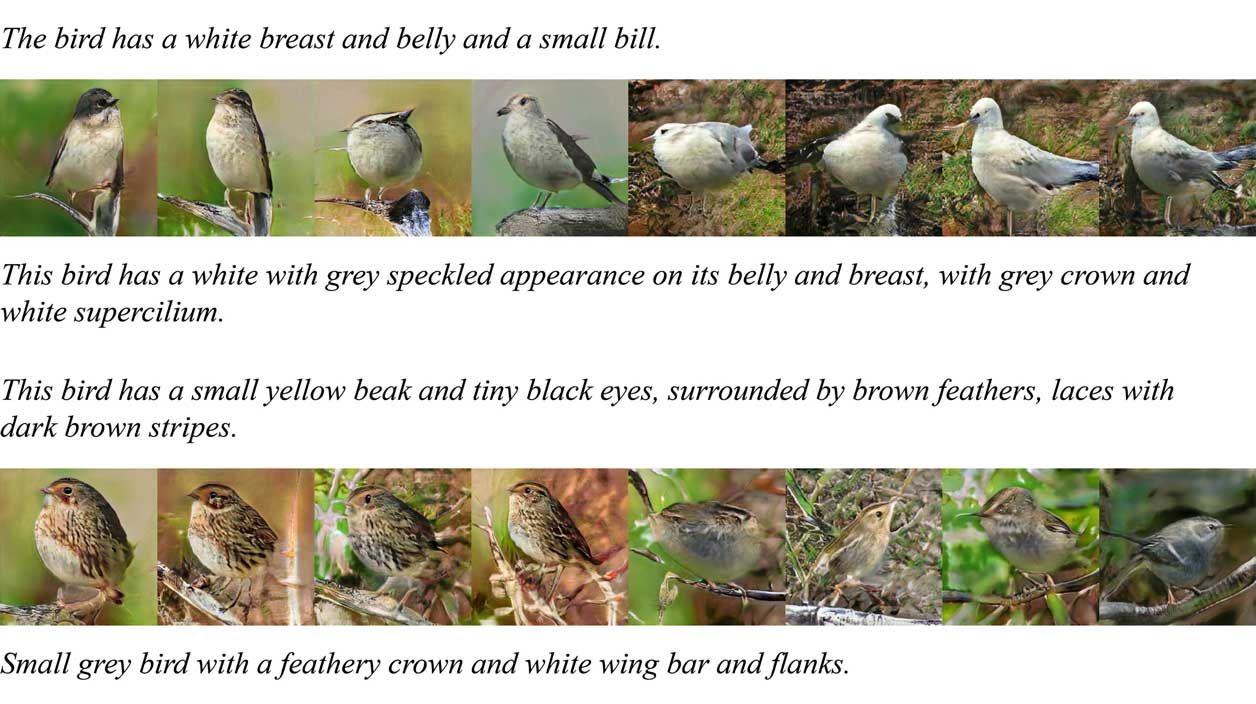}
    \caption[CWPGGAN birds dataset 256$\times$256 interpolations]{Samples (256$\times$256) generated by CWPGGAN from interpolations of text embedding from the birds dataset.}
\end{figure}

\begin{figure}[h]
    \centering
    \includegraphics[width=\textwidth]{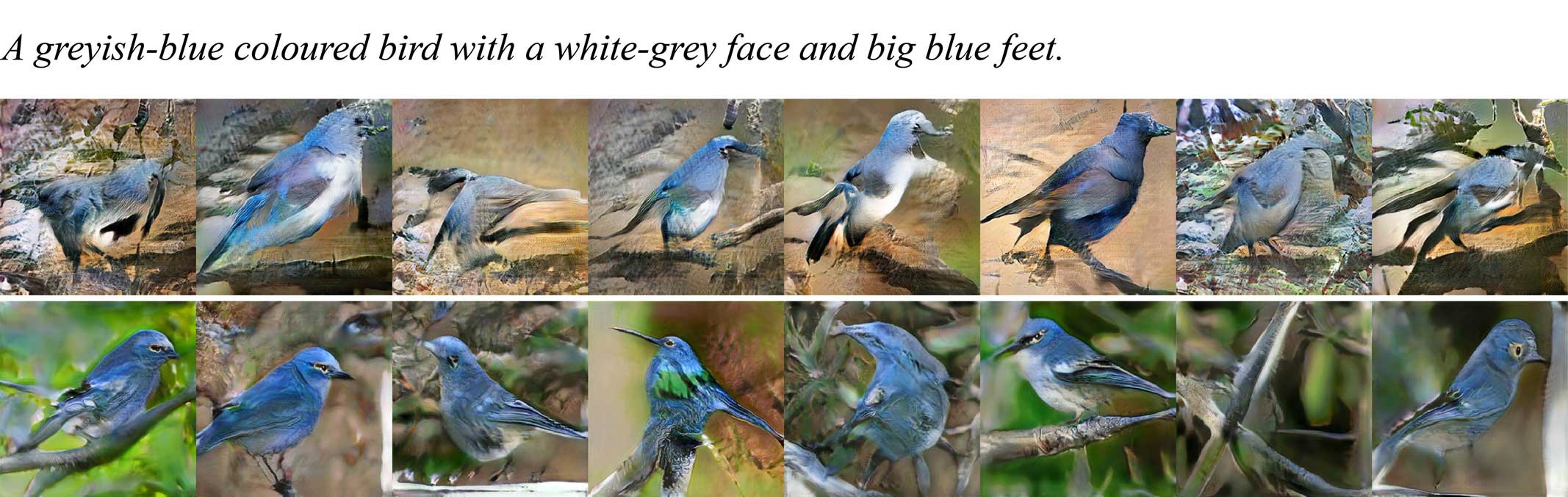}
    \caption[Comparison of 256$\times$256 models on the birds dataset (I)]{Comparison of 256$\times$256 models on the birds dataset: Stage II of StackGAN (first row), CWPGGAN (second row).}
\end{figure}

\begin{figure}[h]
    \centering
    \includegraphics[width=\textwidth]{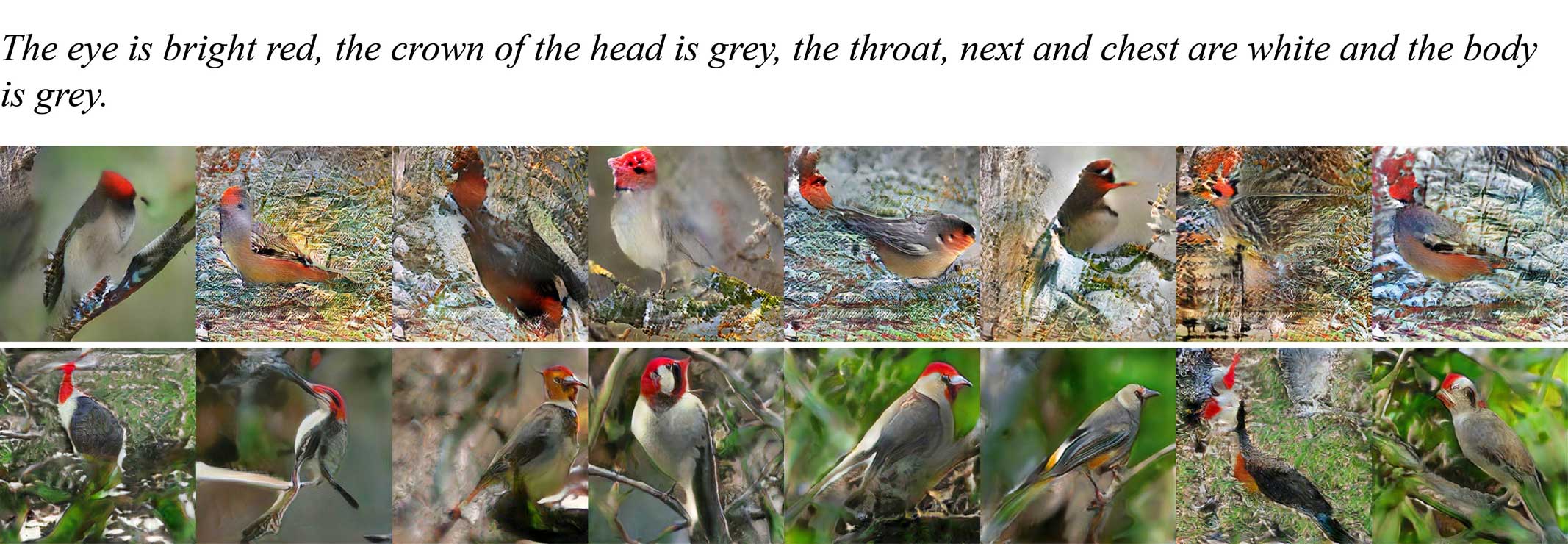}
    \caption[Comparison of 256$\times$256 models on the birds dataset (II)]{Comparison of 256$\times$256 models on the birds dataset: Stage II of StackGAN (first row), CWPGGAN (second row).}
\end{figure}

%% file: appendix5.tex
\chapter{Conditional PGGANs} \label{appendix:cpggan}

This section offers more details about the Conditional Least Squares and Wasserstein Progressive Growing GAN implementations. These details are omitted in the body of the work.

\begin{table}[h]
    \centering
    \resizebox{\linewidth}{!}{
    \begin{tabular}{| l  l  l |}
        \hline
        Generator & Act & Output shape \\ \hline
        $\vec{z} || \vec{e}$ & - & 256 \\ 
        Projection & linear & 4 $\times$ 4 $\times$ 512 \\
        Conv 3$\times$3 & ReLU & 4 $\times$ 4 $\times$ 512 \\
        Conv 3$\times$3 & ReLU & 4 $\times$ 4 $\times$ 512 \\ \hline
        Upsample & - & 8 $\times$ 8 $\times$ 512 \\
        Conv 3$\times$3 & ReLU & 8 $\times$ 8 $\times$ 512 \\
        Conv 3$\times$3 & ReLU & 8 $\times$ 8 $\times$ 512 \\ \hline
        Upsample & - & 16 $\times$ 16 $\times$ 512 \\
        Conv 3$\times$3 & ReLU & 16 $\times$ 16 $\times$ 512 \\
        Conv 3$\times$3 & ReLU & 16 $\times$ 16 $\times$ 512 \\ \hline
        Upsample & - & 32 $\times$ 32 $\times$ 512 \\
        Conv 3$\times$3 & ReLU & 32 $\times$ 32 $\times$ 512 \\
        Conv 3$\times$3 & ReLU & 32 $\times$ 32 $\times$ 512 \\ \hline
        Upsample & - & 64 $\times$ 64 $\times$ 512 \\
        Conv 3$\times$3 & ReLU & 64 $\times$ 64 $\times$ 256 \\
        Conv 3$\times$3 & ReLU & 64 $\times$ 64 $\times$ 256 \\ \hline
        Upsample & - & 128 $\times$ 128 $\times$ 256 \\
        Conv 3$\times$3 & ReLU & 128 $\times$ 128 $\times$ 128 \\
        Conv 3$\times$3 & ReLU & 128 $\times$ 128 $\times$ 128 \\ \hline
        Upsample & - & 256 $\times$ 256 $\times$ 128 \\
        Conv 3$\times$3 & ReLU & 256 $\times$ 256 $\times$ 64 \\
        Conv 3$\times$3 & ReLU & 256 $\times$ 256 $\times$ 64 \\ \hline
        Conv 2$\times$2 & ReLU & 256 $\times$ 256 $\times$ 9 \\
        Conv 1$\times$1 & linear & 256 $\times$ 256 $\times$ 3 \\ \hline
    \end{tabular}
    \begin{tabular}{| l  l  l |}
        \hline
        Discriminator & Act & Output shape \\ \hline
        Input image & - & 256 $\times$ 256 $\times$ 3 \\
        Conv 1$\times$1 & ReLU & 256 $\times$ 256 $\times$ 32 \\
        Conv 3$\times$3 & ReLU & 256 $\times$ 256 $\times$ 32 \\
        Conv 3$\times$3 & ReLU & 256 $\times$ 256 $\times$ 64 \\
        Downsample & - & 128 $\times$ 128 $\times$ 64 \\ \hline
        Conv 3$\times$3 & ReLU & 128 $\times$ 128 $\times$ 64 \\
        Conv 3$\times$3 & ReLU & 128 $\times$ 128 $\times$ 128 \\
        Downsample & - & 64 $\times$ 64 $\times$ 128 \\ \hline
        Conv 3$\times$3 & ReLU & 64 $\times$ 64 $\times$ 128 \\
        Conv 3$\times$3 & ReLU & 64 $\times$ 64 $\times$ 256 \\
        Downsample & - & 32 $\times$ 32 $\times$ 256 \\ \hline
        Conv 3$\times$3 & ReLU & 64 $\times$ 64 $\times$ 128 \\
        Conv 3$\times$3 & ReLU & 64 $\times$ 64 $\times$ 256 \\
        Downsample & - & 32 $\times$ 32 $\times$ 256 \\ \hline
        Conv 3$\times$3 & ReLU & 32 $\times$ 32 $\times$ 256 \\
        Conv 3$\times$3 & ReLU & 32 $\times$ 32 $\times$ 512 \\
        Downsample & - & 16 $\times$ 16 $\times$ 512 \\ \hline
        Conv 3$\times$3 & ReLU & 16 $\times$ 16 $\times$ 512 \\
        Conv 3$\times$3 & ReLU & 16 $\times$ 16 $\times$ 512 \\
        Downsample & - & 8 $\times$ 8 $\times$ 512 \\ \hline
        Conv 3$\times$3 & ReLU & 8 $\times$ 8 $\times$ 512 \\
        Conv 3$\times$3 & ReLU & 8 $\times$ 8 $\times$ 512 \\
        Downsample & - & 4 $\times$ 4 $\times$ 512 \\ \hline
        Embedding concatenation  & ReLU & 4 $\times$ 4 $\times$ 640 \\
        Conv 3$\times$3 & ReLU & 4 $\times$ 4 $\times$ 512 \\
        Conv 4$\times$4 & ReLU & 1 $\times$ 1 $\times$ 512 \\
        Fully connected & linear & 1 $\times$ 1 $\times$ 1 \\ \hline
      
    \end{tabular}
    }
    \caption[Conditional PGGAN architecture]{All the layers of the Conditional PGGAN generator (left) and discriminator (right). All the stages are separated by horizontal lines. }
    \label{tab:pggan_layers}
\end{table}

\section{More on Architecture and Training}

To be able to fit the model in memory at the high-resolution stages, small batch sizes have to be used. I use a batch size of 16 until the 64$\times$64 resolution is reached (including this resolution). For the higher resolutions, I use a batch size of 8. Because the performance of batch normalisation depends on bigger batch sizes, I use layer normalisation in the generator. In the discriminator, I use layer normalisation only for CLSPGGAN. The detailed architecture of the Conditional PGGAN is given in Table \ref{tab:pggan_layers}.

For training the Conditional Wasserstein PGGAN I use $n_{critic} = 1$ and a learning rate of $0.0001$ both for the generator and the discriminator. The other parameters are $\alpha = 1$, $\lambda = 150$, $\rho = 8$. The Adam optimiser is used with $\beta_{1} = 0$ and $\beta_2 = 0.99$.

For training the Conditional Least Squares PGGAN I use a learning rate of $0.0001$ both for the generator and the discriminator and $\rho = 8$. The Adam optimiser is used with $\beta_{1} = 0.5$ and $\beta_2 = 0.9$.

%% file: main.bbl
\begin{thebibliography}{10}

\bibitem{tensorflow2015-whitepaper}
Mart\'{\i}n Abadi, Ashish Agarwal, Paul Barham, Eugene Brevdo, Zhifeng Chen,
  Craig Citro, Greg~S. Corrado, Andy Davis, Jeffrey Dean, Matthieu Devin,
  Sanjay Ghemawat, Ian Goodfellow, Andrew Harp, Geoffrey Irving, Michael Isard,
  Yangqing Jia, Rafal Jozefowicz, Lukasz Kaiser, Manjunath Kudlur, Josh
  Levenberg, Dandelion Man\'{e}, Rajat Monga, Sherry Moore, Derek Murray, Chris
  Olah, Mike Schuster, Jonathon Shlens, Benoit Steiner, Ilya Sutskever, Kunal
  Talwar, Paul Tucker, Vincent Vanhoucke, Vijay Vasudevan, Fernanda Vi\'{e}gas,
  Oriol Vinyals, Pete Warden, Martin Wattenberg, Martin Wicke, Yuan Yu, and
  Xiaoqiang Zheng.
\newblock {TensorFlow}: Large-scale machine learning on heterogeneous systems,
  2015.
\newblock Software available from tensorflow.org.

\bibitem{Arjovsky2017TowardsPM}
Mart{\'i}n Arjovsky and L{\'e}on Bottou.
\newblock Towards principled methods for training generative adversarial
  networks.
\newblock {\em CoRR}, abs/1701.04862, 2017.

\bibitem{DBLP:conf/icml/ArjovskyCB17}
Mart{\'{\i}}n Arjovsky, Soumith Chintala, and L{\'{e}}on Bottou.
\newblock Wasserstein generative adversarial networks.
\newblock In Doina Precup and Yee~Whye Teh, editors, {\em Proceedings of the
  34th International Conference on Machine Learning, {ICML} 2017, Sydney, NSW,
  Australia, 6-11 August 2017}, volume~70 of {\em Proceedings of Machine
  Learning Research}, pages 214--223. {PMLR}, 2017.

\bibitem{DBLP:journals/corr/BaMK14}
Jimmy Ba, Volodymyr Mnih, and Koray Kavukcuoglu.
\newblock Multiple object recognition with visual attention.
\newblock {\em CoRR}, abs/1412.7755, 2014.

\bibitem{imagenet_cvpr09}
J.~Deng, W.~Dong, R.~Socher, L.-J. Li, K.~Li, and L.~Fei-Fei.
\newblock {ImageNet: A Large-Scale Hierarchical Image Database}.
\newblock In {\em CVPR09}, 2009.

\bibitem{Dong2017I2T2ILT}
Hao Dong, Jingqing Zhang, Douglas McIlwraith, and Yike Guo.
\newblock I2t2i: Learning text to image synthesis with textual data
  augmentation.
\newblock {\em 2017 IEEE International Conference on Image Processing (ICIP)},
  pages 2015--2019, 2017.

\bibitem{dumoulin2016guide}
Vincent {Dumoulin} and Francesco {Visin}.
\newblock {A guide to convolution arithmetic for deep learning}.
\newblock {\em ArXiv e-prints}, mar 2016.

\bibitem{gelfand2012calculus}
I.M. Gelfand and S.V. Fomin.
\newblock {\em Calculus of Variations}.
\newblock Dover Books on Mathematics. Dover Publications, 2012.

\bibitem{NIPS2014_5423}
Ian Goodfellow, Jean Pouget-Abadie, Mehdi Mirza, Bing Xu, David Warde-Farley,
  Sherjil Ozair, Aaron Courville, and Yoshua Bengio.
\newblock Generative adversarial nets.
\newblock In Z.~Ghahramani, M.~Welling, C.~Cortes, N.~D. Lawrence, and K.~Q.
  Weinberger, editors, {\em Advances in Neural Information Processing Systems
  27}, pages 2672--2680. Curran Associates, Inc., 2014.

\bibitem{gulrajani+al-2017-wasserstein-arxiv}
Ishaan Gulrajani, Faruk Ahmed, Martin Arjovsky, Vincent Dumoulin, and Aaron
  Courville.
\newblock Improved training of wasserstein gans.
\newblock In {\em Advances in Neural Information Processing Systems 30 (NIPS
  2017)}, pages 5769--5779. Curran Associates, Inc., December 2017.
\newblock arxiv: 1704.00028.

\bibitem{7780459}
K.~He, X.~Zhang, S.~Ren, and J.~Sun.
\newblock Deep residual learning for image recognition.
\newblock In {\em 2016 IEEE Conference on Computer Vision and Pattern
  Recognition (CVPR)}, pages 770--778, June 2016.

\bibitem{resnet}
K.~He, X.~Zhang, S.~Ren, and J.~Sun.
\newblock Deep residual learning for image recognition.
\newblock In {\em 2016 IEEE Conference on Computer Vision and Pattern
  Recognition (CVPR)}, pages 770--778, June 2016.

\bibitem{DBLP:journals/corr/HeZR015}
Kaiming He, Xiangyu Zhang, Shaoqing Ren, and Jian Sun.
\newblock Delving deep into rectifiers: Surpassing human-level performance on
  imagenet classification.
\newblock {\em CoRR}, abs/1502.01852, 2015.

\bibitem{He2016IdentityMI}
Kaiming He, Xiangyu Zhang, Shaoqing Ren, and Jian Sun.
\newblock Identity mappings in deep residual networks.
\newblock In {\em ECCV}, 2016.

\bibitem{NIPS2017_7240}
Martin Heusel, Hubert Ramsauer, Thomas Unterthiner, Bernhard Nessler, and Sepp
  Hochreiter.
\newblock Gans trained by a two time-scale update rule converge to a local nash
  equilibrium.
\newblock In I.~Guyon, U.~V. Luxburg, S.~Bengio, H.~Wallach, R.~Fergus,
  S.~Vishwanathan, and R.~Garnett, editors, {\em Advances in Neural Information
  Processing Systems 30}, pages 6626--6637. Curran Associates, Inc., 2017.

\bibitem{Ioffe:2015:BNA:3045118.3045167}
Sergey Ioffe and Christian Szegedy.
\newblock Batch normalization: Accelerating deep network training by reducing
  internal covariate shift.
\newblock In {\em Proceedings of the 32Nd International Conference on
  International Conference on Machine Learning - Volume 37}, ICML'15, pages
  448--456. JMLR.org, 2015.

\bibitem{pggan}
Tero Karras, Timo Aila, Samuli Laine, and Jaakko Lehtinen.
\newblock Progressive growing of gans for improved quality, stability, and
  variation.
\newblock {\em CoRR}, abs/1710.10196, 2017.

\bibitem{journals/corr/KingmaB14}
Diederik~P. Kingma and Jimmy Ba.
\newblock Adam: A method for stochastic optimization.
\newblock {\em CoRR}, abs/1412.6980, 2014.

\bibitem{journals/corr/KingmaW13}
Diederik~P. Kingma and Max Welling.
\newblock Auto-encoding variational bayes.
\newblock {\em CoRR}, abs/1312.6114, 2013.

\bibitem{Kiros:2015:SV:2969442.2969607}
Ryan Kiros, Yukun Zhu, Ruslan Salakhutdinov, Richard~S. Zemel, Antonio
  Torralba, Raquel Urtasun, and Sanja Fidler.
\newblock Skip-thought vectors.
\newblock In {\em Proceedings of the 28th International Conference on Neural
  Information Processing Systems - Volume 2}, NIPS'15, pages 3294--3302,
  Cambridge, MA, USA, 2015. MIT Press.

\bibitem{LeCun:1998:EB:645754.668382}
Yann LeCun, L{\'e}on Bottou, Genevieve~B. Orr, and Klaus-Robert M\"{u}ller.
\newblock Efficient backprop.
\newblock In {\em Neural Networks: Tricks of the Trade, This Book is an
  Outgrowth of a 1996 NIPS Workshop}, pages 9--50, London, UK, UK, 1998.
  Springer-Verlag.

\bibitem{lsgan}
Xudong Mao, Qing Li, Haoran Xie, Raymond~Y.K. Lau, Zhen Wang, and Stephen~Paul
  Smolley.
\newblock Least squares generative adversarial networks.
\newblock {\em arXiv preprint arXiv:1611.04076}, 2016.

\bibitem{DBLP:journals/corr/MnihHGK14}
Volodymyr Mnih, Nicolas Heess, Alex Graves, and Koray Kavukcuoglu.
\newblock Recurrent models of visual attention.
\newblock {\em CoRR}, abs/1406.6247, 2014.

\bibitem{NIPS2010_3958}
Hariharan Narayanan and Sanjoy Mitter.
\newblock Sample complexity of testing the manifold hypothesis.
\newblock In J.~D. Lafferty, C.~K.~I. Williams, J.~Shawe-Taylor, R.~S. Zemel,
  and A.~Culotta, editors, {\em Advances in Neural Information Processing
  Systems 23}, pages 1786--1794. Curran Associates, Inc., 2010.

\bibitem{Nilsback08}
M-E. Nilsback and A.~Zisserman.
\newblock Automated flower classification over a large number of classes.
\newblock In {\em Proceedings of the Indian Conference on Computer Vision,
  Graphics and Image Processing}, Dec 2008.

\bibitem{petzka2018on}
Henning Petzka, Asja Fischer, and Denis Lukovnikov.
\newblock On the regularization of wasserstein {GAN}s.
\newblock In {\em International Conference on Learning Representations}, 2018.

\bibitem{Radford2015UnsupervisedRL}
Alec Radford, Luke Metz, and Soumith Chintala.
\newblock Unsupervised representation learning with deep convolutional
  generative adversarial networks.
\newblock {\em CoRR}, abs/1511.06434, 2015.

\bibitem{reed2016learning}
Scott Reed, Zeynep Akata, Bernt Schiele, and Honglak Lee.
\newblock Learning deep representations of fine-grained visual descriptions.
\newblock In {\em IEEE Computer Vision and Pattern Recognition}, 2016.

\bibitem{reed2016generative}
Scott Reed, Zeynep Akata, Xinchen Yan, Lajanugen Logeswaran, Bernt Schiele, and
  Honglak Lee.
\newblock Generative adversarial text-to-image synthesis.
\newblock In {\em Proceedings of The 33rd International Conference on Machine
  Learning}, 2016.

\bibitem{Salimans2016ImprovedTF}
Tim Salimans, Ian~J. Goodfellow, Wojciech Zaremba, Vicki Cheung, Alec Radford,
  and Xi~Chen.
\newblock Improved techniques for training gans.
\newblock In {\em NIPS}, 2016.

\bibitem{7298594inception}
C.~Szegedy, Wei Liu, Yangqing Jia, P.~Sermanet, S.~Reed, D.~Anguelov, D.~Erhan,
  V.~Vanhoucke, and A.~Rabinovich.
\newblock Going deeper with convolutions.
\newblock In {\em 2015 IEEE Conference on Computer Vision and Pattern
  Recognition (CVPR)}, pages 1--9, June 2015.

\bibitem{Szegedy2016RethinkingTI}
Christian Szegedy, Vincent Vanhoucke, Sergey Ioffe, Jonathon Shlens, and
  Zbigniew Wojna.
\newblock Rethinking the inception architecture for computer vision.
\newblock {\em 2016 IEEE Conference on Computer Vision and Pattern Recognition
  (CVPR)}, pages 2818--2826, 2016.

\bibitem{Veit2016ResidualNB}
Andreas Veit, Michael~J. Wilber, and Serge~J. Belongie.
\newblock Residual networks behave like ensembles of relatively shallow
  networks.
\newblock In {\em NIPS}, 2016.

\bibitem{Villani2009}
C{\'e}dric Villani.
\newblock {\em Cyclical monotonicity and Kantorovich duality}, pages 51--92.
\newblock Springer Berlin Heidelberg, Berlin, Heidelberg, 2009.

\bibitem{WahCUB_200_2011}
C.~Wah, S.~Branson, P.~Welinder, P.~Perona, and S.~Belongie.
\newblock {The Caltech-UCSD Birds-200-2011 Dataset}.
\newblock Technical Report CNS-TR-2011-001, California Institute of Technology,
  2011.

\bibitem{DBLP:journals/corr/abs-1711-10485}
Tao Xu, Pengchuan Zhang, Qiuyuan Huang, Han Zhang, Zhe Gan, Xiaolei Huang, and
  Xiaodong He.
\newblock Attngan: Fine-grained text to image generation with attentional
  generative adversarial networks.
\newblock {\em CoRR}, abs/1711.10485, 2017.

\bibitem{han2017stackgan}
Han Zhang, Tao Xu, Hongsheng Li, Shaoting Zhang, Xiaogang Wang, Xiaolei Huang,
  and Dimitris Metaxas.
\newblock Stackgan: Text to photo-realistic image synthesis with stacked
  generative adversarial networks.
\newblock In {\em {ICCV}}, 2017.

\bibitem{DBLP:journals/corr/abs-1710-10916}
Han Zhang, Tao Xu, Hongsheng Li, Shaoting Zhang, Xiaogang Wang, Xiaolei Huang,
  and Dimitris~N. Metaxas.
\newblock Stackgan++: Realistic image synthesis with stacked generative
  adversarial networks.
\newblock {\em CoRR}, abs/1710.10916, 2017.

\end{thebibliography}
